\documentclass[twoside,11pt]{article}

\usepackage{jmlr2e}
\usepackage{amsmath}
\usepackage{graphicx} 
\usepackage{pgfplots}
\usepackage{rotating}
\usepackage{algorithm}
\usepackage{algorithmic}
\usepackage{subcaption}
\usepackage{adjustbox}

\newtheorem{problem}{Problem} 
\newcommand{\vecOnes}{\textbf{1}}
\newcommand{\matK}{\mathbf{K}}
\newcommand{\E}{\mathbb{E}}
\newcommand\var{\operatorname{Var}}
\newcommand{\mathbbP}{\mathbb{P}}
\newcommand{\HSICpop}[3]{\operatorname{HSIC}\left[\mathcal{#1},\mathcal{#2},\mathbb{P}_{#3}\right]}
\newcommand{\HSIC}[4]{\operatorname{HSIC}_u\left[\mathcal{#1},\mathcal{#2},\left({#3},{#4}\right)\right]}
\newcommand{\MMDpop}[2]{\operatorname{MMD}\left[\mathcal{F},\mathbb{P}_{#1},\mathbb{P}_{#2}\right]}

\newcommand{\squaredMMDpop}[2]{\operatorname{MMD}^2\left[\mathcal{F},\mathbb{P}_{#1},\mathbb{P}_{#2}\right]}
\newcommand{\squaredMMDu}[2]{\operatorname{MMD}_u^2\left[\mathcal{F},{#1},{#2}\right]}
\newcommand{\Eone}{\operatorname{E}_{v_1}}

\newcommand{\EXoneYone}{\operatorname{E}_{x_1, y_1}}
\newcommand{\EXtwoYtwo}{\operatorname{E}_{x_2,y_2}}
\newcommand{\EXone}{\operatorname{E}_{x_1}}
\newcommand{\EYone}{\operatorname{E}_{y_1}}
\newcommand{\EZone}{\operatorname{E}_{z_1}}
\newcommand{\EXoneYoneZone}{\operatorname{E}_{x_1,y_1,z_1}}
\newcommand{\EXtwoYtwoZtwo}{\operatorname{E}_{x_2, y_2,z_2}}

\newcommand{\wachacodeurl}{\url{https://github.com/wbounliphone/reldep}}

\newcommand{\eugenecodeurl}{\url{https://github.com/eugenium/MMD}}

\jmlrheading{0}{0}{0}{0}{0}{Bounliphone, et al.}

\newcommand{\ourtitle}{Fast Non-Parametric Tests of Relative Dependency and Similarity}

\ShortHeadings{\ourtitle}{Bounliphone, et al.}
\firstpageno{1}

\begin{document}

\title{\ourtitle}

\author{\name Wacha\ Bounliphone$^{1,2,3}$ \email wacha.bounliphone@centralesupelec.fr \\
		\name Eugene\ Belilovsky$^{1,2}$ \email eugene.belilovsky@inria.fr \\
        \name Arthur\ Tenenhaus$^{3}$ \email arthur.tenenhaus@centralesupelec.fr  \\
        \name Ioannis\ Antonoglou$^{4}$ \email ioannisa@google.com  \\
        \name Arthur\ Gretton$^{5}$\thanks{These authors contributed equally.} \email arthur.gretton@gmail.com \\  
		\name Matthew B.\ Blaschko$^{1}$\footnotemark[1] \email matthew.blaschko@esat.kuleuven.be \\
        \\
       \addr KU Leuven -- Center for Processing Speech \& Images$^1$  \\
		Kasteelpark Arenberg 10, 3001 Leuven, Belgium \\ \\
        Inria -- Galen  \& Universit\'{e} Paris-Saclay, CentraleSup\'{e}lec -- Center for Visual Computing$^2$  \\
        Grande Voie des Vignes, 92290 Ch\^{a}tenay-Malabry, France  \\ \\
        Universit\'{e} Paris-Saclay, CentraleSup\'{e}lec -- Laboratory of signals and systems$^3$ \\
        3 Rue Joliot-Curie, 91192 Gif sur Yvette, France  \\ \\
		Google DeepMind$^4$ \\
        5 New Street Square
London EC4A 3TW, UK\\ \\
        UCL -- Gatsby Computational Neuroscience Unit$^5$  \\
		25 Howland Street, London W1T 4JG, UK\\
       }

\editor{}

\maketitle

\begin{abstract} \label{jmlr2016:sec:abstract}

We introduce two novel non-parametric statistical hypothesis tests. The first test, called the \textit{relative test of dependency},  enables us to determine whether one source variable is significantly more dependent on a first target variable or a second. Dependence is measured via the Hilbert-Schmidt Independence Criterion (HSIC). The second test, called the \textit{relative test of similarity}, is use to determine which of the two samples from arbitrary distributions is significantly closer to a reference sample of interest and the relative measure of similarity is based on the Maximum Mean Discrepancy (MMD). To construct these tests, we have used as our test statistics the difference of HSIC statistics and of MMD statistics, respectively. The resulting tests are consistent and unbiased, and (being based on $U$-statistics) have favorable convergence properties. The effectiveness of the relative dependency test is demonstrated on several real-world problems: we identify languages groups from a multilingual parallel corpus, and we show that tumor location is more dependent on gene expression than chromosome imbalance. We also demonstrate the performance of the relative test of similarity over a broad selection of model comparisons problems in deep generative models.  Open source implementations of the tests developed here are available for download from \wachacodeurl\ and \eugenecodeurl. 
\end{abstract}


\section{Introduction} \label{jmlr2016:sec:introduction}

This article is based upon and extends \cite{bounliphone2015,Bounliphone2015b}.  
We address two related problems using analogous tools based on estimating correlated $U$-statistics for dependency and similarity in a non-parametric setting.

The first problem (called the \textit{relative dependency} test) is to  compare multiple dependencies to determine which of two variables most strongly influences the third, by proposing a statistical test of the null hypothesis that a source variable is more dependent to a first target variable against the alternative hypothesis that a source variable is more dependent to a second target variable. Much recent research on dependence measurement has focused on non-parametric measures of dependence, which apply even when the dependence is nonlinear, or the variables are multivariate or non-Euclidean (for instance images, strings, and graphs). The statistics for such tests are diverse, and include kernel measures of covariance \citep{GreFukTeoSonetal08,ZhaPetJanSch11} and correlation \citep{DauNki98,FukGreSunSch08}, distance covariances (which are instances of kernel tests) \citep{SzeRizBak07,SejSriGreFuk13}, kernel regression tests \citep{Cortes09, Gunn02}, rankings  \citep{HelHelGor13}, and space partitioning approaches \citep{GreGyo10,ResTesFinGroetal11,KinAtw14}. Specialization of such methods to univariate linear dependence can yield similar tests to classical approaches such as~\citet{darlington1968multiple,bring1996geometric}.  For many  problems in data analysis, however, the question of whether dependence exists is secondary: there may be multiple dependencies, and the question becomes which dependence is the strongest.  For the dependence measure, we use as our test statistic between each of the target and the source is computing using the Hilbert-Schmidt Independent Criterion (HSIC)~\citep{GreHerSmoBouSch05,GreFukTeoSonetal08} which is the distance between embeddings of the joint distribution and the product of the marginals in a reproducing kernel Hilbert space (RKHS). When the RKHSs are characteristic, the variables are independent iff HSIC = 0. 

The second problem (called the \textit{relative similarity} test) is to compare samples from three probability distributions by proposing a statistical test of the null hypothesis that a first candidate probability distribution is closer to a reference probability distribution against the alternative hypothesis that the second candidate probability distribution is closer. We have developed an application of this test to model selection for generative models.

Generative models based on deep learning techniques aim to provide sophisticated and accurate models of data, without expensive manual annotation. This is especially of interest as deep networks tend to require comparatively large training samples to achieve a good result.  Model selection within this class of techniques can be a challenge, however.  First, likelihoods can be difficult to compute for some families of recently proposed models based on deep learning~\citep{goodfellow2014generative,li2015generative}. The current best method to evaluate such models is based on Parzen-window estimates of the log likelihood \cite[Section 5]{goodfellow2014generative}. Second, if we are  given two models with similar likelihoods, we typically do not have a computationally inexpensive hypothesis test to determine whether one likelihood is significantly higher than the other.  Permutation testing or other generic strategies are often computationally prohibitive, bearing in mind the relatively high computational requirements of deep networks~\citep{krizhevsky2012imagenet}. So in this work, we provide an alternative strategy for model selection, based on our non-parametric hypothesis test of relative similarity. We treat the two  trained networks being compared as  generative models \citep{goodfellow2014generative,hinton2006fast,salakhutdinov2009deep}, and test whether the first candidate model generates samples significantly closer to a reference  validation set.
For the metric on the space of probability distribution, we use as our test statistic the difference of two the Maximum Mean Discrepancy (MMD)~\citep{gretton2006kernel,gretton2012kernel}, where MMD is the difference between mean embeddings of the distributions in a RKHS. When the RKHS is equal to the unit ball in a characteristic RKHS, the probability measures are equal iff MMD = 0. 

For both tests, care must be taken in analyzing the asymptotic behavior of the test statistics, since the measure of dependence and the measure of similarity will themselves be correlated: they are both computed with respect to the same source. Thus, we derive the joint asymptotic distribution of both dependencies and similarities. The derivation of our test utilizes classical results of $U$-statistics~\citep{hoeffding1963probability,serfling2009approximation,arcones1993limit}. In particular, we make use of results by \citet{hoeffding1963probability} and \citet{serfling2009approximation} to determine the asymptotic joint distributions of the statistics (see Theorems~\ref{jmlr2016:thm:dependency_theory:joint_asymptotic_HSIC} $\&$~\ref{jmlr2016:thm:similarity_theory:joint_asymtotic_dist_MMD}). Consequently, we derive the \emph{lowest} variance unbiased estimator of the test statistic. We prove our approach to have greater statistical power than constructing two uncorrelated statistics on the same data by subsampling, and testing on these. 

Our paper is structured as follows. In Section~\ref{jmlr2016:sec:background_material}, we formalize the two proposed problems and introduce the Maximum Mean Discrepancy (MMD) and the Hilbert-Schmidt Independent Criterion (HSIC). We provide unbiased estimators, as well as asymptotic distributions based on the theory of $U$-statistics. In Section~\ref{jmlr2016:sec:relative_test_dependency} and Section~\ref{jmlr2016:sec:relative_test_similarity} we derive the joint asymptotic distribution of the correlated HSICs and the correlated MMD and prove that our approach is strictly more powerful than a test that does not exploit the covariance between the correlated statistics. Finally, in Section~\ref{jmlr2016:sec:experiments}, we demonstrate the performance of the relative test of dependency on problems from multilingual corpus and neurosciences and we demonstrate the performance of the relative test of similarity in different scenarios where a pair of model output is compare to a validation set over a range of training regimes and settings.
\section{Motivation and Background Material} \label{jmlr2016:sec:background_material}

In this section, we begin with a formal definition of the two problems and introduce essential background knowledge necessary for the development of our later theory.

Our goal is to formulate a statistical test that answers the following questions:
Let $x$, $y$ and $z$ be a random variables defined on a topological space $\mathcal{X} \times \mathcal{Y}\times \mathcal{Z}$, with respective Borel probability measures $\mathbbP_x$, $\mathbbP_y$ and $\mathbbP_z$.  Given observations $X_m:= \lbrace x_1, ..., x_m \rbrace$, $Y_m:=\lbrace y_1, ..., y_m \rbrace$ and $Z_m:=\lbrace z_1, ..., z_m \rbrace$ such that $(x_i,y_i,z_i)$ are independent and identically distributed (i.i.d.) from $\mathbbP_x \times \mathbbP_y \times \mathbbP_z$.

\begin{problem}[Relative dependency test] \label{jmlr2016:pb:background:relative_dependency_test} 
Is the dependency between $x$ and $y$ stronger than the dependency between $x$ and $z$ ?
\end{problem}

Given observations $X_m:= \lbrace x_1, ..., x_m \rbrace$, $Y_m:=\lbrace y_1, ..., y_m \rbrace$ and $Z_m:=\lbrace z_1, ..., z_m \rbrace$ sampled i.i.d.\ from $\mathbbP_x$, $\mathbbP_y$, and $\mathbbP_z$, respectively.

\begin{problem}[Relative similarity test] \label{jmlr2016:pb:background:relative_similarity_test} 
Is the probability measure $\mathbbP_x$ closer to $\mathbbP_z$ or to  $\mathbbP_y$ ?
\end{problem}

In this section, we begin with a formal definition of the two problems and introduce essential background knowledge necessary for the development of our later theory.

To start with, we want to determine an underlying notion of a measure for similarity and dependence. We will first explain a framework for distribution analysis via the notion of kernel \textit{mean embedding} in Section~\ref{jmlr2016:subsec:background:meanmap}, and based on this  kernel mean embedding approach, we present in Section~\ref{jmlr2016:subsec:background:mmd_hsic}, our statistic for the relative test of dependence, the \textit{Hilbert-Schmidt Independence Criterion} (HSIC) and  our statistic for the relative test of similarity: the \textit{maximum mean discrepancy} (MMD).

%
%
\subsection{Kernel Mean Embedding of Distributions} \label{jmlr2016:subsec:background:meanmap}
This section presents the notion of \textit{kernel mean embeddings}~\citep{berlinet2011reproducing,smola2007hilbert}, where the idea is to generalize the Hilbert-space embedding of distributions by the kernel feature map of a distribution to Dirac measures. The kernel mean embedding has been used to define metrics for probability distributions which is important for many problems in statistics and machine learning.

First, we briefly review the properties of the Reproducing Kernel Hilbert-space (RKHS)~\citep{aronszajn1950theory}. A RKHS $\mathcal{H}$  with a reproducing kernel $k(x,y)$ is a Hilbert space of functions $f: \mathcal{X} \rightarrow \mathbb{R}$ with inner product $\langle \cdot ,\cdot \rangle_{\mathcal{H}}$. Its element $k(x,\cdot)$ satisfies the reproducing property: $\langle f,k(x,\cdot) \rangle_{\mathcal{H}} = f(x)$ for any $f \in \mathcal{H}$ and consequently, $\langle k(x,\cdot),k(y,\cdot) \rangle_{\mathcal{H}} = k(x,y)$. We define the \textit{feature map} $\phi: \mathcal{X} \rightarrow \mathcal{H}$  by $\phi(x)= k(x,\cdot)$ and using the reproducing property, we obtain that $k(x,y) = \langle \phi(x),\phi(y) \rangle_{\mathcal{H}}$.

We extend the notion of feature map to the mean embedding of probability measure : Suppose that a space $\mathcal{P}(\mathcal{X})$ consists of all Borel probability measures $\mathbbP$ on some input space $\mathcal{X}$. we define the mean embedding $\mu$ of $\mathbbP$ associated with a reproduction kernel $k$ by a mapping $\mu: \mathcal{P}(\mathcal{X}) \rightarrow \mathcal{H}$, by $\mu_{\mathbbP} = \E_{X \sim \mathbbP} \left[ k(x,\cdot) \right] = \int_{\mathcal{X}} k(\cdot,x) d\mathbb{P}(x)$. The distribution $\mathbbP$ is mapped to its expected feature map, i.e., to a point in a potentially infinite-dimensional and implicit feature space. The mean embedding $\mu$ has the property that  $ \E_{x \sim \mathbbP} \left[ f(X) \right] = \langle \mu_\mathbbP,f \rangle_{\mathcal{H}}  $ for any  $f \in \mathcal{H}$.

The notion of universal kernels and characteristic kernels are essential to the study of kernel mean embeddings \citep{FukGreSunSch08}. The kernel $k$ is said to be \textit{universal} if the corresponding RKHS $\mathcal{H}$ is dense in the space of bounded continuous functions on $\mathcal{X}$ \citep{steinwart2002influence}. It was shown that for a universal kernel $k$, $\Vert \mu_{\mathbb{P}} - \mu_{\mathbb{Q}} \Vert_{\mathcal{H}}$ iff $\mathbb{P} = \mathbb{Q}$, i.e. the map $\mu$ is injective. The kernel $k$ is said to be \textit{characteristic} if the map $\mu$ is injective and the RKHS $\mathcal{H}$ is said to be characteristic if its reproducing kernel is characteristic. This notion was introduced by \citet{FukGreSunSch08} and it was shown that Gaussian and Laplacian
kernels are characteristic on $\mathbb{R}^d$. 

Furthermore, the notion of mean embedding can be generalized to joint distributions of two variables using tensor product feature spaces. Let $(x,y)$ be random variables  on $\mathcal{X} \times \mathcal{Y}$ and $\mathcal{F}$ be a RKHS with measurable kernel $k$ on $\mathcal{X}$ and $\mathcal{H}$ be a RKHS with measurable kernel $l$ on $\mathcal{Y}$. We assume that $\E_x [k(x,x)] < \infty$ and $\E_y [l(y,y')] < \infty$, the cross-covariance operator (see \cite{baker1973joint} and \cite{fukumizu2004dimensionality}) $C_{yx}: \mathcal{H} \rightarrow \mathcal{F}$   is defined as $C_{yx} := \E_{yx} \left[ \phi(y) \otimes \phi(x) \right] - \mu_{\mathbbP_x} \otimes \mu_{\mathbbP_y} = \mu_{\mathbbP_{yx}} - \mu_{\mathbbP_y} \otimes \mu_{\mathbbP_x}$. The unique bounded operator $C_{yx}$  satisfies $\langle g, C_{yx}f \rangle = \operatorname{Cov} [f(x),g(y)]$ for all $f \in \mathcal{F}$ and $g \in \mathcal{H}$.

%
%
\subsection{The Maximum Mean Discrepancy and the Hilbert-Schmidt Independence Criterion} \label{jmlr2016:subsec:background:mmd_hsic}

In this section, we give a formal definition of the Maximum Mean Discrepancy (MMD) and the Hilbert-Schmidt Independence Criterion (HSIC). 
\begin{definition} \label{jmlr2016:def:background:MMDpop}
Let $\mathbbP_x$ and $\mathbbP_y$ be the marginal distributions on domains $\mathcal{X}$ and $\mathcal{Y}$ and let $\mathcal{F}$ be a unit ball in a characteristic RKHS $\mathcal{H}$, with the continuous feature mapping $\phi(x) \in \mathcal{F}$ from each $x \in \mathcal{X}$, such that the inner product between the features is given by the positive definite kernel function $k(x,x') := \langle \phi(x), \phi(x')\rangle$. We denote the expectation of $\phi(x)$ by $\mu_\mathbbP:= \E_{\mathbbP} [\phi(x)]$ and we define the maximum mean discrepancy in a RKHS $\mathcal{F}$ as
\begin{align}
\MMDpop{x}{y} = \Vert \mu_{\mathbbP_x} - \mu_{\mathbbP_y}\Vert_{\mathcal{H}}. 
\end{align} 
Given $x$ and $x'$ independent random variables with distribution $\mathbbP_x$, and $y$ and $y'$ independent variables with distribution $\mathbbP_y$, the population $\operatorname{MMD}^2$ can be expressed in terms of expectations of kernel functions $k$ 
\begin{align}
\squaredMMDpop{x}{y} = \E_{x,x'} \left[ k(x,x') \right] - 2 \E_{x,y} \left[ k(x,y) \right] + \E_{y,y'} \left[ k(y,y') \right]. \
\end{align}
\end{definition}
The following theorem describes an unbiased quadratic-time estimate of the MMD, and  its asymptotic distribution when $\mathbbP_x$ and $\mathbbP_y$ are different.
\begin{theorem}\label{jmlr2016:thm:background:unbiaisedMMD}
Given observations $X_m:= \lbrace x_1, ..., x_m \rbrace$ and $Y_n:=\lbrace y_1, ..., y_n \rbrace$ i.i.d.\ respectively from $\mathbbP_x$ and $\mathbbP_y$, an unbiased empirical estimate of $\squaredMMDpop{x}{y}$ is a sum of two $U$-statistics and a sample average
\begin{align} \label{jmlr2016:eq:background:unbiaisedMMD_differentsizes}
\squaredMMDu{X_m}{Y_n} &= \frac{1}{m(m-1)} \sum_{i=1}^m \sum_{j \ne i}^m k(x_i,x_j) + \frac{1}{n(n-1)} \sum_{i=1}^n \sum_{j \ne i}^n k(y_i,y_j) \\
& \qquad - \frac{2}{mn} \sum_{i=1}^m \sum_{j=1}^n k(x_i,y_j). \nonumber
\end{align}
Let $\mathcal{V} := (v_1,...,v_m)$ be $m$ i.i.d. random variables, where $v := (x,y) \sim \mathbbP_x \times \mathbbP_y$. When $ m = n$, an unbiased empirical estimate of $\squaredMMDpop{x}{y}$ is 
\begin{equation} \label{jmlr2016:eq:background:unbiaisedMMD}
\squaredMMDu{X_m}{Y_m} = \frac{1}{m(m-1)} \sum_{i \ne j}^m f(v_i,v_j)
\end{equation}
with $f(v_i,v_j) = k(x_i,x_j) + k(y_i,y_j) -k(x_i,y_j)  -k(x_j,y_i)$. We assume that $\E(f^2) < \infty$. \\
When $P_x\neq P_y$, as $m \rightarrow \infty$, $\squaredMMDu{X}{Y}$ converges in distribution to a Gaussian according to 
\begin{equation} \label{jmlr2016:eq:background:asymptoticdistributionMMD}
m^{1/2} \left( \squaredMMDu{X_m}{Y_m} - \squaredMMDpop{x}{y} \right) \longrightarrow \mathcal{N} \left( 0, \sigma^2_{XY} \right)
\end{equation}
where 
\begin{equation}
\sigma^2_{XY} = 4 \left( \E_{v_1} [(E_{v_2}f(v_1,v_2))^2] - [(E_{v_1,v_2}f(v_1,v_2))^2] \right)
\label{jmlr2016:eq:background:variance_asymptoticdistributionMMD}
\end{equation}
uniformly at rate $1/\sqrt{m}$.
\end{theorem}
MMD determines if two samples are from different distributions: if $\mathcal{F}$ is a unit ball in a universal RKHS $\mathcal{H}$, defined on the compact metric space $\mathcal{X}$ with associated kernel $k(.,.)$, then $\squaredMMDpop{x}{y} = 0$ if and only if $\mathbbP_x = \mathbbP_y$ \citep{gretton2012kernel}.

We now present HSIC. 
\begin{definition} \label{jmlr2016:def:background:HSICpop}
Let $P_{xy}$ be a Borel probability measure over ($\mathcal{X} \times \mathcal{Y}, \Gamma \times \Lambda$) with $\Gamma$ and $\Lambda$ the respective Borel sets on $\mathcal{X}$ and $\mathcal{Y}$. Let $\mathcal{F}$ and $\mathcal{G}$ be separable RKHSs with the continuous feature mapping $\phi(x) \in \mathcal{F}$ from each $x \in \mathcal{X}$, such that the inner product between the features is given by the kernel function $k(x,x') := \langle \phi(x), \phi(x')\rangle$ and $\varphi(x) \in \mathcal{G}$ from each $y \in \mathcal{Y}$, such that $l(y,y') := \langle \varphi(y), \varphi(y')\rangle$. When the kernels $k$ and $l$ are respectively associated uniquely and bounded on $\mathcal{X}$ and $\mathcal{Y}$, the Hilbert-Schmidt Independence Criterion (HSIC) is defined as as the squared HS-norm of the associated cross-covariance operator $C_{yx}$. 
When the kernels $k$, $l$ are associated uniquely withs respective RKHSs $\mathcal{F}$ and $\mathcal{G}$ and bounded, the population HSIC can be expressed in terms of expectations of kernel functions 
\begin{align}\label{jmlr2016:eq:background:HSICpop}
\HSICpop{F}{G}{xy} :&= \Vert C_{yx} \Vert ^2_{HS}  \nonumber \\
&= \E_{xx'yy'} \left[ k(x,x')l(y,y') \right] + \E_{xx'} \left[ k(x,x') \right] 
\E_{yy'} \left[ l(y,y') \right] \nonumber \\
&\qquad - 2\E_{xy} \left[ \E_{x'} [k(x,x')] \E_{y'} [l(y,y')] \right]  .
\end{align}
\end{definition}
The following theorem describes an unbiased quadratic-time estimate of the HSIC, and  its asymptotic property.
\begin{theorem} 
\label{jmlr2016:thm:background:unbiaisedHSIC}
Given $(X_m,Y_m) = \lbrace (x_1,y_1), ..., (x_m,y_m) \rbrace$ of size $m$ drawn i.i.d.\ from $\mathbbP_{xy}$.  An unbiased estimator $\HSIC{F}{G}{X_m}{Y_m}$ is given by 
\begin{align} \label{jmlr2016:eq:background:unbiaisedHSIC}
\HSIC{F}{G}{X_m}{Y_m} = \dfrac{1}{m(m-3)} \left[ \operatorname{Tr} ( \tilde{\mathbf{K}} \tilde{\mathbf{L}} )  + \dfrac{\textbf{1}'\tilde{\mathbf{K}} \textbf{1} \textbf{1}' \tilde{\mathbf{L}} \textbf{1}}{(m-1)(m-2)} - \dfrac{2}{m-2} \textbf{1}' \tilde{\mathbf{K}} \tilde{\mathbf{L}} \textbf{1} \right]
\end{align}
where $\textbf{1}$ is the vector of all ones and $\tilde{\mathbf{K}}$ and $\tilde{\mathbf{L}} \in \mathbb{R}^{m \times m}$ are kernel matrices related to $\mathbf{K}$ and $\mathbf{L}$ by $ \tilde{\mathbf{K}}_{ij} = (1- \delta_{ij})k(x_i,x_j)$ and $ \tilde{\mathbf{L}}_{ij} = (1- \delta_{ij})l(y_i,y_j)$.

This finite sample unbiased estimator of $\HSICpop{F}{G}{xy}$ can be written as a U-statistic,
\begin{equation} \label{jmlr2016:eq:background:ustatisticHSIC}
\HSIC{F}{G}{X_m}{Y_m} = (m)_4^{-1} \displaystyle \sum_{(i,j,q,r) \in i^m_4} h_{ijqr}
\end{equation}
where $(m)_4 := \dfrac{m!}{(m-4)!}$, the index set $i^m_4$ denotes the set of all $4-$tuples drawn without replacement from the set $\left\lbrace 1, \dots m \right\rbrace $, and the kernel h of the U-statistic is defined as
\begin{equation} \label{jmlr2016:eq:background:kernel_ustatisticHSIC}
h_{ijqr} = \dfrac{1}{24} \displaystyle \displaystyle \sum_{(s,t,u,v)}^{(i,j,q,r)} k_{st} (l_{st} + l_{uv} -2 l_{su})
\end{equation}
We assume that $\E[h^2] < \infty$. When $\mathbbP_{xy} \ne \mathbbP_{x} \mathbbP_{y}$,  as $m \rightarrow \infty$,
\begin{equation}\label{jmlr2016:eq:background:asymptoticdistributionHSIC}
m^{1/2} \left(\HSICpop{F}{G}{xy} - \HSIC{F}{G}{X_m}{Y_m} \right) \longrightarrow \mathcal{N}(0, \sigma^2_{XY})
\end{equation}
where 
\begin{equation} \label{jmlr2016:eq:background:variance_asymptoticdistributionHSIC}
\sigma^2_{XY} = 16 \left( \E_{x_i} \left( \E_{x_j,x_q,x_r} h_{ijqr} \right)^2 - \HSICpop{F}{G}{xy} \right). 
\end{equation}
Its empirical estimate is $\hat{\sigma}_{XY} = 16\left( R_{XY} - \left(\HSIC{F}{G}{X_m}{Y_m} \right)^2 \right)$ where $R_{XY} = \dfrac{1}{m} \displaystyle\sum_{\substack{i=1}}^m \left( (m-1)_3^{-1} \sum_{(j,q,r) \in i^m_{3} \backslash \left\lbrace i \right\rbrace } h_{ijqr} \right) ^2 $ and the index set $i^m_3 \backslash \left\lbrace i \right\rbrace$ denotes the set of all $3-$tuples drawn without replacement from the set $\left\lbrace 1, \dots m \right\rbrace \backslash \left\lbrace i \right\rbrace$.

\end{theorem}

HSIC determines independence: $\HSICpop{F}{G}{xy} = 0$ if and only if $\mathbbP_{xy} = \mathbbP_x \mathbbP_y$ when kernels $k$ and $l$ are characteristic on their respective marginal domains~\citep{gretton2006kernel}.

%
%
\subsection{Statistical Hypothesis Testing} \label{jmlr2016:subsec:background:statistical_hypothesis_testing}

Having defined the two tests statistics (the $\HSICpop{F}{G}{xy}$, as a dependence measure and the $\MMDpop{x}{y}$, as a similarity measure on probability), we address the problems Pb.~\ref{jmlr2016:pb:background:relative_dependency_test} and Pb.~\ref{jmlr2016:pb:background:relative_similarity_test} as \textit{statistical hypothesis tests}. We briefly described the framework of statistical hypothesis testing~\citep{lehmann1986testing} as it applies in this context. 
\begin{description}
\item[(Pb.~\ref{jmlr2016:pb:background:relative_dependency_test}) - Relative HSIC] We denote by $\left( X_m,Y_m,Z_m \right)$ the joint sample of observations drawn i.i.d.\ with Borel probability measure $\mathbbP_{xyz}$ defined on the domain $\mathcal{X} \times \mathcal{Y} \times \mathcal{Z}$ and the kernels $k$, $l$ and $d$ associated uniquely with RKHSs $\mathcal{F}$, $\mathcal{G}$, and $\mathcal{H}$, respectively. The statistical \textit{relative independence} test $\mathcal{T}_{HSIC}: \mathcal{X}^m \times \mathcal{X}^m \times \mathcal{X}^m \mapsto \lbrace 0,1 \rbrace$ is use to test the null hypothesis $\mathcal{H}_0^{HSIC}$ : $\HSICpop{F}{G}{xy} \leq \HSICpop{F}{H}{xz}$ versus the alternative hypothesis $\mathcal{H}_1^{HSIC}$ : $\HSICpop{F}{G}{xy} > \HSICpop{F}{H}{xz}$ at a given significance level $\alpha$.
\item[(Pb.~\ref{jmlr2016:pb:background:relative_similarity_test}) - Relative MMD] We denote the observations $X_m:= \lbrace x_1, ..., x_m \rbrace$, $Y_m:=\lbrace y_1, ..., y_m \rbrace$ and $Z_m:=\lbrace z_1, ..., z_m \rbrace$ i.i.d.\ with respective Borel probability measures $\mathbbP_x$, $\mathbbP_y$ and $\mathbbP_z$ defined on $\mathcal{X}$ and the kernel $k$ associated uniquely with the  separable RKHS $\mathcal{F}$. The statistical \textit{relative similarity} test $\mathcal{T}_{MMD}: \mathcal{X}^m \times \mathcal{X}^m \times \mathcal{X}^m \mapsto \lbrace 0,1 \rbrace$ is used to test the null hypothesis $\mathcal{H}_0^{MMD}$: $ \MMDpop{x}{y} \leq \MMDpop{x}{z}$ versus the alternative hypothesis $\mathcal{H}_1^{MMD}$: $\MMDpop{x}{y} > \MMDpop{x}{z}$ at a given significance level $\alpha$.
\end{description}
The tests are constructed by comparing the test statistics, in our case respectively the difference of the two empirical unbiased estimates $\HSIC{F}{G}{X_m}{Y_m} - \HSIC{F}{H}{X_m}{Z_m}$ or $\squaredMMDu{X_m}{Y_m} - \squaredMMDu{X_m}{Z_m}$, with a particular threshold: if the threshold is exceeded, then the test rejects the null hypothesis at a given significance level $\alpha$. The approach that we adopt here is to derive respectively the joint asymptotic distribution of the empirical estimate of  $\big[ \HSIC{F}{G}{X_m}{Y_m},\  \HSIC{F}{H}{X_m}{Z_m}\big]^T$ and $\big[ \squaredMMDu{X_m}{Y_m},\  \squaredMMDu{X_m}{Z_m} \big]^T$.
\section{A relative test of dependency} \label{jmlr2016:sec:relative_test_dependency}

In this Section, we calculate two dependent HSIC statistics and derive the joint asymptotic distribution of these dependent quantities, which is used to construct a consistent test for the Relative HSIC problem (Pb.~\ref{jmlr2016:pb:background:relative_dependency_test}).
We next construct a simpler consistent test,  by computing two independent HSIC statistics on sample subsets.  While the simpler strategy is superficially attractive and slightly less effort to implement, we prove the dependent strategy is strictly more powerful.
%
%
\subsection{Joint asymptotic distribution of HSIC and test}
\label{jmlr2016:subsec:relative_test_dependency:joint_asympt_dist}
Given observations $X_m:= \lbrace x_1, ..., x_m \rbrace$, $Y_m:=\lbrace y_1, ..., y_m \rbrace$ and $Z_m:=\lbrace z_1, ..., z_m \rbrace$ i.i.d.\ from $\mathbbP_x$, $\mathbbP_y$ and $\mathbbP_z$ respectively, We denote by $k$, $l$ and $d$ the kernels associated uniquely with respective reproducing kernel Hilbert spaces $\mathcal{F}$, $\mathcal{G}$ and $\mathcal{H}$.  Moreover, $\mathbf{K}$, $\mathbf{L}$ and $\mathbf{D} \in {R}^{m \times m}$ are kernel matrices containing $k_{ij} = k(x_i,x_j)$, $l_{ij} = l(y_i,y_j)$ and $d_{ij} = d(z_i,z_j)$.  Let $\HSIC{F}{G}{X_m}{Y_m}$ and $\HSIC{F}{H}{X_m}{Z_m}$ be respectively the unbiased estimators of $\HSICpop{F}{G}{xy}$ and $\HSICpop{F}{H}{xz}$, written as a sum of $U$-statistics with respective kernels $h_{ijqr}$ and $g_{ijqr}$ as described in Eq.~\eqref{jmlr2016:eq:background:kernel_ustatisticHSIC},
\begin{align}\label{jmlr2016:eq:dependecy_theory_ustat_kernels_h_g}
h_{ijqr} = \dfrac{1}{24} \displaystyle \displaystyle \sum_{(s,t,u,v)}^{(i,j,q,r)} k_{st} (l_{st} + l_{uv} -2 l_{su}), \hspace{2pt} g_{ijqr} = \dfrac{1}{24} \displaystyle \displaystyle \sum_{(s,t,u,v)}^{(i,j,q,r)} k_{st} (d_{st} + d_{uv} -2 d_{su}).
\end{align}
\begin{theorem} \label{jmlr2016:thm:dependency_theory:joint_asymptotic_HSIC}
\emph{\textbf{(Joint asymptotic distribution of HSIC)}} If $\E[h^2] < \infty$ and $\E[g^2] < \infty$, then as $m \rightarrow \infty$,
\begin{align}
m^{1/2} 
& \left( \begin{pmatrix}
\HSIC{F}{G}{X_m}{Y_m} \\ 
\HSIC{F}{H}{X_m}{Z_m}
\end{pmatrix}
-
\begin{pmatrix}
\HSICpop{F}{G}{xy} \\ 
\HSICpop{F}{H}{xz}
\end{pmatrix}  
\right) 
\overset{d}{\longrightarrow} \mathcal{N} 
\left( 
\begin{pmatrix}
0 \\ 
0
\end{pmatrix},
\begin{pmatrix} 
\sigma_{XY}^2 & \sigma_{XYXZ} \\ 
\sigma_{XYXZ} & \sigma_{XZ}^2 
\end{pmatrix} 
\right), 
\label{jmlr2016:eq:dependency_theory:joint_asymptotic_HSIC}
\end{align}
where $\sigma^2_{XY}$ and $\sigma^2_{XZ}$ are as in Eq.~\eqref{jmlr2016:eq:background:variance_asymptoticdistributionHSIC}.
The empirical estimate of $\sigma_{XYXZ}$ is
$\hat{\sigma}_{XYXZ} = \dfrac{16}{m} \left( R_{XYXZ} - \HSIC{F}{G}{X_m}{Y_m} \HSIC{F}{H}{X_m}{Z_m} \right)$, where 
\begin{equation}
R_{XYXZ}= \dfrac{1}{m} \displaystyle\sum_{\substack{i=1}}^m \left( (m-1)_3^{-2} \sum_{(j,q,r) \in i^m_{3} \backslash \left\lbrace i \right\rbrace } h_{ijqr}  g_{ijqr}\right).
\label{jmlr2016:eq:dependency_theory:term_Rxyxz}
\end{equation} 
\end{theorem}
\begin{proof} 
Eq.~\eqref{jmlr2016:eq:dependency_theory:term_Rxyxz} is constructed with the definition of variance of a U-statistic as given by \citet[Ch. 5]{serfling2009approximation} where one variable is fixed. Eq.~\eqref{jmlr2016:eq:dependency_theory:joint_asymptotic_HSIC} follows from the application of  \citet[Theorem 7.1]{hoeffding1963probability},  which gives the joint asymptotic distribution of U-statistics.
\end{proof}
Based on the joint asymptotic distribution of HSIC described in Theorem~\ref{jmlr2016:thm:dependency_theory:joint_asymptotic_HSIC}, we can now describe a statistical test to solve Pb.~\ref{jmlr2016:pb:background:relative_dependency_test} described in Section~\ref{jmlr2016:subsec:background:statistical_hypothesis_testing}. This is achieved by projecting the distribution to 1D using the statistic $\HSIC{F}{G}{X_m}{Y_m} - \HSIC{F}{H}{X_m}{Z_m}$, and determining where the statistic falls relative to a conservative estimate of the the $1-\alpha$ quantile of the null $\mathcal{H}_0^{HSIC}$.
We now derive this conservative estimate.  
A simple way of achieving this is to
rotate the distribution by $\frac{\pi}{4}$ counter-clockwise about the origin, and to integrate the resulting distribution projected
onto the first axis (cf.\ Fig.~\ref{jmlr2016:fig:dependency_experiments:empirical_HSIC_differentsamplesize}).
Let denote the asymptotically normal distribution of $m^{1/2}\HSIC{F}{G}{X_m}{Y_m}  \HSIC{F}{H}{X_m}{Z_m}]^T$  as $\mathcal{N}(\mu,\mathbf{\Sigma})$. 
The distribution resulting from rotation and projection is
\begin{align}
\label{jmlr2016:eq:dependency_theory:asymptotic_dist_keepAll}
\mathcal{N} &\left( [\mathbf{Q} \mu]_{1}, [\mathbf{Q} \mathbf{\Sigma} \mathbf{Q}^{T}]_{11} \right),
\end{align}
where $\mathbf{Q} = \dfrac{\sqrt{2}}{2}\begin{pmatrix}
1 & -1 \\ 
1 & 1
\end{pmatrix}$ is the rotation matrix by $\frac{\pi}{4}$ and
\begin{align}
&[\mathbf{Q} \mu]_{1} = \frac{\sqrt{2}}{2} \left( 
\HSIC{F}{G}{X_m}{Y_m} - \HSIC{F}{H}{X_m}{Z_m}
\right), \\
&[\mathbf{Q} \mathbf{\Sigma} \mathbf{Q}^{T}]_{11} = \frac{1}{2}(\sigma_{XY}^2 + \sigma_{XZ}^2 - 2 \sigma_{XYXZ}).
\label{jmlr2016:eq:dependency_theory:variance_asymptotic_dist_keepAll}
\end{align}
Following the empirical distribution from Eq.~\eqref{jmlr2016:eq:dependency_theory:asymptotic_dist_keepAll},
a test with statistic $\HSIC{F}{G}{X_m}{Y_m} - \HSIC{F}{H}{X_m}{Z_m}$ has $p$-value
\begin{equation}
p \le 1 - \mathbf{\Phi}\left( \frac{ ( \HSIC{F}{G}{X_m}{Y_m} - \HSIC{F}{H}{X_m}{Z_m} )}{\sqrt{ \sigma^2_{XY} + \sigma^2_{XZ} - 2 \sigma_{XYXZ} }} \right),
\label{jmlr2016:eq:dependency_theory:pvalue}
\end{equation}
where $\mathbf{\Phi}$ is the CDF of a standard normal distribution, and we have made the most conservative possible assumption that $\HSICpop{F}{G}{xy} - \HSICpop{F}{H}{xz}=0$ under the null (the null also allows for the difference in population dependence measures to be negative).\\ 
To implement the test in practice, the variances of $\sigma_{XY}^2,\sigma_{XZ}^2$ and $\sigma_{XYXZ}^2$ may be replaced by their empirical estimates. The test will still be consistent for a large enough sample size, since the estimates will be sufficiently well converged to ensure the test is  calibrated. \\
Eq.~\eqref{jmlr2016:eq:dependency_theory:term_Rxyxz} is expensive to compute na\"{i}vely, because even computing the kernels $h_{ijqr}$ and $g_{ijqr}$ of the $U$-statistic itself is a non trivial task.  Following~\citet[Section 2.5]{SonSmoGreBedetal12}, we first form a vector $\mathbf{h_{XY}}$ with entries corresponding to $\sum_{(j,q,r) \in i^m_{3} \backslash \left\lbrace i \right\rbrace } h_{ijqr}$, and a vector $\mathbf{h_{XZ}}$ with entries corresponding to $\sum_{(j,q,r) \in i^m_{3} \backslash \left\lbrace i \right\rbrace } g_{ijqr} $.  Collecting terms in Eq.~\eqref{jmlr2016:eq:dependecy_theory_ustat_kernels_h_g} related to kernel matrices $\tilde{\mathbf{K}}$ and $\tilde{\mathbf{L}}$, $\mathbf{h_{XY}}$ can be written as
\begin{align}\label{jmlr2016:eq:dependency_theory::compute_hxy}
\mathbf{h_{XY}} &= (m-2)^2 \left( \tilde{\mathbf{K}} \odot \tilde{\mathbf{L}} \right)  \textbf{1} 
- m (\tilde{\mathbf{K}} \textbf{1}) \odot (\tilde{\mathbf{L}} \textbf{1}) \\
&+ (m-2) \left( (\operatorname{Tr}(\tilde{\mathbf{K}}  \tilde{\mathbf{L}})) \textbf{1} - \tilde{\mathbf{K}} (\tilde{\mathbf{L}} \textbf{1}) - \tilde{\mathbf{L}} (\tilde{\mathbf{K}} \textbf{1}) \right) \nonumber \\
&+ (\textbf{1}^T \tilde{\mathbf{L}} \textbf{1})\tilde{\mathbf{K}} \textbf{1} + (\textbf{1}^T \tilde{\mathbf{K}} \textbf{1})\tilde{\mathbf{L}} \textbf{1} - ((\textbf{1}^T \tilde{\mathbf{K}}) (\tilde{\mathbf{L}} \textbf{1}))\textbf{1} \nonumber
\end{align}
where $\odot$ denotes the Hadamard product. 
Then $R_{XYXZ}$ in Eq.~\eqref{jmlr2016:eq:dependency_theory:term_Rxyxz} can be computed as $R_{XYXZ} = (4m)^{-1}(m-1)_3^{-2} \mathbf{h_{XY}}^T \mathbf{h_{XZ}}$.  Using the order of operations implied by the parentheses in Eq.~\eqref{jmlr2016:eq:dependency_theory::compute_hxy}, the computational cost of the cross covariance term is  $\mathcal{O}(m^2)$.  Combining this with the unbiased estimator of HSIC in Eq.~\eqref{jmlr2016:eq:background:unbiaisedHSIC} leads to a final computational complexity of $\mathcal{O}(m^2)$. Code for performing the test is available at \wachacodeurl.

In addition to the asymptotic consistency result, we provide a finite sample bound on the deviation between the difference of two population HSIC statistics and the difference of two empirical HSIC estimates.  
\begin{theorem}[Generalization bound on the difference of empirical HSIC statistics]
Assume that $k$, $l$, and $d$ are bounded almost everywhere by 1, and are non-negative. Then for $m>1$ and all $\delta > 0$ with probability at least $1-\delta$, for all $\mathbbP_{xy}$ and $\mathbbP_{xz}$, the generalization bound on the difference of empirical HSIC statistics is 
\begin{align}\label{jmlr2016:eq:dependency_theory:GeneralizationBound}
&\big|  \left( \HSICpop{F}{G}{xy} - \HSICpop{F}{H}{xz} \right)   - \left( \HSIC{F}{G}{X_m}{Y_m} \HSIC{F}{H}{X_m}{Z_m} \right) \big| \nonumber \\
& \qquad \leq 2 \left\lbrace \sqrt{\frac{\log(6/\delta)}{\alpha^2 m}} + \frac{C}{m} \right\rbrace 
\end{align}
where $\alpha>0.24$ and $C$ are constants.
\end{theorem}
\begin{proof}
In~\cite{GreBouSmoSch05} a finite sample bound is given for a single HSIC statistic. Eq.~\eqref{jmlr2016:eq:dependency_theory:GeneralizationBound} is proved by using a union bound:
\begin{align}
& \big| \left\lbrace \HSICpop{F}{G}{xy} - \HSICpop{F}{H}{xz}\right\rbrace  - 
\left\lbrace \HSIC{F}{G}{X_m}{Y_m} - \HSIC{F}{H}{X_m}{Z_m} \right\rbrace \big|  \nonumber \\
&= \big| \left\lbrace HSIC(\mathcal{F},\mathcal{G},P_{xy}) - \HSIC{F}{G}{X_m}{Y_m} \right\rbrace   + \left\lbrace \HSIC{F}{H}{X_m}{Z_m} - HSIC(\mathcal{F},\mathcal{H},P_{xz}) \right\rbrace \big| \nonumber \\
&\leq \big| HSIC(\mathcal{F},\mathcal{G},P_{xy}) - \HSIC{F}{G}{X_m}{Y_m} \big| 
+ \big| \HSIC{F}{H}{X_m}{Z_m} - HSIC(\mathcal{F},\mathcal{H},P_{xz}) \big| \nonumber \\
&\leq 2 \left\lbrace \sqrt{\frac{log(6/\delta)}{\alpha^2 m}} + \frac{C}{m} \right\rbrace \nonumber
\end{align}
\end{proof}
\begin{corollary}
$\HSIC{F}{G}{X_m}{Y_m} - \HSIC{F}{H}{X_m}{Z_m}$ converges to the population statistic at rate $\mathcal{O}(m^{1/2})$.
\end{corollary}
%
%
\subsection{A simple consistent test via uncorrelated HSICs}
\label{jmlr2016:subsec:relative_test_dependency:consistent_approach}
From the result in Eq.~\eqref{jmlr2016:eq:background:asymptoticdistributionHSIC}, a simple, consistent  test of relative dependence can be constructed as follows: split the samples from $\mathbbP_x$ into two equal sized sets denoted by $X'$ and $X''$, and drop the second half of  the sample pairs with $Y$ and the first half of the sample pairs with $Z$.
We will denote the remaining samples as $Y'$ and $Z''$.  We can now estimate the joint distribution of $m^{1/2}[\HSIC{F}{G}{X_{m/2}}{Y_{m/2}}, \HSIC{F}{H}{X_{m/2}}{Z_{m/2}}]^T$ as 
\begin{equation}
\mathcal{N} \left( 
\begin{pmatrix} 
\HSICpop{F}{G}{xy} \\ 
\HSICpop{F}{H}{xz}\end{pmatrix}, 
\begin{pmatrix} 
\sigma_{X'Y'}^{2} & 0 \\ 
0 & \sigma_{X''Z''}^2 
\end{pmatrix} 
\right), 
\end{equation}
which we will write as $\mathcal{N} \left( \mu' , \Sigma' \right)$.  Given this joint distribution, we need to 
determine the distribution over the half space defined by 
$\HSICpop{F}{G}{xy}<\HSICpop{F}{H}{xz}.$
As in the previous section, we  achieve this by rotating the distribution by $\frac{\pi}{4}$  counter-clockwise about the origin, and integrating the resulting distribution projected onto the first axis (cf.\ Fig.~\ref{jmlr2016:fig:dependency_experiments:empirical_HSIC_differentsamplesize}).  The resulting projection of the rotated distribution onto the primary axis is 
\begin{align}\label{jmlr2016:eq:dependency_theory:asymptotic_dist_dropHalf}
\mathcal{N} &\left( \left[ \mathbf{Q} \mu' \right]_{1}, \left[\mathbf{Q} \mathbf{\Sigma'} \mathbf{Q}^{T} \right]_{11} \right)
\end{align}
where 
\begin{align} \label{jmlr2016:eq:dependency_theory:variance_asymptotic_dist_dropHalf}
[ \mathbf{Q} \mu']_{1} = \frac{\sqrt{2}}{2} \left( \HSICpop{F}{G}{xy} - \HSICpop{F}{H}{xz} \right),
[\mathbf{Q} \mathbf{\Sigma'} \mathbf{Q}^{T}]_{11} = \frac{1}{2}(\sigma_{X'Y'}^2 + \sigma_{X''Z''}^2).
\end{align}
From this empirically estimated distribution, it is  straightforward to construct a consistent test (cf.\ Eq.\ \eqref{jmlr2016:eq:dependency_theory:pvalue}).  The power of this test varies inversely with the variance of the distribution in Eq.\ \eqref{jmlr2016:eq:dependency_theory:asymptotic_dist_dropHalf}.  
%
%
\subsection{The dependent test is more powerful}  \label{jmlr2016:subsec:dependency_theory:dependent_test_more_powerful}

While discarding half the samples leads to a consistent test, we might expect some loss of power over the approach in Section~\ref{jmlr2016:subsec:relative_test_dependency:joint_asympt_dist}, due to the increase in variance with lower sample size.
In this section, we prove the Section~\ref{jmlr2016:subsec:relative_test_dependency:joint_asympt_dist} test is more powerful than that
of Section~\ref{jmlr2016:subsec:relative_test_dependency:consistent_approach}, regardless of $\mathbbP_{xy}$ and $\mathbbP_{xz}$.
We call the simple and consistent approach in Section~\ref{jmlr2016:subsec:relative_test_dependency:consistent_approach}, the \emph{independent approach}, and the lower variance approach in Section~\ref{jmlr2016:subsec:relative_test_dependency:joint_asympt_dist}, the \emph{dependent approach}. The following theorem compares these approaches.
\begin{theorem} \label{jmlr2016:thm:dependency_theory:powerful_test}
The asymptotic relative efficiency (ARE) of the independent approach relative to the dependent approach is always greater than 1.
\end{theorem}
\begin{remark} The \textit{asymptotic relative efficiency} is defined in e.g.~\citet[Chap.\ 5, Section 1.15.4]{serfling2009approximation}.  If $m_A$ and $m_B$ are the sample sizes at which tests "perform equivalently" (i.e. have equal power), then the ratio $\frac{m_A}{m_B}$ represents the relative efficiency.  When $m_A$ and $m_B$ tend to $+\infty$ and the ratio $\frac{m_A}{m_B} \rightarrow L$ (at equivalent performance), then the value $L$ represents the asymptotic relative efficiency of procedure B relative to procedure A. This example is relevant to our case since we are comparing two test statistics with different asymptotically normal distributions.
\end{remark}

The following lemma is used for the proof of  Thm.~\ref{jmlr2016:thm:dependency_theory:powerful_test}.
\begin{lemma}(Lower Variance) \label{jmlr2016:lem:dependency_theory:lowerVariance} 
The variance of the dependent test statistic is smaller than the variance of the independent test statistic.
\end{lemma}
\begin{proof} From the convergence of moments in the application of the central limit theorem~\citep{bahr1965}, we have that $\sigma_{X'Y'}^2 = 2 \sigma_{XY}^2$.  Then the variance summary in Eq.~\eqref{jmlr2016:eq:dependency_theory:variance_asymptotic_dist_keepAll} is $\frac{1}{2}(\sigma_{XY}^2 + \sigma_{XZ}^2 - 2 \sigma_{XYXZ})$ and the variance summary in Eq.~\eqref{jmlr2016:eq:dependency_theory:variance_asymptotic_dist_dropHalf} is $
\frac{1}{2}(2\sigma_{XY}^2 + 2\sigma_{XZ}^2)$ where in both cases the statistic is scaled by $\sqrt{m}$.  We have that the variance of the independent test statistic is smaller than the variance of the dependent test statistic when
\begin{align}
\frac{1}{2} (\sigma_{XY}^2 + \sigma_{XZ}^2 - 2 \sigma_{XYXZ}) &< \frac{1}{2} (2\sigma_{XY}^2 + 2\sigma_{XZ}^2) \nonumber \\
\Longleftrightarrow - 2 \sigma_{XYXZ} &< \sigma_{XY}^2 + \sigma_{XZ}^2
\end{align}
which is implied by the positive definiteness of $\Sigma$.
\end{proof}
\begin{proof}[Proof of Thm~\ref{jmlr2016:thm:dependency_theory:powerful_test}]
The Type II error probability of the independent test at level $\alpha$ is
\begin{equation} \label{jmlr2016:eq:dependency_theory:indepTest_typeIIerror}
  \Phi \left[\Phi^{-1}(1-\alpha) - \dfrac{
   \begin{matrix}  m^{-1/2}\big(   \HSICpop{F}{G}{xy} - \HSICpop{F}{H}{xz} \big)\end{matrix}}
    {\sqrt{\sigma^2_{X'Y'}+\sigma^2_{X''Z''}}} \right], 
\end{equation}
where we again make the most conservative possible assumption that $\HSICpop{F}{G}{xy} - \HSICpop{F}{H}{xz}=0$ under the null. The Type II error probability of the dependent test at level $\alpha$ is
\begin{equation}
  \label{jmlr2016:eq:dependency_theory:depTest_typeIIerror}
  \Phi \left[ \Phi^{-1}(1-\alpha) -
\dfrac{
   \begin{matrix}  m^{-1/2}\big(  \HSICpop{F}{G}{xy} - \HSICpop{F}{H}{xz} \big)\end{matrix}}
    {\sqrt{\sigma^2_{XY}+\sigma^2_{XZ}-2\sigma_{XYXZ}}} \right]
\end{equation}
where $\Phi$ is the CDF of the standard normal distribution. 
The numerator in Eq.~\eqref{jmlr2016:eq:dependency_theory:indepTest_typeIIerror} is the same as the numerator in Eq.~\eqref{jmlr2016:eq:dependency_theory:depTest_typeIIerror}, and the denominator in Eq.~\eqref{jmlr2016:eq:dependency_theory:depTest_typeIIerror} is smaller due to Lemma~\ref{jmlr2016:lem:dependency_theory:lowerVariance}.  
The lower variance dependent test therefore has higher ARE, i.e.,\ for a sufficient sample size $m > \tau$ for some distribution dependent $\tau \in \mathbb{N}_+$, the dependent test will be more powerful than the independent test.
\end{proof}
%
%
\subsection{Generalizing to more than two HSIC statistics}\label{jmlr2016:subsec:dependency_theory:generalization_dependent_test}

The generalization of the dependence test to more than three random variables follows from the earlier derivation  by applying successive rotations to a higher dimensional joint Gaussian distribution over multiple HSIC statistics.  Given observations $X_1:= \lbrace x_1^1, ..., x^1_m\rbrace, ..., X_n:=\lbrace x_1^n, ..., x_m^n \rbrace$ i.i.d.\ from respectively $\mathbbP_{x_1}, ..., \mathbbP_{x_n}$, we denote by $f_1, ..., f_n$ the kernels associated uniquely with respective reproducing kernel Hilbert spaces $\mathcal{F}_1, ..., \mathcal{F}_n$.  We define a generalized statistical test, $\mathcal{T}_g: (\mathcal{X}^m \times ... \times \mathcal{X}^m) \rightarrow \lbrace 0,1 \rbrace$ to test the null hypothesis $\mathcal{H}_0$ : $\sum_{(x,y) \in \{1,\dots,n\}^2} v_{(x,y)} \operatorname{HSIC}\left[ \mathcal{F}_1, ..., \mathcal{F}_n, \mathbbP_{x_1...x_n}\right] \leq 0$ versus the alternative hypothesis $\mathcal{H}_n$ : $\sum_{(x,y) \in \{1,\dots,n\}^2} v_{(x,y)}  \operatorname{HSIC}\left[ \mathcal{F}_1, ..., \mathcal{F}_n, \mathbbP_{x_1...x_n}\right] > 0$, where $v$ is a vector of weights on each HSIC statistic.  We may recover the test in the previous section by setting $v_{(1,2)}=+1$ $v_{(1,3)}=-1$ and $v_{(i,j)}=0$ for all $(i,j) \in \{1,2,3\}^2\setminus \{(1,2),(1,3)\}$.

The derivation of the test follows the general strategy used in the previous section: we construct a rotation matrix so as to project the joint Gaussian distribution onto the first axis, and read the $p$-value from a standard normal table.  To construct the rotation matrix, we simply need to rotate $v$ such that it is aligned with the first axis.  Such a rotation can be computed by composing $n$ 2-dimensional rotation matrices as in Algorithm~\ref{jmlr2016:algo:dependency_theory:generalization_algo}.
\begin{algorithm}[h!]
   \caption{Successive rotation for generalized high-dimensional relative tests of dependency (cf.\ Section~\ref{jmlr2016:subsec:dependency_theory:generalization_dependent_test})}
 \label{jmlr2016:algo:dependency_theory:generalization_algo}
\begin{algorithmic}
\REQUIRE $v \in \mathbb{R}^n$
\ENSURE $[\mathbf{Q} v]_i = 0 \ \ \forall i \ne 1$, $\mathbf{Q}^T \mathbf{Q} = \mathbf{I}$
\STATE $\mathbf{Q}=I$
\FOR{$i=2$ {\bfseries to} $n$}
\STATE{$\mathbf{Q}_i = \mathbf{I}$; $\theta = - \tan^{-1} \frac{v_i}{[\mathbf{Q}v]_1}$}
\STATE{$[\mathbf{Q}_i]_{11} = \cos(\theta)$; $[\mathbf{Q}_i]_{1i} = -\sin(\theta)$}; {$[\mathbf{Q}_i]_{i1} = \sin(\theta)$; $[\mathbf{Q}_i]_{ii} = \cos(\theta)$}
\STATE{$\mathbf{Q} = \mathbf{Q}_i \mathbf{Q}$}
\ENDFOR
\end{algorithmic}
\end{algorithm}

\section{A relative test of similarity} \label{jmlr2016:sec:relative_test_similarity}
\label{jmlr2016:sec:similarity_theory:joint_asymtotic_dist_MMD}
%
%

In this section, we derive our statistical test for relative similarity as measured by MMD.  In order to maximize the statistical efficiency of the test, we will reuse samples from the reference distribution, denoted by $\mathbbP_x$, to compute the MMD estimates with two candidate distributions $\mathbbP_y$ and $\mathbbP_z$.
We consider two MMD estimates $\squaredMMDu{X_m}{Y_n}$ and $\squaredMMDu{X_m}{Z_r}$, and as the data sample $X_m$ is identical between them, these estimates will be correlated.  We therefore first derive the joint asymptotic distribution of these two metrics and use this to construct a statistical test.

\begin{theorem}\label{jmlr2016:thm:similarity_theory:joint_asymtotic_dist_MMD}
We assume that $\mathbbP_x \ne \mathbbP_y$, $\mathbbP_x \ne \mathbbP_z$,  $\E(k(x_i,x_j)) < \infty$,  $\E(k(y_i,y_j)) < \infty$ and $\E(k(x_i,y_j)) < \infty$, then 
\begin{equation}
m^{1/2} 
\left( \begin{pmatrix}
\squaredMMDu{X_m}{Y_n} \\ 
\squaredMMDu{X_m}{Z_r} 
\end{pmatrix}
-
\begin{pmatrix}
\squaredMMDpop{x}{y} \\ 
\squaredMMDpop{x}{z}
\end{pmatrix}  
\right)
\overset{d}{\longrightarrow} \mathcal{N} 
\left( 
\begin{pmatrix}
0 \\ 
0
\end{pmatrix},
\begin{pmatrix} 
\sigma_{XY}^2 & \sigma_{XYXZ} \\ 
\sigma_{XYXZ} & \sigma_{XZ}^2 
\end{pmatrix} 
\right)
\label{jmlr2016:eq:similarity_theory:joint_asymtotic_dist_MMD}
\end{equation}
\end{theorem}

We substitute the kernel MMD definition from Equation~\eqref{jmlr2016:eq:background:unbiaisedMMD}, expand the terms in the expectation, and determine their empirical estimates in order to compute the variances in practice.  The proof and additional details of the following derivations are given in Appendix~\ref{jmlr2016:appendix}. 

An empirical estimate of $\sigma_{XYXZ}$ in Eq.~\eqref{jmlr2016:eq:similarity_theory:joint_asymtotic_dist_MMD}, neglecting higher order terms, can be computed in $\mathcal{O}(m^2)$:
\begin{align}\label{jmlr2016:eq:dependency_theory:derivation_of_covariance}
\hat{\sigma}_{XYXZ} &\approx
\frac{1}{m(m-1)^2} \vecOnes^T \tilde{\matK}_{xx}\tilde{\matK}_{xx} \vecOnes - \left(\frac{1}{m(m-1)} \vecOnes^T \tilde{\matK}_{xx} \vecOnes \right)^2 \\
& \qquad - \left( \frac{1}{m(m-1)r} \vecOnes^T \tilde{\matK}_{xx} \matK_{xz} \vecOnes - \frac{1}{m^2(m-1)r}  \vecOnes^T \tilde{\matK}_{xx} \vecOnes \vecOnes^T \matK_{xz} \vecOnes \right) \nonumber \\
& \qquad - \left( \frac{1}{m(m-1)n} \vecOnes^T \tilde{\matK}_{xx} \matK_{xy} \vecOnes - \frac{1}{m^2(m-1)n} \vecOnes^T \tilde{\matK}_{xx} \vecOnes \vecOnes^T \matK_{xz} \vecOnes \right) \nonumber \\
& \qquad + \left( \frac{1}{mnr} \vecOnes^T \matK_{yx} \matK_{xz} \vecOnes - \frac{1}{m^2nr} \vecOnes^T \matK_{xy}\vecOnes \vecOnes^T \matK_{xz} \vecOnes\right) \nonumber
\end{align}
where $\vecOnes$ is a vector of all ones, while $\left[\tilde{\matK}_{xx} \right]_{ij} = (1- \delta_{ij})k(x_i,x_j) ,\left[\matK_{xy}\right]_{ij} = k(x_i,y_j)$ and $\left[\matK_{xz}\right]_{ij} = k(x_i,z_j)$ refer to the kernel matrices. Similarly, $\sigma_{XY}^2$ in Eq.~\eqref{jmlr2016:eq:similarity_theory:joint_asymtotic_dist_MMD} is constructed as in Eq.~\eqref{jmlr2016:eq:background:variance_asymptoticdistributionMMD}.

Based on the empirical distribution from Eq.~\eqref{jmlr2016:eq:similarity_theory:joint_asymtotic_dist_MMD}, we now derive the RelativeMMD test in a manner similar than Section~\ref{jmlr2016:subsec:relative_test_dependency:joint_asympt_dist}. The test statistic $\squaredMMDu{X_m}{Y_n}-\squaredMMDu{X_m}{Z_r}$ is used to compute the $p$-value $p$ for the standard normal distribution.  The test statistic is obtained by rotating the joint distribution (cf.\ Eq.~\eqref{jmlr2016:eq:similarity_theory:joint_asymtotic_dist_MMD}) by $\frac{\pi}{4}$ about the origin, and integrating the resulting projection on the first axis.  Denote the asymptotically normal distribution of $\sqrt{m}[\squaredMMDu{X_m}{Y_n};\: \squaredMMDu{X_m}{Z_r}]^T$ as $\mathcal{N}(\mu, \Sigma)$.  The resulting distribution from rotating by $\pi/4$ and projecting onto the primary axis is $\mathcal{N} \left( [\mathbf{R} \mu]_{1}, [\mathbf{R} \mathbf{\Sigma} R^{T}]_{11} \right)$ where
\begin{align}
[\mathbf{R} \mu]_{1} &= \frac{\sqrt{2}}{2} ( \squaredMMDu{X_m}{Y_n}-\squaredMMDu{X_m}{Z_r} ) \\
[\mathbf{R} \mathbf{\Sigma} \mathbf{R}^{T}]_{11}&= \frac{1}{2}(\sigma_{XY}^2 + \sigma_{XZ}^2 - 2 \sigma_{XYXZ}) \label{eq:denominator_of_the_TestStatistics}
\end{align}
with $R$ is the rotation  by $\pi/4.$  Then, the $p$-values for testing $\mathcal{H}_0$ versus $\mathcal{H}_1$ are
\begin{equation}
p \leq \Phi \left( -\frac{\squaredMMDu{X_m}{Y_n}-\squaredMMDu{X_m}{Z_r}}{\sqrt{\sigma_{XY}^2 + \sigma_{XZ}^2 - 2 \sigma_{XYXZ}}} \right)
\end{equation}
where $\Phi$ is the CDF of a standard normal distribution. We have made code for performing the test publicly available. An implementation and examples are available at \eugenecodeurl.

\section{Relation of the HSIC and the Relative MMD} \label{jmlr2016:sec:additionaldiscussions}

We propose two new statistical hypothesis tests using equivalent mathematical derivation and using two statistics that are link (as demonstrated in~\citet[Section 7.3]{gretton2012kernel}).
The MMD can be expressed as the distance in a RKHS $\mathcal{H}$ between mean embeddings as
\begin{equation}
\MMDpop{x}{y} = \Vert \mu_{\mathbbP_x} - \mu_{\mathbbP_y}\Vert_{\mathcal{H}}. 
\end{equation}
Let assume that $\mathcal{V}$ is an RKHS over $\mathcal{X} \times \mathcal{Y}$ with kernel $v \left( (x,y),(x',y') \right)$. If x and y are independent, then $\mu_{\mathbbP_{xy}} = \mu_{\mathbbP_x} \mu_{\mathbbP_y}$. Hence we may use $\Vert \mu_{\mathbbP_{xy}} - \mu_{\mathbbP_x} \mu_{\mathbbP_y} \Vert_{\mathcal{V}}$ as a measure of dependence.
Let assume that $v \left( (x,y),(x',y') \right) = k(x,x') l(y,y')$, then HSIC can be expressed as the distance in $\mathcal{V}$ between the mean embedding of the joint distribution and the product of the marginal distributions
\begin{equation}
\HSICpop{F}{G}{xy} = \Vert \mu_{\mathbbP_{xy}} -  \mu_{\mathbbP_x}\mu_{\mathbbP_y} \Vert_{\mathcal{V}}.
\end{equation}
Although, the application of MMD is related to the HSIC, they lead to very different type of expressions for the variance terms of their asymptotic distributions and this is the same for the Relative MMD and the Relative HSIC. The computation of the $p$-value for the Relative MMD is done using the computation of the kernel matrices $\mathbf{K}_{xy}$ and $\mathbf{K}_{xz}$, whereas for the Relative HSIC this is done using $\mathbf{K}_{xx}$, $\mathbf{K}_{yy}$ and $\mathbf{K}_{zz}$, showing the need of having two different implementations for each hypothesis statistical test. \\
However, we can compute $\mathbf{K}_{xy}$ as a transformation of $\mathbf{K}_{xx}$ and $\mathbf{K}_{yy}$ such that $[\mathbf{K}_{xy}]_{jj} = [\mathbf{K}_{xx} \odot (\mathbf{\Pi} \mathbf{K}_{yy} \mathbf{\Pi}^T ) ]_{jj}$ where $\mathbf{\Pi}$ is a permutation matrix.  But, because of this permutation matrix, we have that as the sample size $n$ increases, the computation complexity become $\mathcal{O}(n^4)$ instead of $\mathcal{O}(n^2)$.
\section{Experiments} \label{jmlr2016:sec:experiments}
We analyze the Relative HSIC and the Relative MMD in an extensive set of experiments. The section contains results on synthetics experiments and real data.  
%
%
\subsection{Experiments for the relative dependency test} \label{jmlr2016:subsec:dependency_experiments}

We apply our estimates of statistical dependence to three challenging problems. The first is a synthetic data experiment, in which we can directly control the relative degree of functional dependence between variates.  The second experiment uses a multilingual corpus to determine the relative relations between European languages.  The last experiment is a $3$-block dataset which combines gene expression, comparative genomic hybridization, and a qualitative phenotype measured on a sample of Glioma patients.
%
%
%
%
\subsubsection{Synthetic experiments} \label{jmlr2016:subsec:dependency_experiments:simulationStudies}

We constructed 3 distributions as defined in Eq.~\eqref{jmlr2016:subsec:dependency_experiments:eq:equation_simulationstudies} and illustrated in Fig.~\ref{jmlr2016:fig:dependency_experiments:illustration_simulationstudies}.  
\begin{figure*}
\centering
\begin{tabular}{ccccc}

\begin{subfigure}{.28\columnwidth}
\setlength{\tabcolsep}{0.1em}
\renewcommand{\arraystretch}{0.5}
\begin{tabular}{cc}
\begin{sideways} \quad \scalebox{0.8}{$ \sin(t) + \gamma_1 \mathcal{N}(0,1)$ }\end{sideways} & \includegraphics[width=\linewidth]{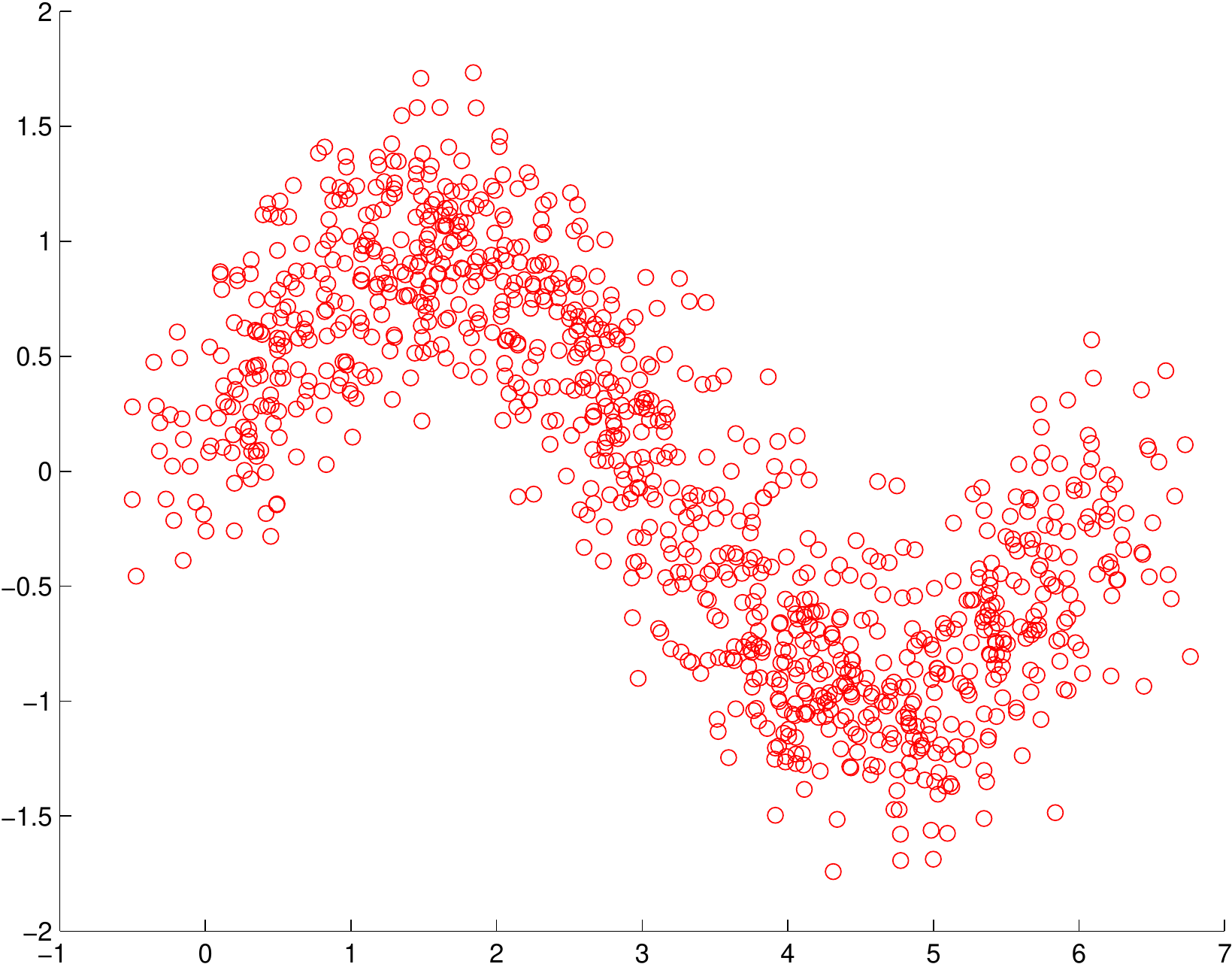} \\
& \scalebox{0.8}{$t + \gamma_1 \mathcal{N}(0,1)$}
\end{tabular}
\caption{$\gamma_1 = 0.3$} \label{jmlr2016:fig:dependency_experiments:illustration_a}
\end{subfigure}\hfill
&&
\begin{subfigure}{.28\textwidth}
\setlength{\tabcolsep}{0.1em}
\renewcommand{\arraystretch}{0.5}
\begin{tabular}{cc}
\begin{sideways} \quad \scalebox{0.8}{ $ t \sin(t) + \gamma_2 \mathcal{N}(0,1)$} \end{sideways}
& \includegraphics[width=\columnwidth]{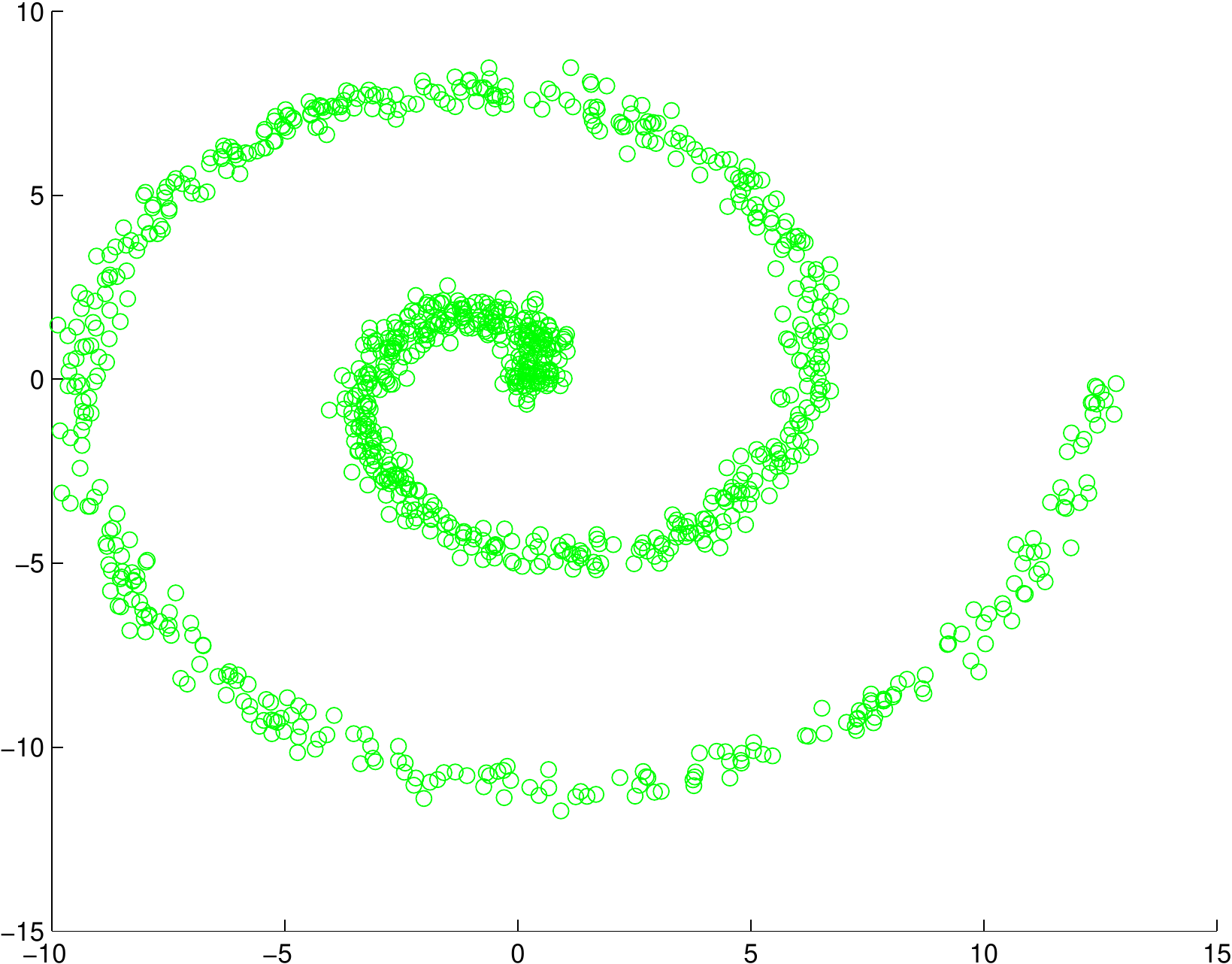} \\
& \scalebox{0.8}{$t \cos(t) + \gamma_2 \mathcal{N}(0,1)$}
\end{tabular}
\caption{$\gamma_2 = 0.3$} \label{jmlr2016:fig:dependency_experiments:illustration_b}
\end{subfigure}\hfill
&&
\begin{subfigure}{.28\textwidth}
\setlength{\tabcolsep}{0.1em}
\renewcommand{\arraystretch}{0.5}
\begin{tabular}{cc}
\begin{sideways} \quad \scalebox{0.8}{$ t \cos(t) + \gamma_3 \mathcal{N}(0,1)$} \end{sideways}
& \includegraphics[width=\columnwidth]{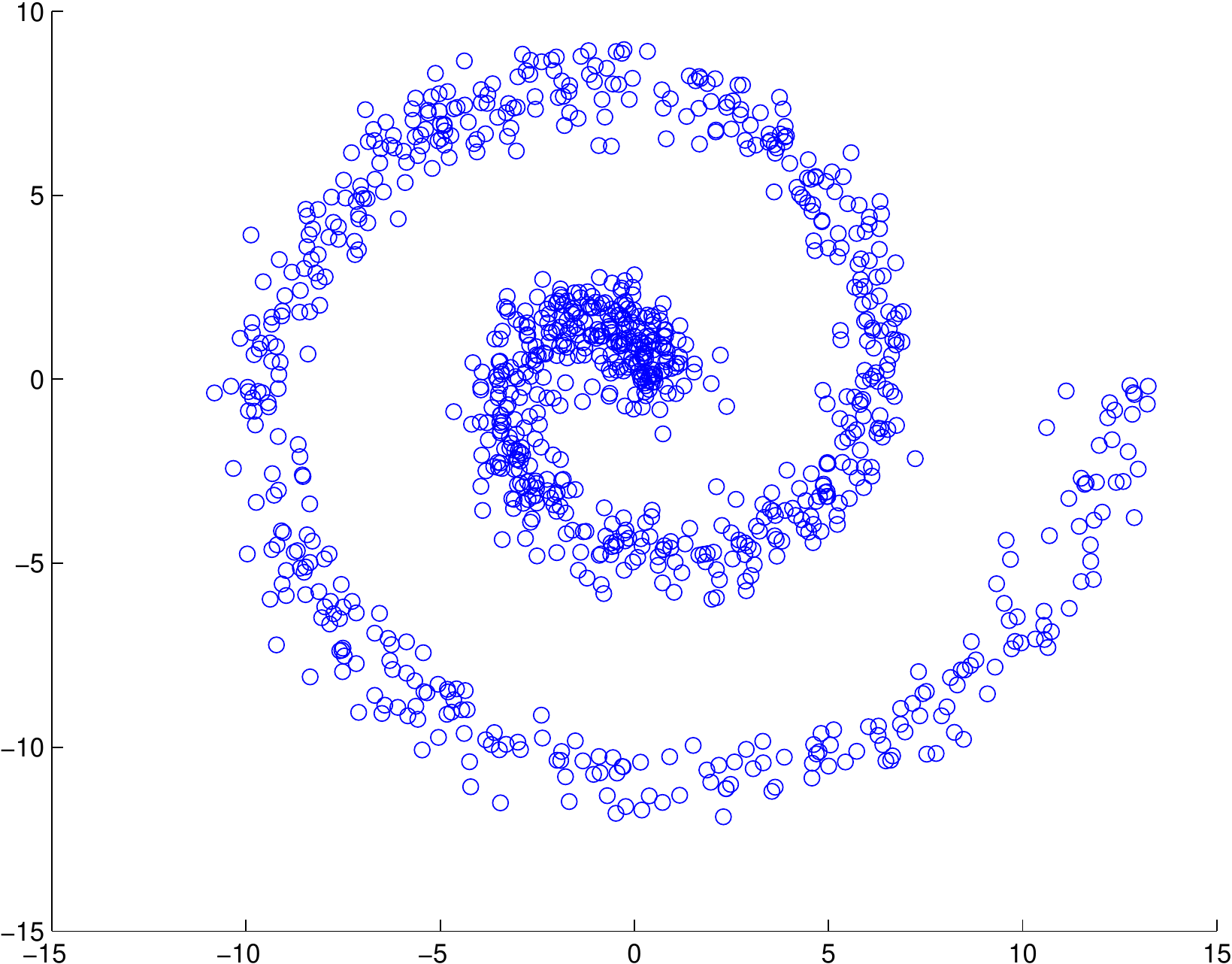} \\
& \scalebox{0.8}{$t \cos(t) + \gamma_3 \mathcal{N}(0,1)$}
\end{tabular}
\caption{$\gamma_3 = 0.6$} \label{jmlr2016:fig:dependency_experiments:illustration_c}
\end{subfigure}\par
\end{tabular}
\caption{ Illustration of a synthetic dataset sampled from the distribution in Eq.~\eqref{jmlr2016:subsec:dependency_experiments:eq:equation_simulationstudies}.}
   \label{jmlr2016:fig:dependency_experiments:illustration_simulationstudies}
\end{figure*}
\begin{align}
\label{jmlr2016:subsec:dependency_experiments:eq:equation_simulationstudies}
\mbox{Let } t &\sim  \mathcal{U}[(0,2\pi)], \\
   (a)\ x_1 &\sim t + \gamma_1 \mathcal{N}(0,1) \hspace{0.3cm}
   y_1 \sim \sin(t) + \gamma_1 \mathcal{N}(0,1)& \nonumber \\
   (b)\ x_2 &\sim t \cos(t) + \gamma_2 \mathcal{N}(0,1) \hspace{0.3cm}
   y_2 \sim  t \sin(t) + \gamma_2 \mathcal{N}(0,1)& \nonumber \\
   (c)\ x_3 &\sim  t \cos(t) + \gamma_3 \mathcal{N}(0,1) \hspace{0.3cm}
   y_3 \sim  t \sin(t) + \gamma_3 \mathcal{N}(0,1)& \nonumber    
\end{align}
These distributions are specified so that we can control the relative degree of functional dependence between the variates by varying the relative size of noise scaling parameters $\gamma_1$, $\gamma_2$ and $\gamma_3$. 
The question is then whether the dependence between (a) and (b) is larger than the dependence between (a) and (c). 
In these experiments, we fixed $\gamma_1 = \gamma_2 = 0.3$, while we varied $\gamma_3$, and used a Gaussian kernel with bandwidth $\sigma$ selected as the median pairwise distance between data points.  This kernel is sufficient to obtain good performance, although other choices exist~\citep{NIPS2012_4727}.  

Fig.~\ref{jmlr2016:fig:dependency_experiments:powerofthetest} shows the power of the dependent and the independent tests as we vary $\gamma_3$. 
It is clear from these results that the dependent test is far more powerful than the independent test over the great majority of $\gamma_3$ values considered.
 Fig.~\ref{jmlr2016:fig:dependency_experiments:empirical_HSIC_differentsamplesize} demonstrates that this superior test power arises due to the tighter and more concentrated distribution
of the dependent statistic.
\begin{figure}[ht]
\centering
\setlength{\tabcolsep}{0.1em}
\renewcommand{\arraystretch}{0.5}
\begin{tabular}{cc}
\begin{sideways} \scalebox{0.8}{$\qquad \qquad$ Power of the tests} \end{sideways} & \includegraphics[width=0.45\textwidth]{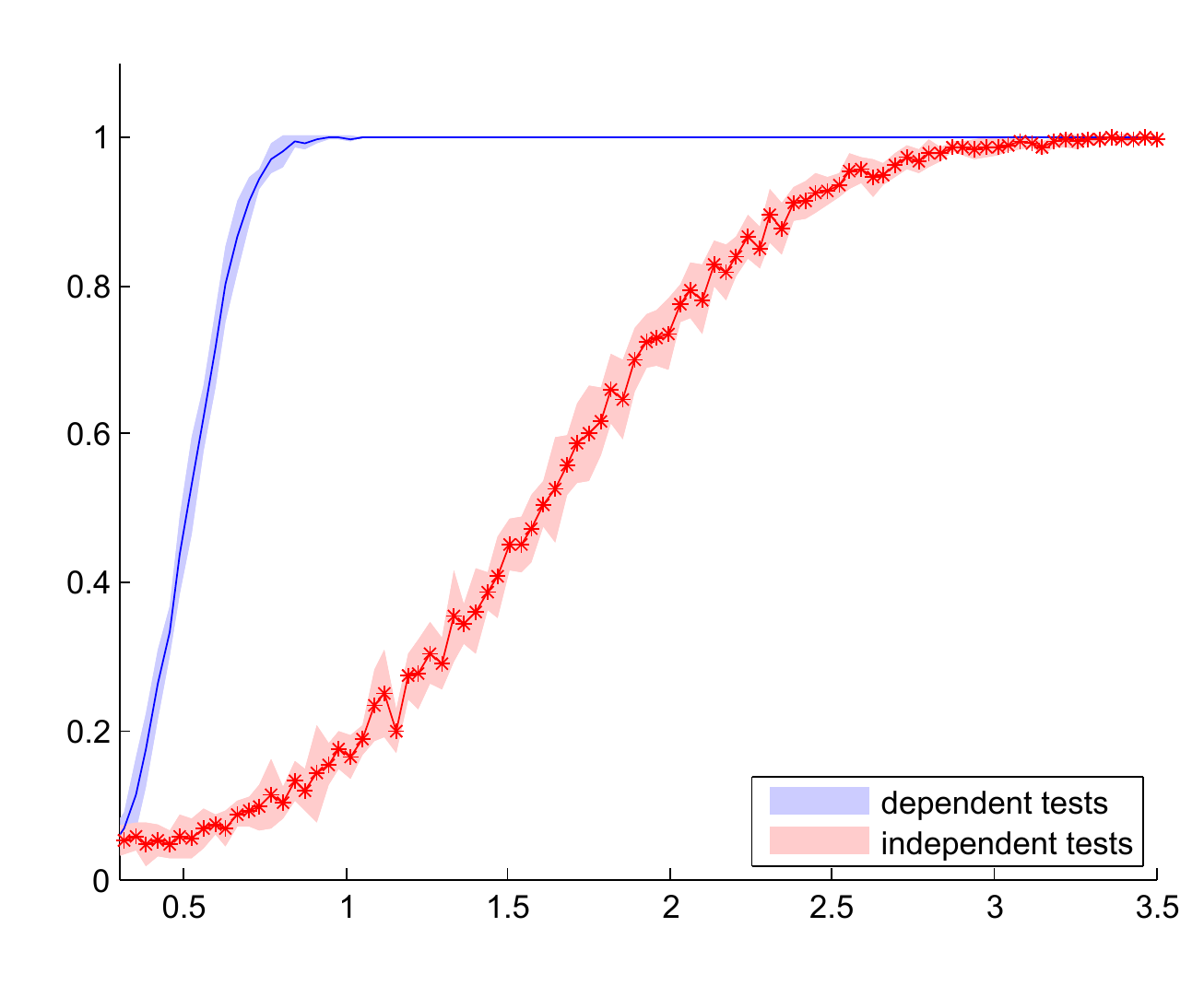} \\
& \scalebox{0.8}{$\gamma_3$}
\end{tabular}
\caption{Power of the dependent and independent test as a function of $\gamma_3$ on the synthetic data described in Section~\ref{jmlr2016:subsec:dependency_experiments:simulationStudies}.  For values of $\gamma_3>0.3$ the distribution in Fig.~\ref{jmlr2016:fig:dependency_experiments:illustration_a} is closer to Fig.~\ref{jmlr2016:fig:dependency_experiments:illustration_b} than to \ref{jmlr2016:fig:dependency_experiments:illustration_c}. The problem becomes difficult as $\gamma_3\rightarrow 0.3$. As predicted by theory, the dependent test is significantly more powerful over almost all values of $\gamma_3$ by a substantial margin. }
\label{jmlr2016:fig:dependency_experiments:powerofthetest}
\end{figure}	
\begin{figure*}
\centering
\begin{tabular}{ccccc}
\begin{subfigure}{.28\columnwidth}
\setlength{\tabcolsep}{-1.5pt}
\renewcommand{\arraystretch}{0.4}
\begin{tabular}{cc}
\begin{sideways} \qquad \scalebox{0.6}{$ \HSIC{F}{G}{X}{Z}$ }\end{sideways} 
& \includegraphics[width=\linewidth]{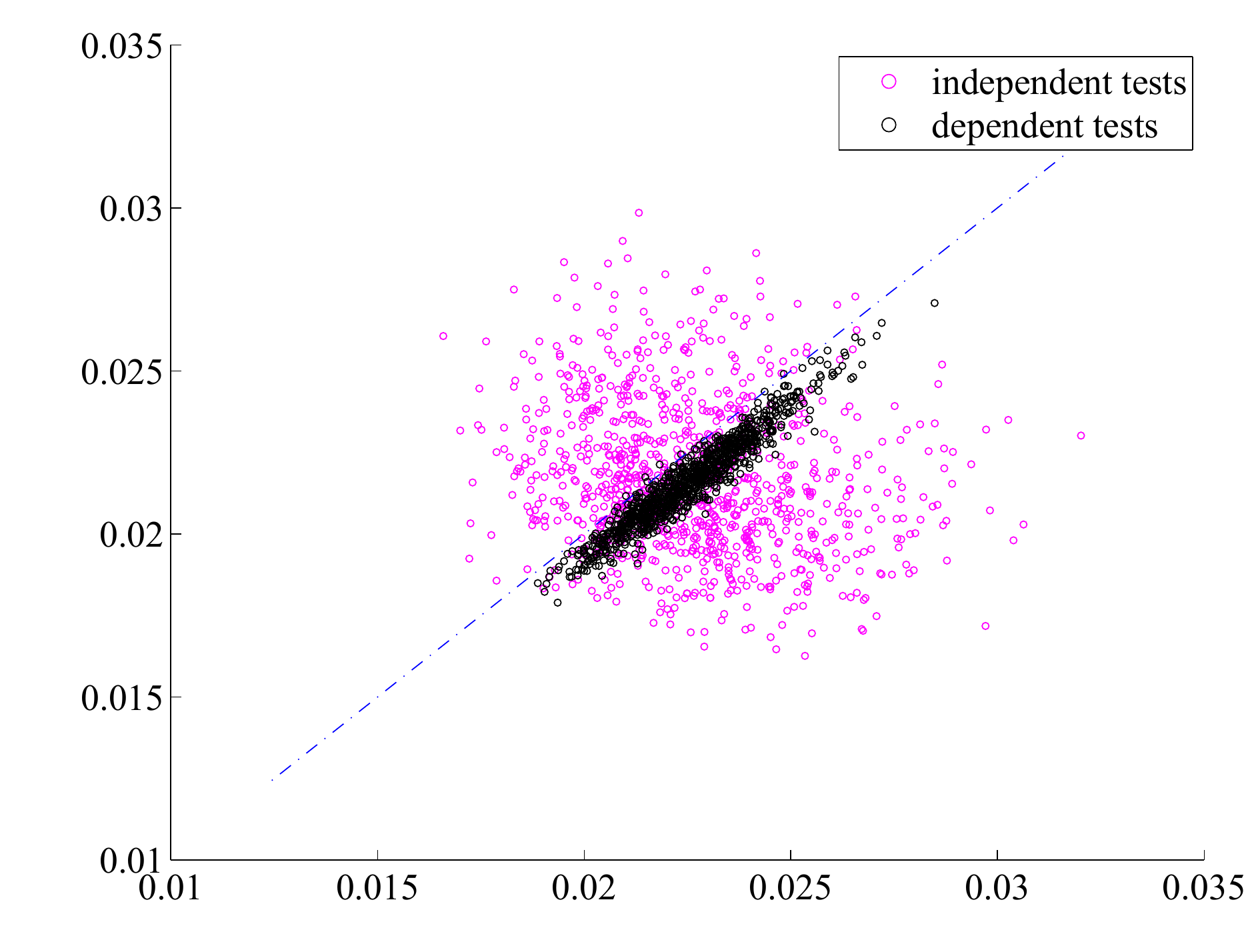} \\
& \scalebox{0.6}{$\HSIC{F}{G}{X}{Y}$}
\end{tabular}
\caption{$m=300, \gamma_3=0.7, \\ p_{\textrm{dep}} = 0.0189, p_{\textrm{indep}} = 0.3492$} 
\end{subfigure}\hfill
&&
\begin{subfigure}{.28\textwidth}
\setlength{\tabcolsep}{-1.5pt}
\renewcommand{\arraystretch}{0.4}
\begin{tabular}{cc}
\begin{sideways} \qquad \scalebox{0.6}{ $ \HSIC{F}{G}{X}{Z}$} \end{sideways}
& \includegraphics[width=\columnwidth]{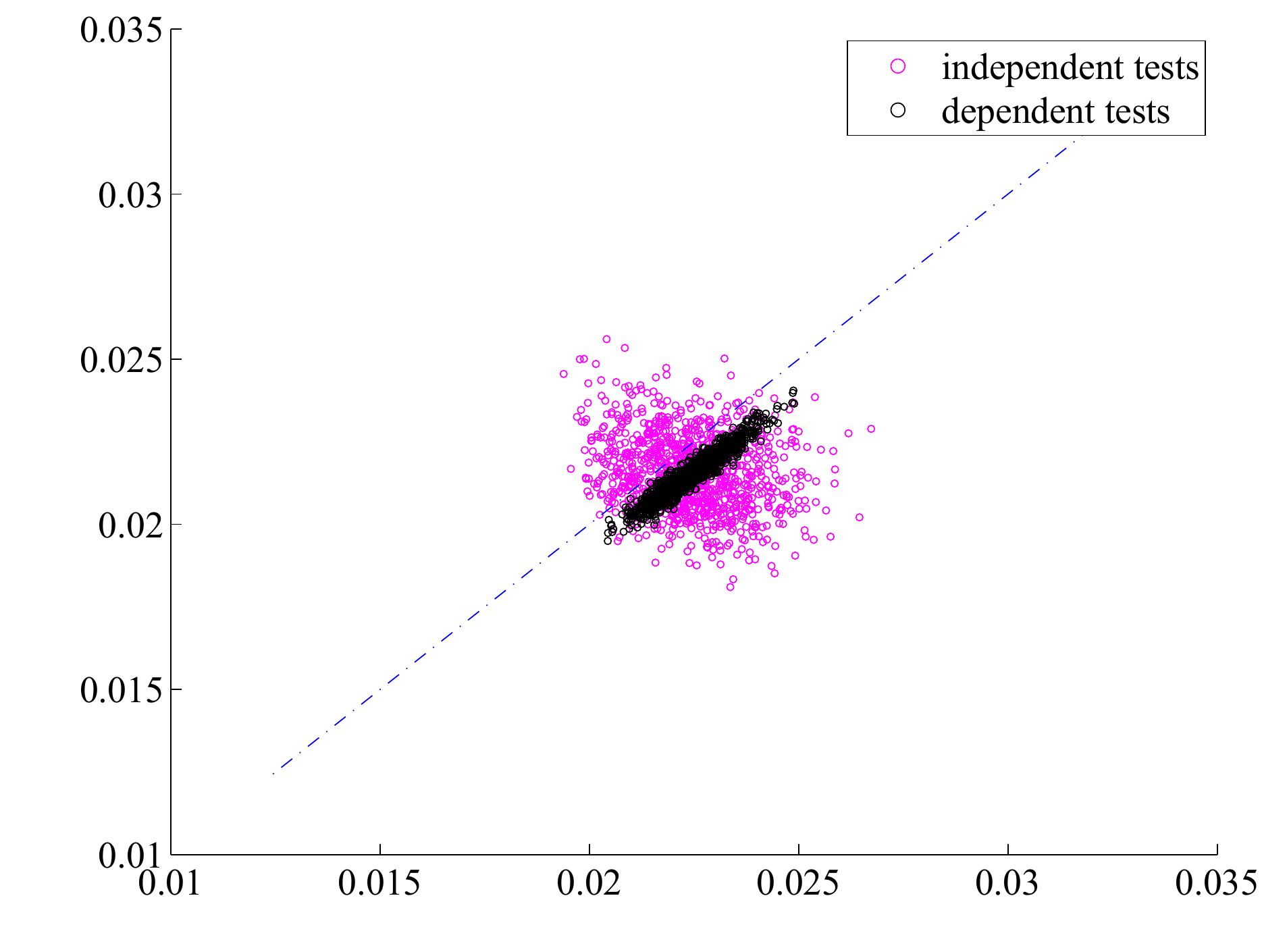} \\
& \scalebox{0.6}{$\HSIC{F}{G}{X}{Y}$}
\end{tabular}
\caption{$m=1000, \gamma_3=0.7, \\ p_{\textrm{dep}} = 10^{-4}, p_{\textrm{indep}} = 0.3690$} 
\end{subfigure}\hfill
&&
\begin{subfigure}{.28\textwidth}
\setlength{\tabcolsep}{-1.5pt}
\renewcommand{\arraystretch}{0.4}
\begin{tabular}{cc}
\begin{sideways} \qquad \scalebox{0.6}{$ \HSIC{F}{G}{X}{Z}$} \end{sideways}
& \includegraphics[width=\columnwidth]{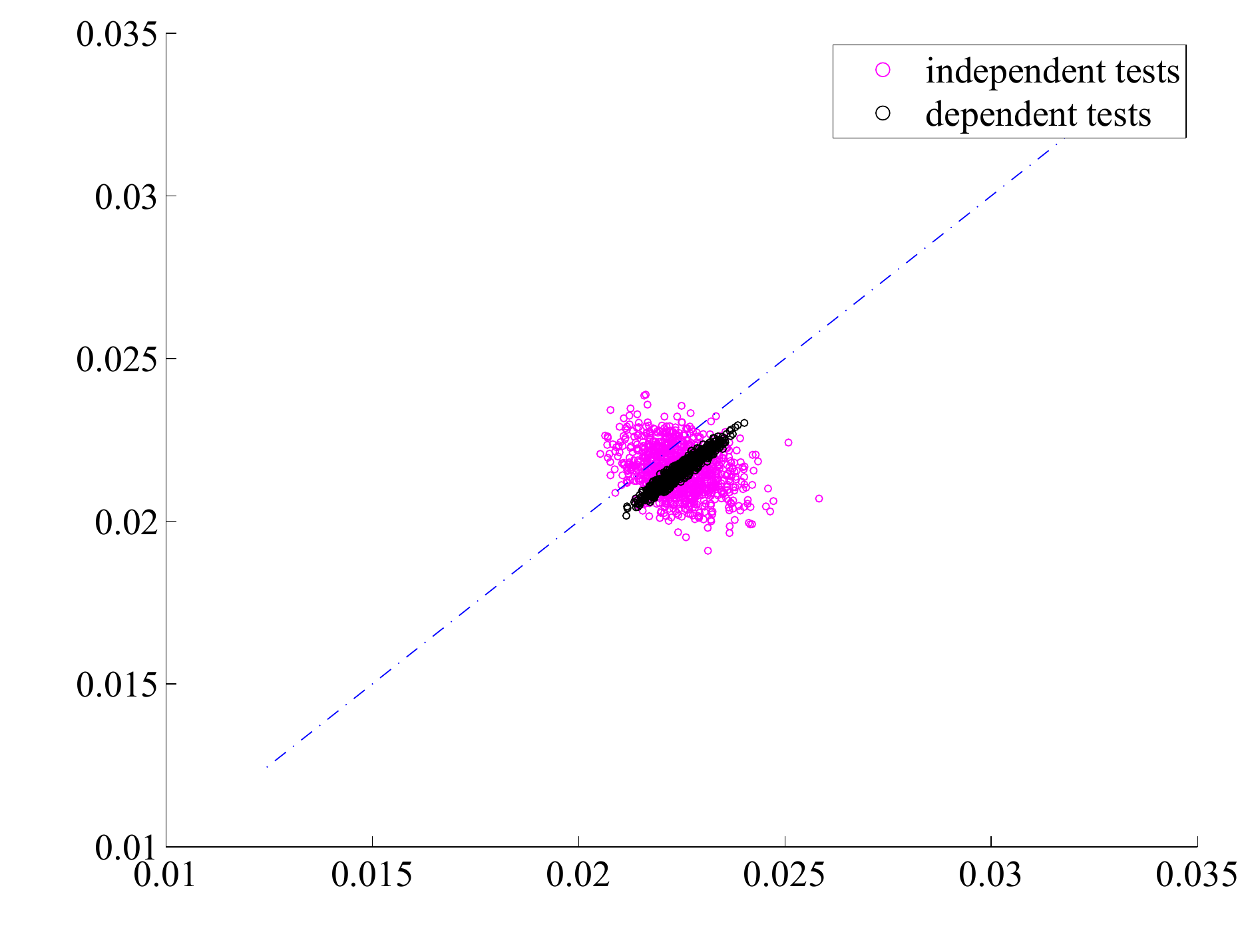} \\
& \scalebox{0.6}{$\HSIC{F}{G}{X}{Y}$} \\
\end{tabular}
\caption{$m=3000, \gamma_3=0.7,\\ p_{\textrm{dep}} = 10^{-6}, p_{\textrm{indep}} = 0.2876$} 
\end{subfigure}\hfill
\\
\begin{subfigure}{.28\columnwidth}
\setlength{\tabcolsep}{-1.5pt}
\renewcommand{\arraystretch}{0.4}
\begin{tabular}{cc}
\begin{sideways} \qquad \scalebox{0.6}{$ \HSIC{F}{G}{X}{Z}$ }\end{sideways} & \includegraphics[width=\linewidth]{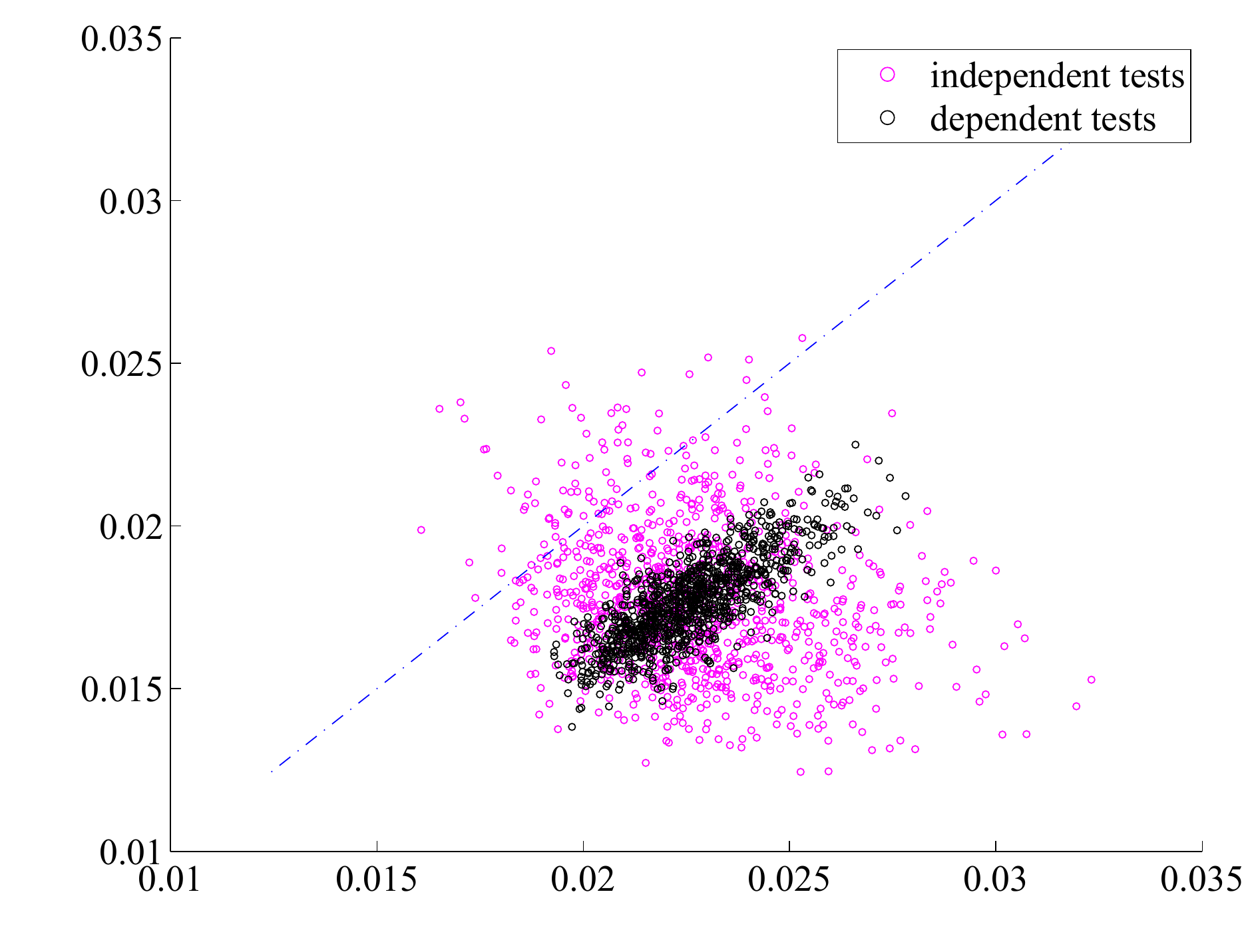} \\
& \scalebox{0.6}{$\HSIC{F}{G}{X}{Y}$}
\end{tabular}
\caption{$m=300, \gamma_3=1.7, \\ p_{\textrm{dep}} = 10^{-9}, p_{\textrm{indep}} = 0.9820$}
\end{subfigure}\hfill
&&
\begin{subfigure}{.28\textwidth}
\setlength{\tabcolsep}{-1.5pt}
\renewcommand{\arraystretch}{0.4}
\begin{tabular}{cc}
\begin{sideways} \qquad \scalebox{0.6}{ $ \HSIC{F}{G}{X}{Z} $} \end{sideways}
& \includegraphics[width=\columnwidth]{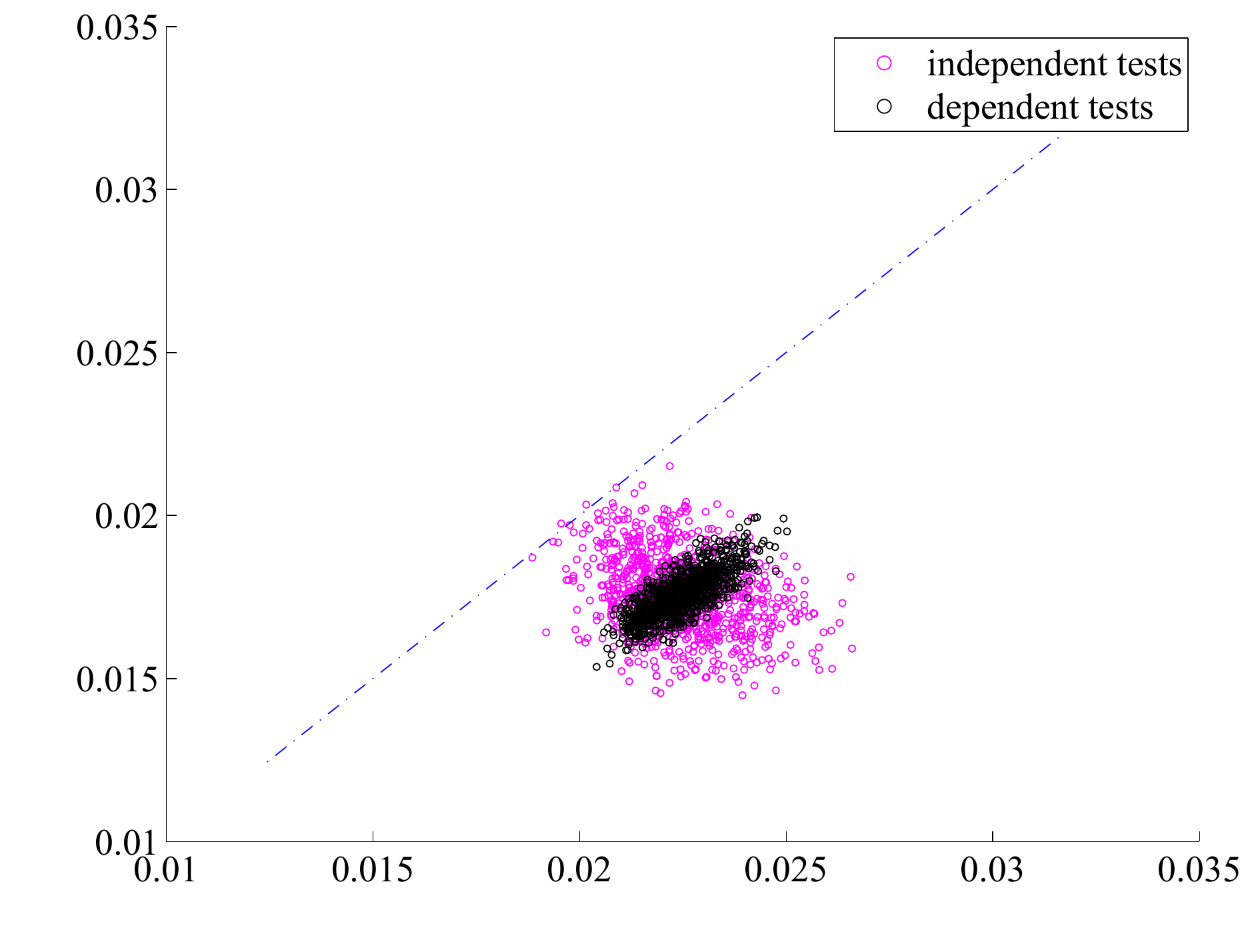} \\
& \scalebox{0.6}{$\HSIC{F}{G}{X}{Y}$}
\end{tabular}
\caption{$m=1000, \gamma_3=1.7,\\ p_{\textrm{dep}} = 10^{-10}, p_{\textrm{indep}} = 0.0326$}
\end{subfigure}\hfill
&&
\begin{subfigure}{.28\textwidth}
\setlength{\tabcolsep}{-1.5pt}
\renewcommand{\arraystretch}{0.4}
\begin{tabular}{cc}
\begin{sideways} \qquad \scalebox{0.6}{$ \HSIC{F}{G}{X}{Z} $} \end{sideways}
& \includegraphics[width=\columnwidth]{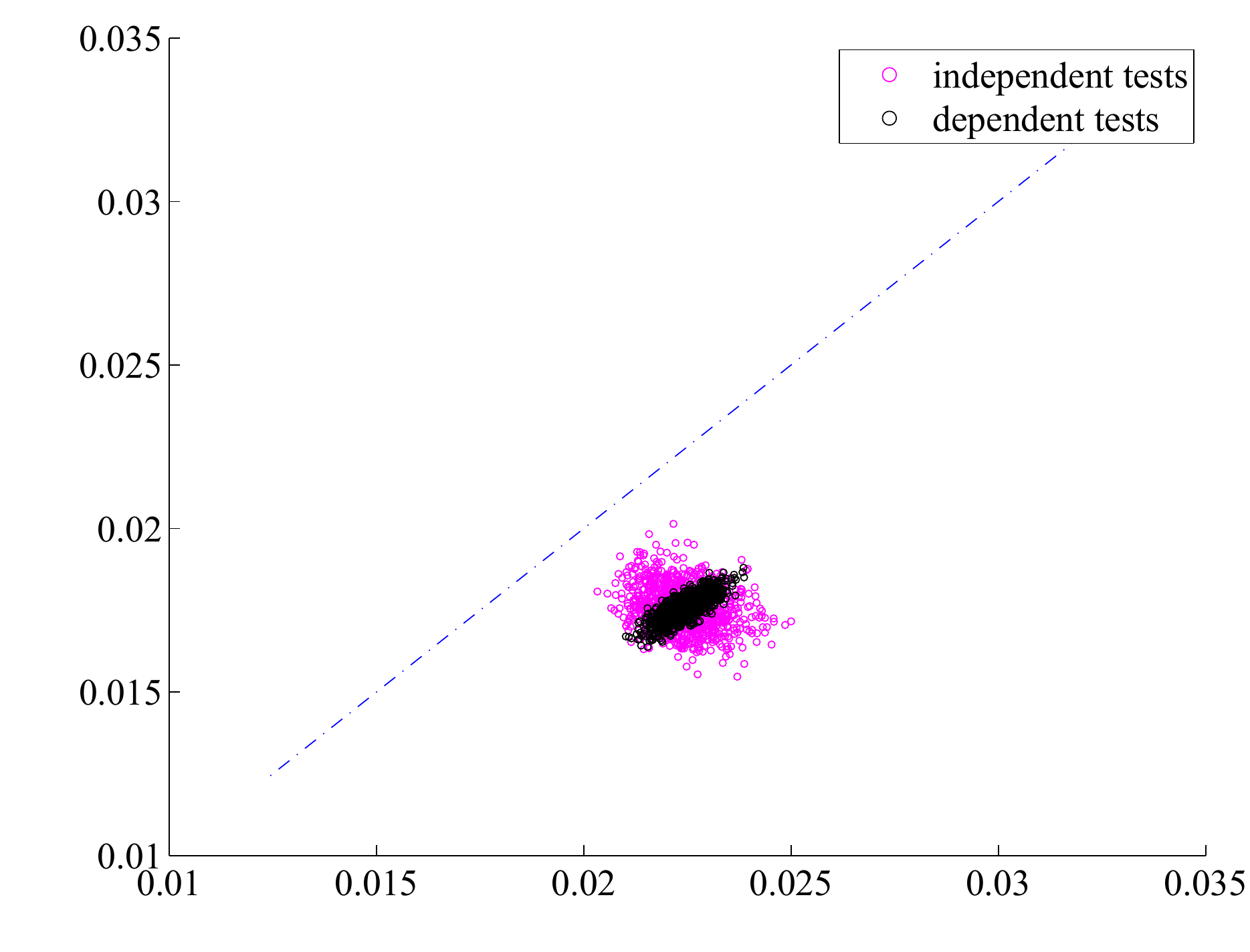} \\
& \scalebox{0.6}{$\HSIC{F}{G}{X}{Y}$} \\
\end{tabular}
\caption{$m=3000, \gamma_3=1.7, \\ p_{\textrm{dep}} = 10^{-13}, p_{\textrm{indep}} = 0.005$} 
\end{subfigure}\par
\end{tabular}
\caption{ For the synthetic experiments described in Section~\ref{jmlr2016:subsec:dependency_experiments:simulationStudies}, we plot empirical HSIC values for dependent and independent tests for 100 repeated draws with different sample sizes.  Empirical $p$-values for each test show that the dependent distribution converges faster than the independent distribution even at low sample size, resulting in a more powerful statistical test.}
   \label{jmlr2016:fig:dependency_experiments:empirical_HSIC_differentsamplesize}
\end{figure*}
%
%
%
%
%
\subsubsection{Multilingual data} \label{jmlr2016:subsec:dependency_experiments_multilingualdata}

In this section,  we demonstrate dependence testing to predict the relative similarity of different languages.  We use a real world dataset taken from the parallel European Parliament corpus~\cite{koehn2005europarl}.
We choose 3000 random documents in common written in: Finnish (fi), Italian (it), French (fr), Spanish (es), Portuguese (pt), English (en), Dutch (nl), German (de), Danish (da) and Swedish (sv).  These languages can be broadly categorized into either the Romance, Germanic or Uralic groups \citep{gray2003language}.  In this dataset, we considered each language as a random variable and each document as an observation. 

Our first goal is to test if the statistical dependence between two languages in the same group is greater than the statistical dependence between languages in different groups.  
For pre-processing, we removed stop-words (\url{http://www.nltk.org}) and performed stemming (\url{http://snowball.tartarus.org}).  We applied the TF-IDF model as a feature representation and used a Gaussian kernel with the bandwidth $\sigma$ set per language as the median pairwise distance between documents. 

In Tab.~\ref{jmlr2016:tab:dependency_experiments:selectioneuropeanlang}, a selection of tests between language groups (Germanic, Romance, and Uralic) is given:  all $p$-values strongly support that our relative dependence test finds the different language groups with very high significance. 

Further, if we focus on the Romance family, our test enables one to answer more fine-grained questions about the relative similarity of languages within the same group.  As before, we determine the ground truth similarities from the topology of the tree of European languages determined by the linguistics community~\cite{gray2003language,Bouckaert2012} as illustrated in Fig.~\ref{jmlr2016:fig:dependency_experiments:romance_tree} for the Romance group.  We have run the test on all triplets from the corpus for which the topology of the tree specifies a correct ordering of the dependencies. In a fraction of a second (excluding kernel computation), we are able to recover certain features of the subtree of relationships between languages present in the Romance language group  (Tab.~\ref{jmlr2016:tab:dependency_experiments:romance_lang}).  
The test always indicates the correct relative similarity of languages when nearby languages (e.g.\ Portuguese and Spanish: pt, es) are compared with those further away (e.g.\ Portuguese and Danish: pt, da),  however errors are sometimes made when comparing triplets of languages for which the nearest common ancestor is more than one link removed.

In our next tests, we evaluate our more general framework for testing relative dependencies with more than two HSIC statistics.  We chose four languages, and tested whether the average dependence between languages in the same group is higher than the dependence between groups.  The results of these tests are in Tab.~\ref{jmlr2016:tab:dependency_experiments:general_selection_europeanlang_pvalues}.  As before, our test is able to distinguish language groups with high significance.
\begin{figure}
\begin{minipage}{.45\textwidth}
\centering 
\begin{tabular}{|c|c|c|c|}
\hline  Source & Target 1 & Target 2 & $p$-value \\
\hline
\hline
 es & pt & fi & $\mathbf{0.0066}$ \\
 fr & it & da & $\mathbf{0.0418}$ \\
 it & es & fi & $\mathbf{0.0169}$ \\
 pt & es & da & $\mathbf{0.0173}$ \\
 de & nl & fi & $\mathbf{<10^{-4}}$ \\
 nl & en & es & $\mathbf{<10^{-4}}$ \\
 da & sv & fr & $\mathbf{<10^{-6}}$ \\
 sv & en & it & $\mathbf{<10^{-4}}$ \\
 en & de & es & $\mathbf{<10^{-4}}$ \\
\hline
\end{tabular}
\captionof{table}{A selection of relative dependency tests between two pairs of HSIC statistics for the multilingual corpus data. Low $p$-values indicate a source is closer to target 1 than to target 2. In all cases, the test correctly identifies that languages within the same group are more strongly related than those in different groups.}
\label{jmlr2016:tab:dependency_experiments:selectioneuropeanlang}
\end{minipage} \hfill
\begin{minipage}{.45\textwidth}
\centering
\begin{tabular}{|c|c|c|}
\hline  Source & Targets & $p$-value \\
\hline \hline
da & de  sv  fi & $\mathbf{<10^{-9}}$ \\
da & sv  en  fr & $\mathbf{<10^{-9}}$ \\
de & sv  en  it & $\mathbf{<10^{-5}}$ \\
fr & it  es  sv & $\mathbf{<10^{-5}}$ \\
es & fr  pt  nl & $\mathbf{0.0175}$ \\
\hline
\end{tabular}
\captionof{table}{Relative dependency tests between Romance languages. The tests are ordered such that a low $p$-value corresponds with a confirmation of the topology of the tree of Romance languages determined by the linguistics community \cite{gray2003language}.}
\label{jmlr2016:tab:dependency_experiments:general_selection_europeanlang_pvalues}
\end{minipage}
\end{figure}
\begin{figure}
\begin{minipage}{.45\textwidth}
\centering 
\begin{tabular}{|c|c|c|c|}
\hline  Source & Target 1 & Target 2 & $p$-value \\
\hline
\hline
fr & es & it & $\mathbf{0.0157}$ \\
fr & pt & it & $0.1882$ \\
es & fr & it & $0.2147$ \\
es & pt & it & $\mathbf{<10^{-4}}$ \\
es & pt & fr & $\mathbf{<10^{-4}}$ \\
pt & fr & it & $0.7649$ \\
pt & es & it & $\mathbf{ 0.0011}$ \\
pt & es & fr & $\mathbf{<10^{-8}}$ \\
\hline
\end{tabular}
\captionof{table}{Relative dependency tests between Romance languages. The tests are ordered such that a low $p$-value corresponds with a confirmation of the topology of the tree of Romance languages determined by the linguistics community \cite{gray2003language}.}
\label{jmlr2016:tab:dependency_experiments:romance_lang}
\end{minipage} \hfill
\begin{minipage}{.45\textwidth}
\centering
\includegraphics[width=0.7\textwidth]{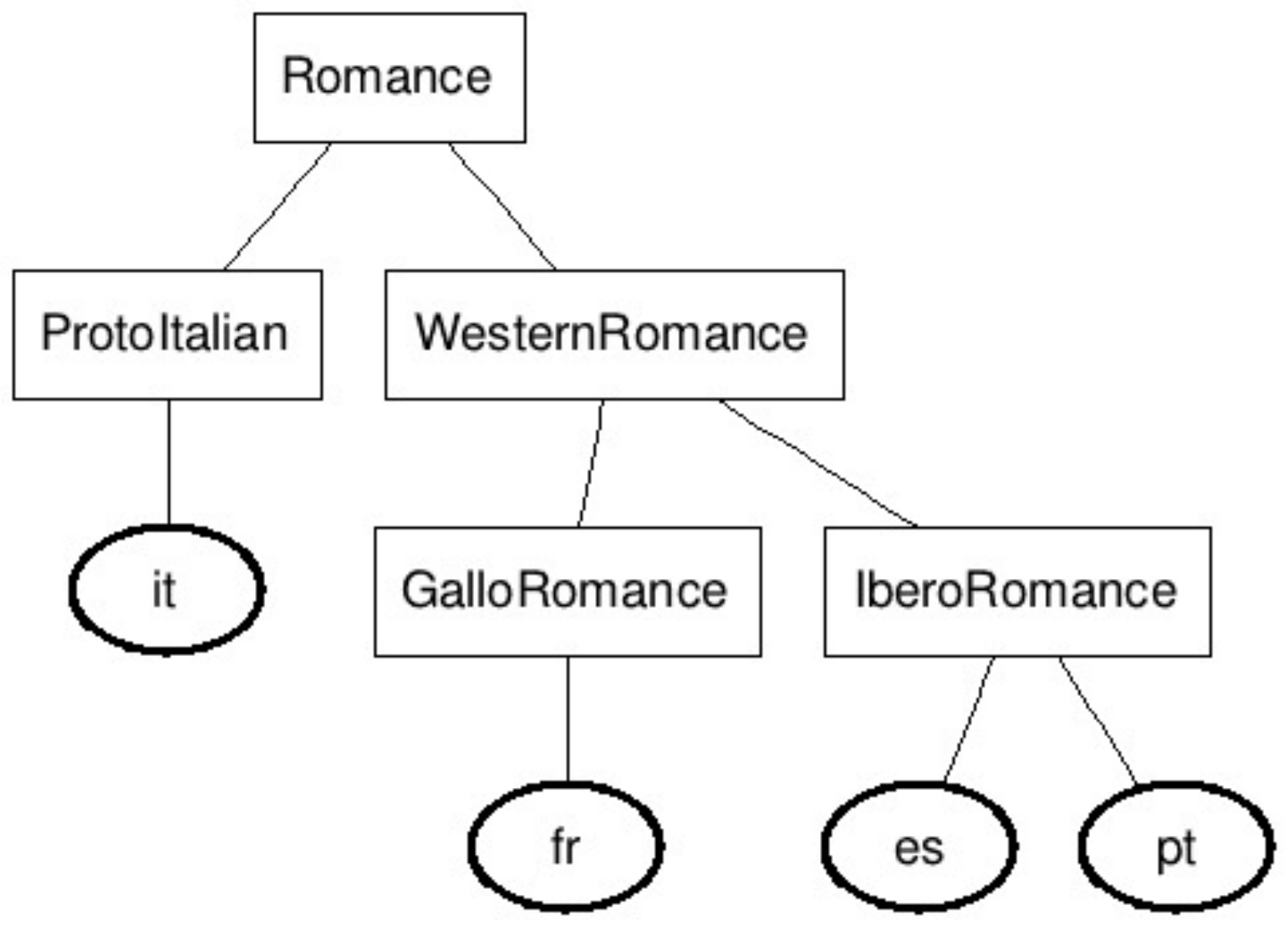}
\captionof{figure}{Partial tree of Romance languages adapted from~\cite{gray2003language}.}
\label{jmlr2016:fig:dependency_experiments:romance_tree}
\end{minipage}
\end{figure}
%
%
%
\subsubsection{Pediatric Glioma Data} \label{jmlr2016:subsec:dependency_experiments_gliomadata}

Brain tumors are the most common solid tumors in children and have the highest mortality rate of all pediatric cancers. Despite advances in multimodality therapy, children with pediatric high-grade gliomas (pHGG) invariably have an overall survival of around 20\% at 5 years. Depending on their location (e.g.\ brainstem, central nuclei, or supratentorial), pHGG present different characteristics in terms of radiological appearance, histology, and prognosis. The hypothesis is that pHGG have different genetic origins and oncogenic pathways depending on their location. Thus, the biological processes involved in the development of the tumor may be different from one location to another.

In order to evaluate such hypotheses, pre-treatment frozen tumor samples were obtained from 53 children with newly diagnosed pHGG from Necker Enfants Malades (Paris, France) from~\citet{puget2012mesenchymal}.  The 53 tumors are divided into 3 locations: supratentorial (HEMI), central nuclei (MIDL), and brain stem (DIPG).  The final dataset is organized in 3 blocks of variables defined for the 53 tumors: $\mathbf{X}$ is a block of indicator variables describing the location category, the second data matrix $\mathbf{Y}$ provides the expression of 15 702 genes (GE). The third data matrix $\mathbf{Z}$ contains the imbalances of 1229 segments (CGH) of chromosomes.  

For $\mathbf{X}$, we use a linear kernel, which is characteristic for indicator variables, and for $\mathbf{Y}$  and $\mathbf{Z}$, the kernel was chosen to be the Gaussian kernel with $\sigma$ selected as the median of pairwise distances.  The $p$-value of our relative dependency test is $<10^{-5}$. This shows that the tumor location in the brain is more dependent on gene expression than on chromosomal imbalances.  By contrast with Section~\ref{jmlr2016:subsec:dependency_experiments:simulationStudies}, the independent test was also able to find the same ordering of dependence, but with a $p$-value that is three orders of magnitude larger ($p=0.005$).  Fig.~\ref{jmlr2016:fig:dependency_experiments:twosigma_curve} shows iso-curves of the Gaussian distributions estimated in the independent and dependent tests. The empirical relative dependency is consistent with findings in the medical literature, and provides additional statistical support for the importance of tumor location in Glioma \citep{gilbertson2007,Palm2009,puget2012mesenchymal}. 
\begin{figure} \centering
\setlength{\tabcolsep}{0.1em}
\renewcommand{\arraystretch}{0.5}
\begin{tabular}{cc}
\begin{sideways} \qquad \qquad \qquad\scalebox{0.7}{$ \HSIC{F}{G}{X}{Z} $} \end{sideways}
&\includegraphics[width=0.45\textwidth]{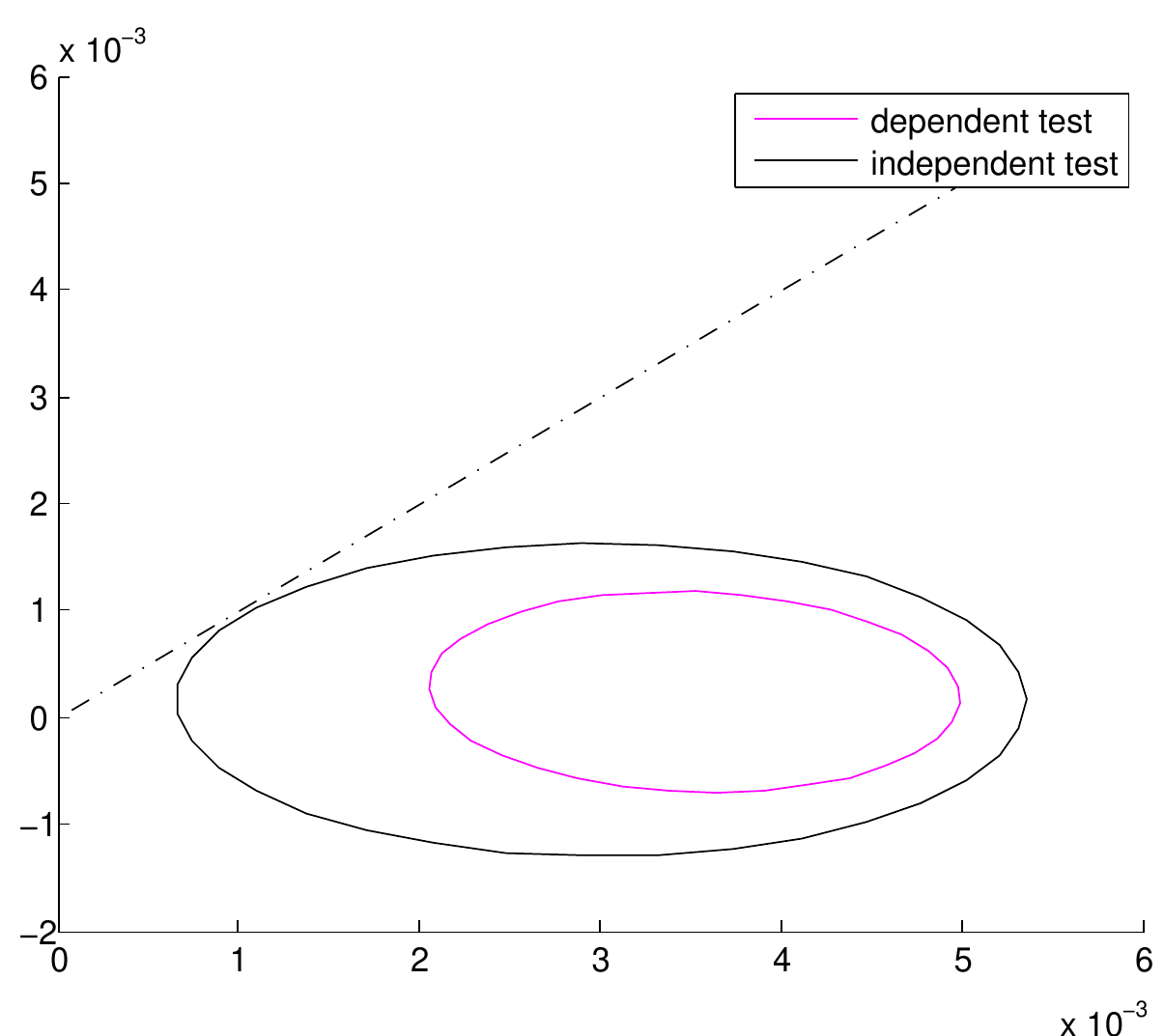} \\
& \scalebox{0.7}{$\HSIC{F}{G}{X}{Y}$}
\end{tabular}
\caption{$2\sigma$ iso-curves of the Gaussian distributions estimated from the pediatric Glioma data.  As before, the dependent test has a much lower variance than the independent test.  The tests support the stronger dependence on the tumor location to gene expression than chromosomal imbalances.}
\label{jmlr2016:fig:dependency_experiments:twosigma_curve}
\end{figure}	 
\subsection{Experiments for the relative similarity test} \label{jmlr2016:subsec:similarity_theory:similarity_experiments}

%
%
\subsubsection{Synthetic experiments} \label{jmlr2016:subsec:similarity_experiments:toyexample}
We verify the validity of the hypothesis test described above using a synthetic data set in which we can directly control the relative similarity between distributions. 

We constructed three Gaussian distributions as illustrated in Fig.~\ref{jmlr2016:fig:similarity_experiments:toy:illustration}.  These Gaussian distributions are specified with different means so that we can control the degree of relative  similarity between them.  The question is whether the similarity between $X$ and $Z$ is greater than the similarity between $X$ and $Y$.  In these experiments, we used a Gaussian kernel with bandwidth selected as the median pairwise distance between data points, and we fixed $\mu_Y = [-20,-20]$, $\mu_Z = [20,20]$ and varied $\mu_X$ such that $\mu_X = ( 1-\gamma)\mu_Y + \gamma \mu_Z$, for 41 regularly spaced values of $\gamma\in [0.1,\; 0.9]$ (avoiding the degenerate cases $P_x=P_y$ or $P_x=P_z$). 

Fig.~\ref{jmlr2016:fig:similarity_experiments:toy:synthetic_experiments_pvalues} shows the $p$-values of the relative similarity test for different distribution.  When $\gamma$ is varying around $0.5$, i.e., when $\squaredMMDu{X}{Y}$ is almost equal to $\squaredMMDu{X}{Z}$, the $p$-values quickly transition from $1$ to $0$, indicating strong discrimination of the test.  In Fig.~\ref{jmlr2016:fig:similarity_experiments:toy:powerofthetest}, we compare the power of our test to the power for both the dependent and independent approaches describes in Section~\ref{jmlr2016:subsec:relative_test_dependency:joint_asympt_dist} and in Section~\ref{jmlr2016:subsec:relative_test_dependency:consistent_approach} respectively.  Fig.~\ref{jmlr2016:fig:similarity_experiments:toy:isocurve_conservativetest} shows an empirical scatter plot of the pairs of MMD statistics along with a $2\sigma$ iso-curve of the estimated distribution, demonstrating that the parametric Gaussian distribution is well calibrated to the empirical values.

\begin{minipage}{.45\textwidth}\centering
\includegraphics[width=0.65\textwidth]{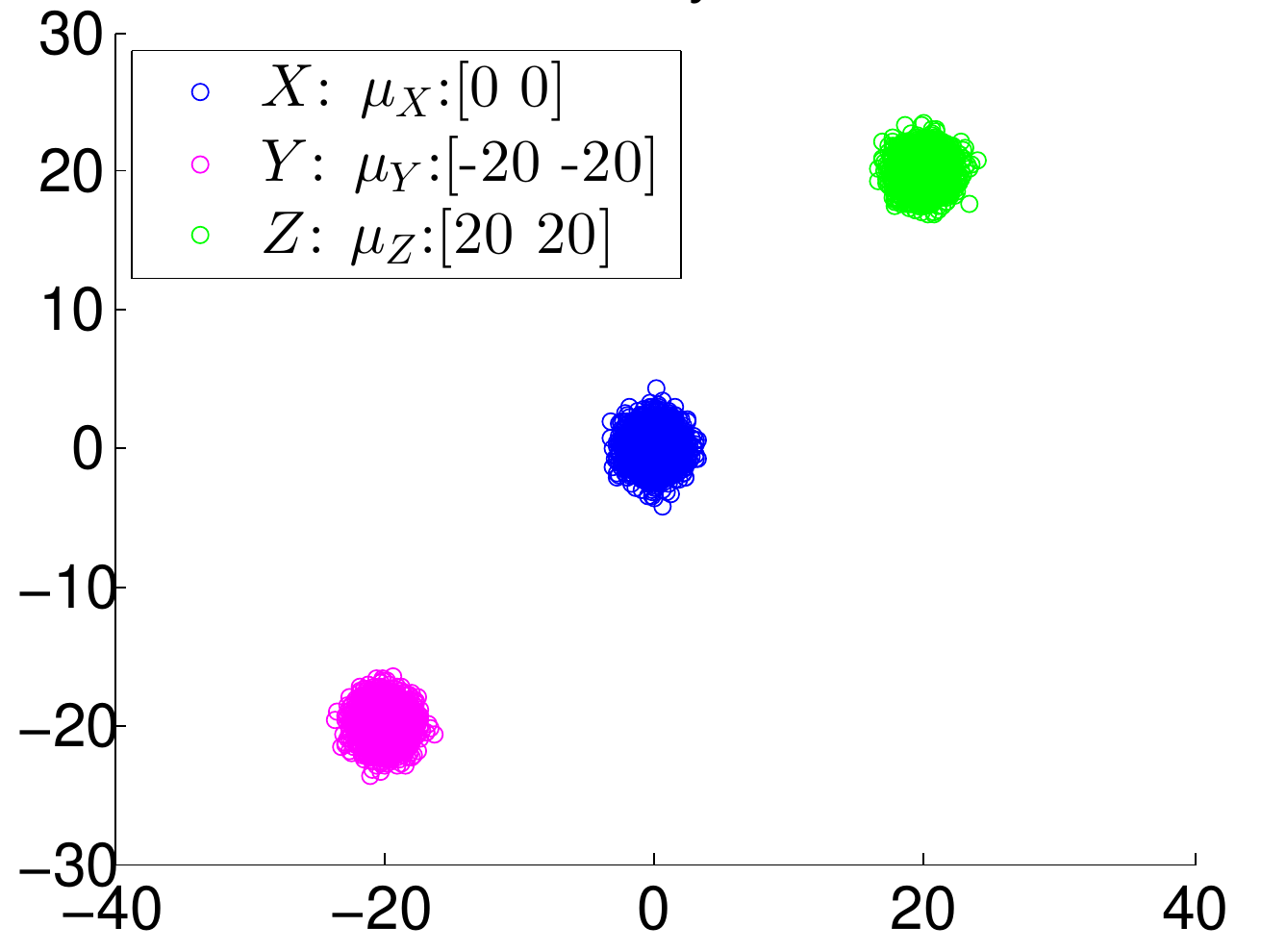}
\captionof{figure}{Illustration of the synthetic dataset where $X$, $Y$ and $Z$ are respectively Gaussian distributed with mean $\mu_X = [0,0]^T$, $\mu_Y=[-20,-20]^T$, $\mu_Z=[20,20]^T$ and with variance $\left(\protect\begin{smallmatrix}1&0\\0&1\protect\end{smallmatrix}\right)$.}
\label{jmlr2016:fig:similarity_experiments:toy:illustration}
\end{minipage}
\hfill
\begin{minipage}{.45\textwidth}\centering
\setlength{\tabcolsep}{0.1em}
\renewcommand{\arraystretch}{0.5}
\begin{tabular}{cc}
\begin{sideways}  \scalebox{0.8}{$\qquad$ Power of the tests}  \end{sideways} & \includegraphics[width=0.65\textwidth]{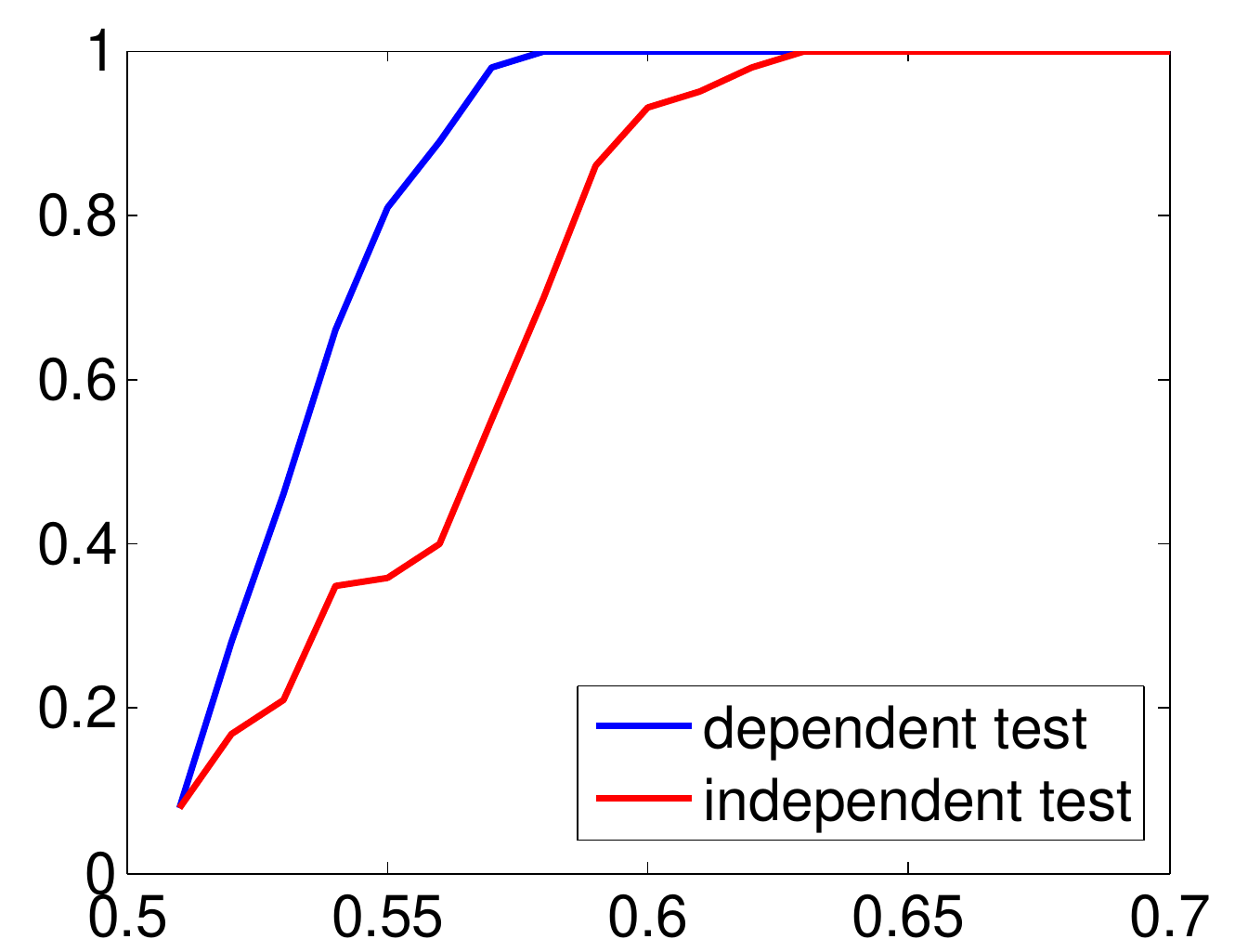} \\
& \scalebox{0.8}{$\gamma$}
\end{tabular}
\captionof{figure}{Comparison of the power of independent test and the dependent test as a function of $\gamma$.
}
\label{jmlr2016:fig:similarity_experiments:toy:powerofthetest}
\end{minipage}
\begin{minipage}{\textwidth}
  \begin{minipage}[b]{0.49\textwidth}
    \centering
        \setlength{\tabcolsep}{0.1em}
\renewcommand{\arraystretch}{0.5}
    \begin{tabular}{cc}
     \begin{sideways} $\qquad \quad$  \scalebox{0.8}{$p$-values} \end{sideways} & \includegraphics[width=.65\textwidth]{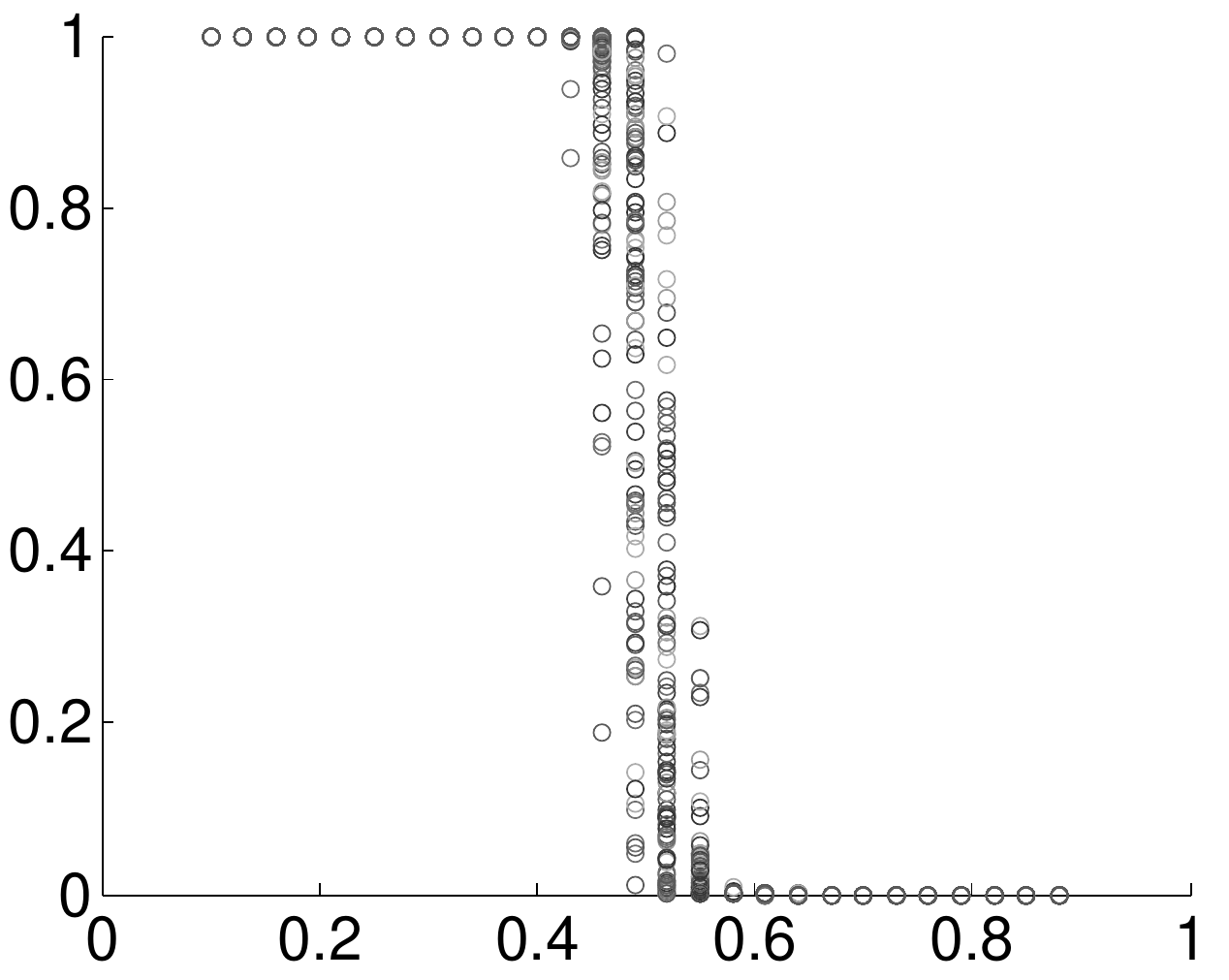} \\
     &  \scalebox{0.8}{$\gamma$} \\
     \end{tabular}
    \captionof{figure}{For $m=1000$, we fixed $\mu_Y = [-5,-5]$, $\mu_Z = [5,5]$ and varied $\mu_X$ such that $\mu_X = ( 1-\gamma)\mu_Y + \gamma \mu_Z$, for 41 regularly spaced values of $\gamma \in [0.1,\; 0.9]$ versus p-values for 100 repeated tests.}\label{jmlr2016:fig:similarity_experiments:toy:synthetic_experiments_pvalues}
  \end{minipage}
  \hfill
    \begin{minipage}[b]{0.49\textwidth}
    \centering 
        \setlength{\tabcolsep}{-1pt}
	\renewcommand{\arraystretch}{0.5}
    \begin{tabular}{cc}
     \begin{sideways} \scalebox{0.6}{$\qquad \qquad \squaredMMDu{X}{Z}$} \end{sideways} &
     \includegraphics[width=.65\textwidth]{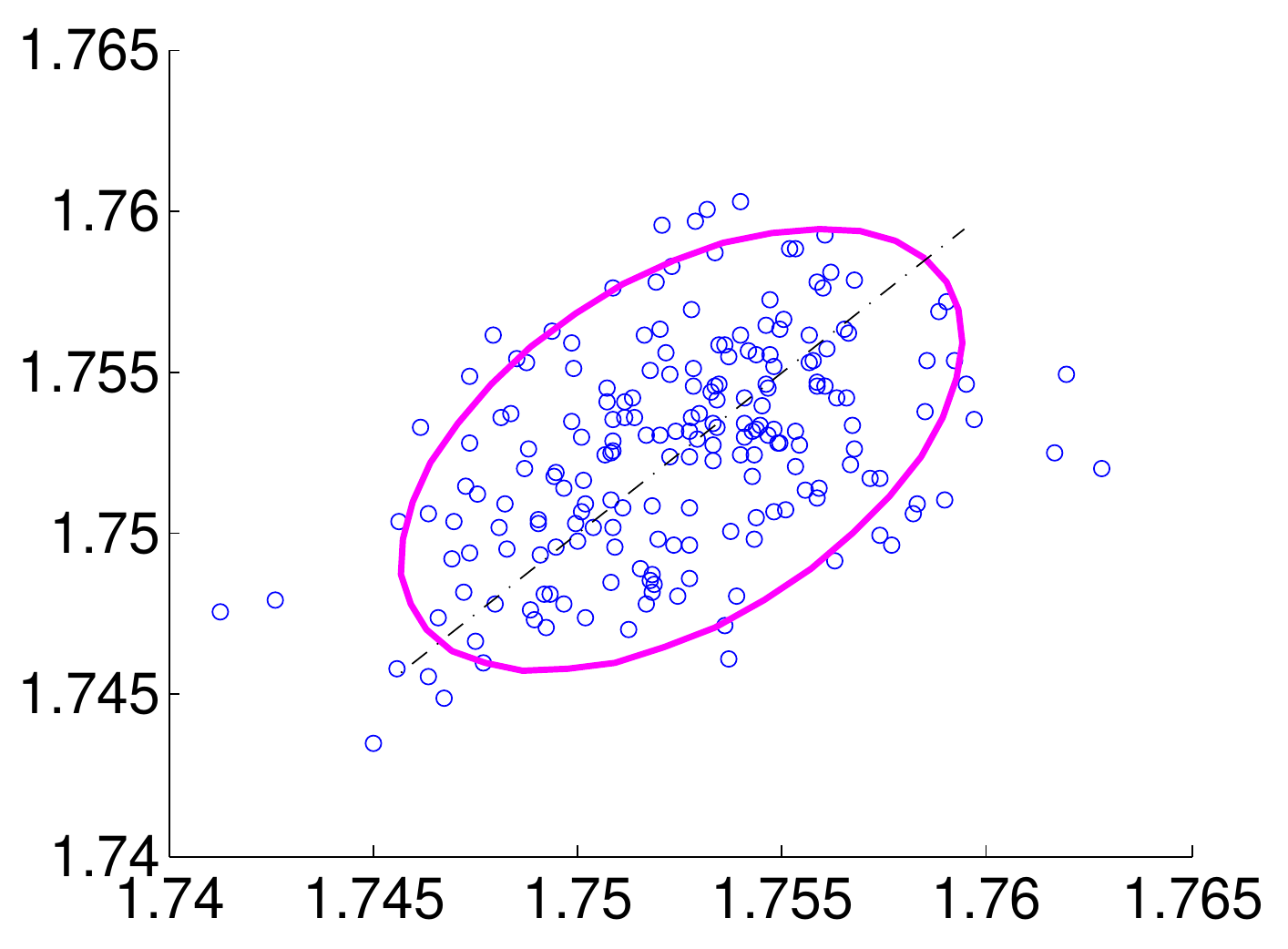} \\
     & \scalebox{0.6}{$\squaredMMDu{X}{Y}$}
     \end{tabular}
    \captionof{figure}{The empirical scatter plot of the joint MMD statistics with $m=1000$ for 200 repeated tests, along with the $2\sigma$ iso-curve of the analytical Gaussian distribution estimated by Eq.~\eqref{jmlr2016:eq:similarity_theory:joint_asymtotic_dist_MMD}.  The analytical distribution closely matches the empirical scatter plot, verifying the correctness of the variances.}\label{jmlr2016:fig:similarity_experiments:toy:isocurve_conservativetest}
  \end{minipage}
  \end{minipage}
%
%
\subsubsection{Model selection for deep unsupervised neural networks} \label{jmlr2016:subsec:similarity_experiments:modelselection}

An important potential application of the RelativeMMD problem (Pb.~\ref{jmlr2016:pb:background:relative_similarity_test}) can be found in recent work on unsupervised learning with deep neural networks \citep{kingma2013auto,bengio2013deep,larochelle2011neural,salakhutdinov2009deep,li2015generative,goodfellow2014generative}.  As noted by several authors, the evaluation of generative models is a challenging open problem \citep{li2015generative,goodfellow2014generative}, and the distributions of samples from these models are very complex and difficult to evaluate. The RelativeMMD performance can be used  to compare different model settings, or even model families, in a statistically valid framework. To compare two models using our test, we generate samples from both, and compare these to a set of real  target data samples that were not used to train either model. 

In the experiments in the sequel we focus on the recently introduced variational auto-encoder (VAE) \citep{kingma2013auto} and the generative moment matching networks (GMMN)  \citep{li2015generative}. The former trains an encoder and decoder network jointly minimizing a regularized variational lower bound \citep{kingma2013auto}. While the latter class of models is purely generative minimizing an MMD based objective, this model works best when coupled with a separate auto-encoder which reduces the dimensionality of the data.  An architectural schematic for both classes of models is provided in Fig.~\ref{jmlr2016:fig:similarity_experiments:modelselection:illustration}. Both these models can be trained using standard backpropagation \citep{Rumelhart:1988:LRB:65669.104451}. Using the latent variable prior we can directly sample the data distribution of these models without using MCMC procedures \citep{hinton2006fast,salakhutdinov2009deep}.  

We use the MNIST and FreyFace datasets for our analysis \citep{lecun1998gradient,kingma2013auto,goodfellow2014generative}. We first demonstrate the effectiveness of our test in a setting where we have a theoretical basis for expecting superiority of one unsupervised model versus another. Specifically, we use a setup where more training samples were used to create one model versus the other. We find that the RelativeMMD framework agrees with the expected results (models trained with more data generalize better). We then demonstrate how the RelativeMMD can be used in evaluating network architecture choices, and we show that our test strongly agrees with other established metrics, but in contrast can provide significance results using just the validation data while other methods may require an additional test set.  

Several practical matters must be considered when applying the RelativeMMD test. The selection of kernel can affect the quality of results, particularly more suitable kernels can give a faster convergence. In this work we extend the logic of the median heuristic \citep{NIPS2012_4727} for bandwidth selection by computing the median pairwise distance between samples from $\mathbbP_x$ and $\mathbbP_y$ and averaging that with the median pairwise distance between samples from $\mathbbP_x$ and $\mathbbP_z$, which helps to maximize the difference between the two MMD statistics.
Although the derivations for the variance of our statistic hold for all cases, the estimates require asymptotic arguments and thus a sufficiently large $m$.  Selecting the kernel bandwidth in an appropriate range can therefore substantially increase the power of the test at a fixed sample size.  While we observed the median heuristic to work well in our experiments, there are cases where alternative choices of kernel can provide greater power: for instance, the kernel can be chosen to maximize the expected test power on a held-out dataset \citep{NIPS2012_4727}.
\begin{figure}
    \centering
    \begin{subfigure}[b]{0.3\textwidth}
        \includegraphics[width=\textwidth]{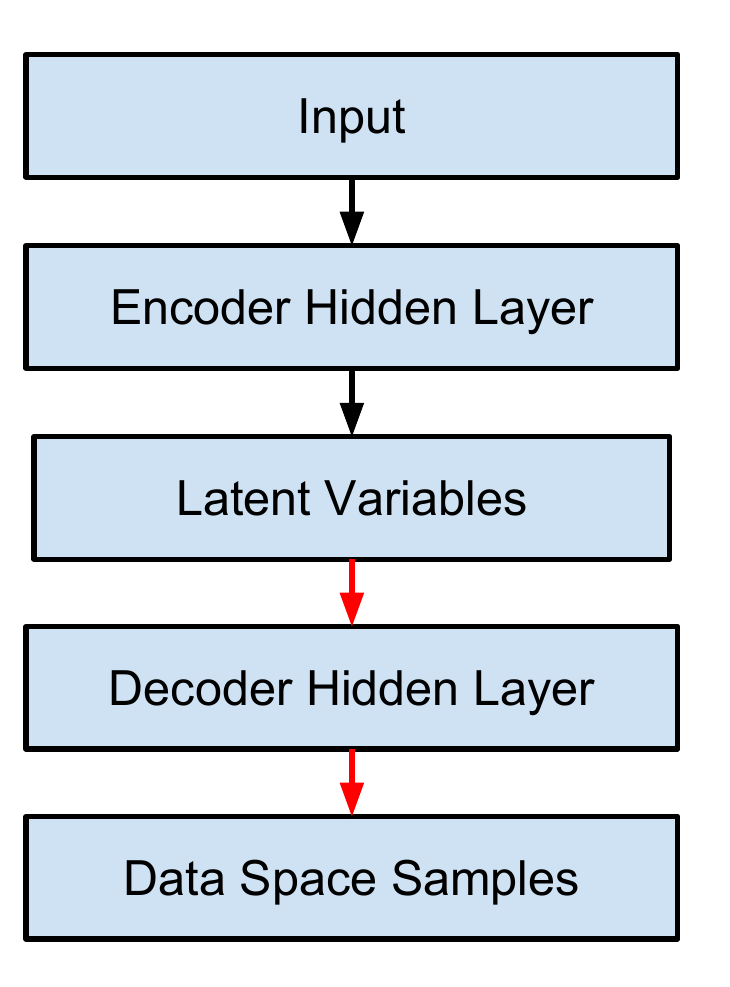}
        \caption{Variational auto-encoder reference model}
        \label{jmlr2016:fig:similarity_experiments:modelselection:illustrationVAR}
    \end{subfigure}
    ~
    \begin{subfigure}[b]{0.3\textwidth}
        \includegraphics[width=\textwidth]{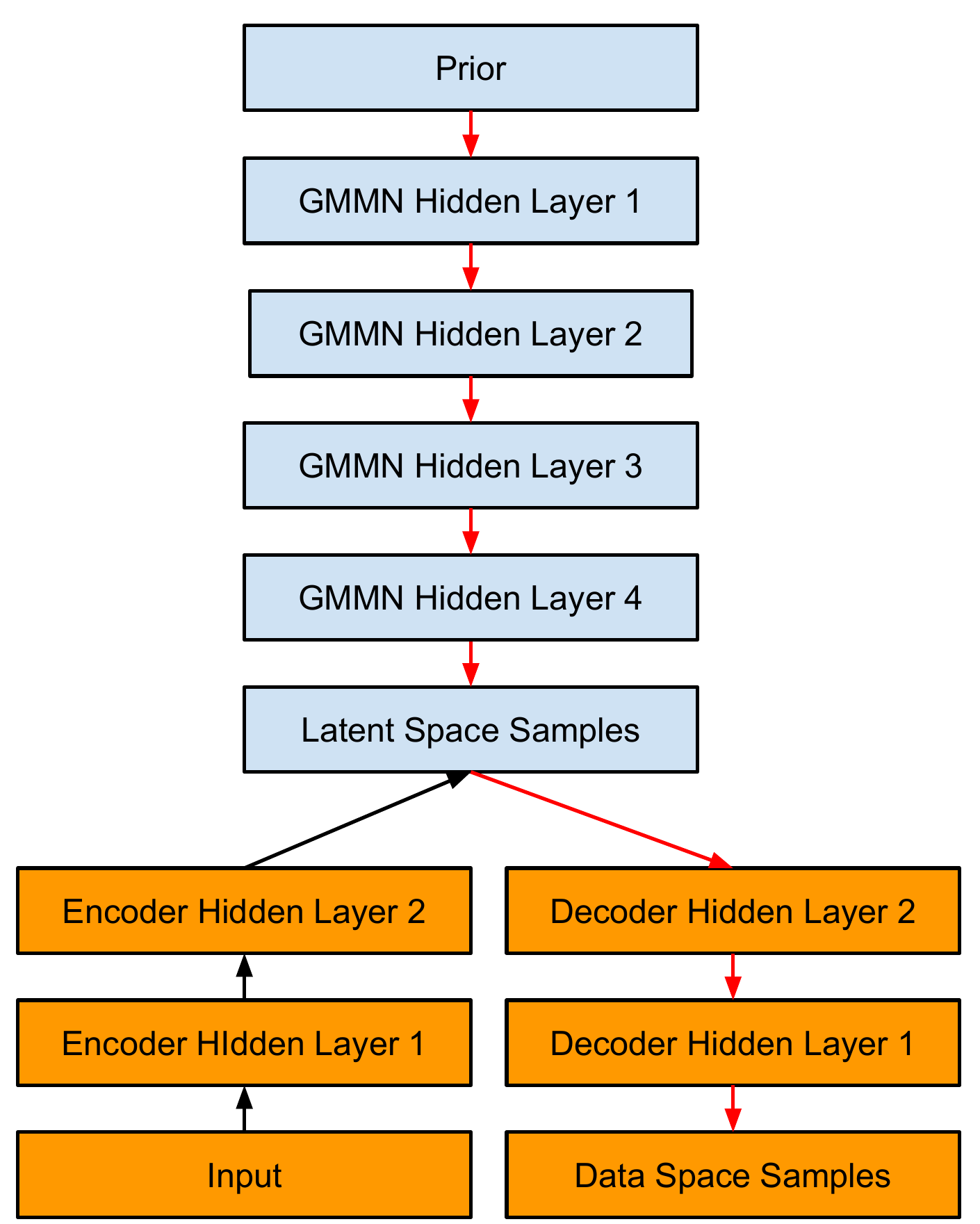}
        \caption{Auto-Encoder + GMMN reference model}
        \label{jmlr2016:fig:similarity_experiments:modelselection:illustrationGMMN}
    \end{subfigure}
    ~ 
    \caption{In Fig.~\ref{jmlr2016:fig:similarity_experiments:modelselection:illustrationVAR}, we have 400 hidden nodes (both encoder and decoder) and 20 latent variables in the reference model for our experiments. In Fig.~\ref{jmlr2016:fig:similarity_experiments:modelselection:illustrationGMMN}, we illustate that the auto-encoder (indicated in orange) is trained separately and has 1024 and 32 hidden nodes in decode and encode hidden layers. The GMMN has 10 variables generated by the prior, and the hidden layers have 64, 256, 256, 1024 nodes in each layer respectively. In both networks red arrows indicate the data flow during sampling}    \label{jmlr2016:fig:similarity_experiments:modelselection:illustration}
\end{figure}
%
%
\paragraph{Variational Auto-Encoder Sample Size and Architecture Experiments}    
We use the architecture from \cite{kingma2013auto} with a hidden layer at both the encoder and decoder and a latent variable layer as shown in Fig. \ref{jmlr2016:fig:similarity_experiments:modelselection:illustrationVAR}. We use sigmoidal activation for the hidden layers of encoder and decoder. For the FreyFace data, we use a Gaussian prior on the latent space and data space. For MNIST, we used  a Bernoulli prior for the data space. 
 We fix the training set size of the second auto-encoder to 300 images for the FreyFace data and 1500 images for the MNIST data. We vary the number of training samples for the first auto-encoder. We then generate samples from both auto-encoders and compare them using Relative MMD to a held out set of data. We use 1500 FreyFace samples as the target in Relative MMD and 15000 images from MNIST. Since a single sample of the data might lead to better generalization performance by chance, we repeat this experiment multiple times and record whether the relative similarity test indicated a network is preferred or if it failed to reject the null hypothesis. The results are shown in Fig.~\ref{jmlr2016:fig:similarity_experiments:modelselection:VARresults} which demonstrates that we are closely following the expected model preferences. Additionally for MNIST we use another separate set of supervised training and test data. We encode this data using both auto-encoders and use logistic regression to obtain a classification accuracy. The indicated accuracies closely match the results of the relative similarity test, further validating the test. 
\begin{figure}
    \centering
    \begin{subfigure}{.5\textwidth}
        \centering    
        \setlength{\tabcolsep}{-1pt}
\renewcommand{\arraystretch}{0.5}
        \begin{tabular}{cc}
     \begin{sideways} $\qquad$  \scalebox{0.6}{Number of times Hypothesis Selected} \end{sideways} & 
        \includegraphics[width=0.9\textwidth]{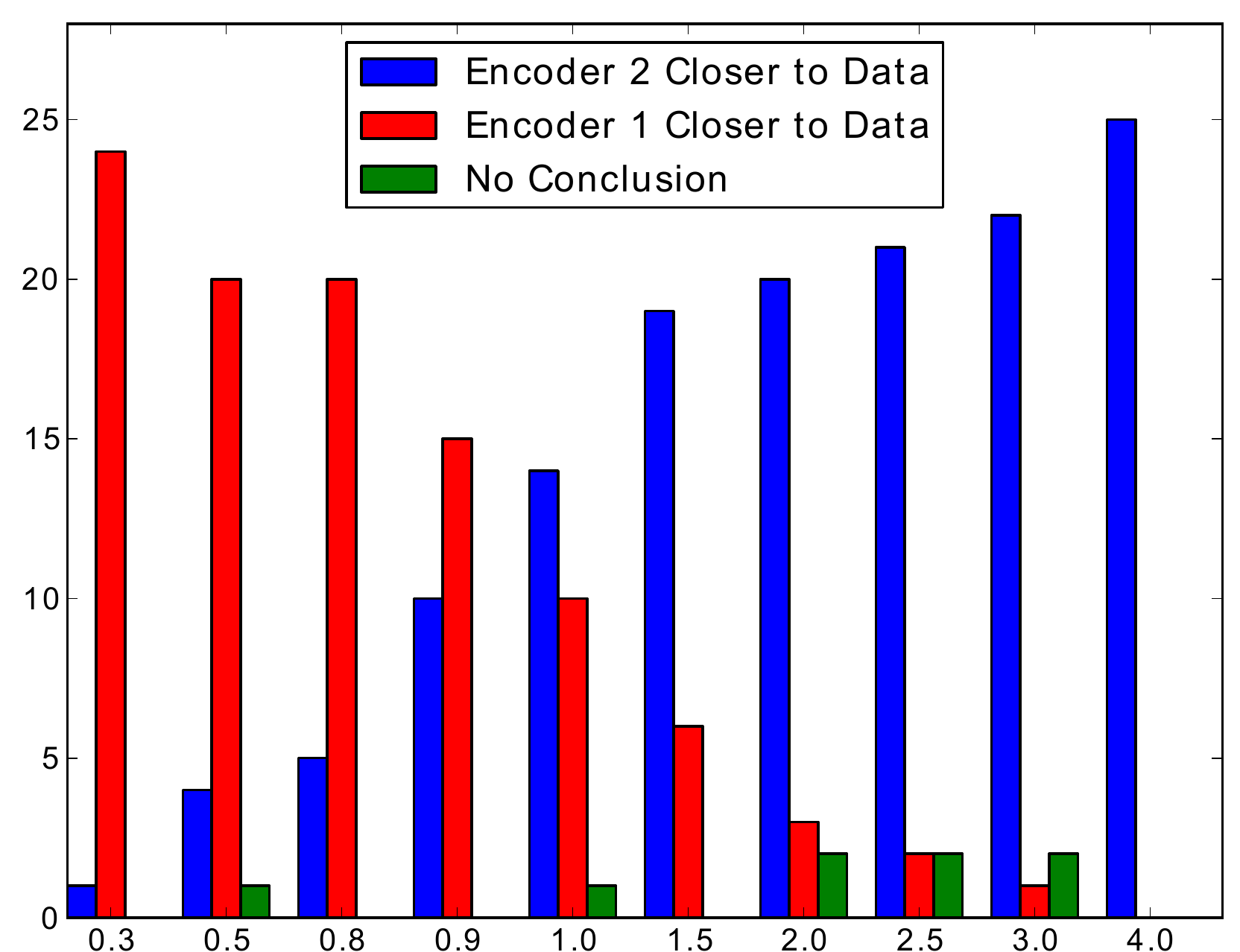} \\
        & \scalebox{0.6}{Ratio Training Samples for Encoder 2 and Training Samples Encoder 1}
        \end{tabular}
        \caption{RelativeMMD decision (25 trials)}
        \label{jmlr2016:fig:similarity_experiments:modelselection:VARresultsMNIST}
    \end{subfigure}\hfill
    \begin{subfigure}{.5\textwidth}
        \centering
    \setlength{\tabcolsep}{-1pt}
	\renewcommand{\arraystretch}{0.5}
        \begin{tabular}{cc}
     \begin{sideways} $\qquad$  \scalebox{0.6}{Number of times Hypothesis Selected} \end{sideways} & 
        \includegraphics[width=0.9\textwidth]{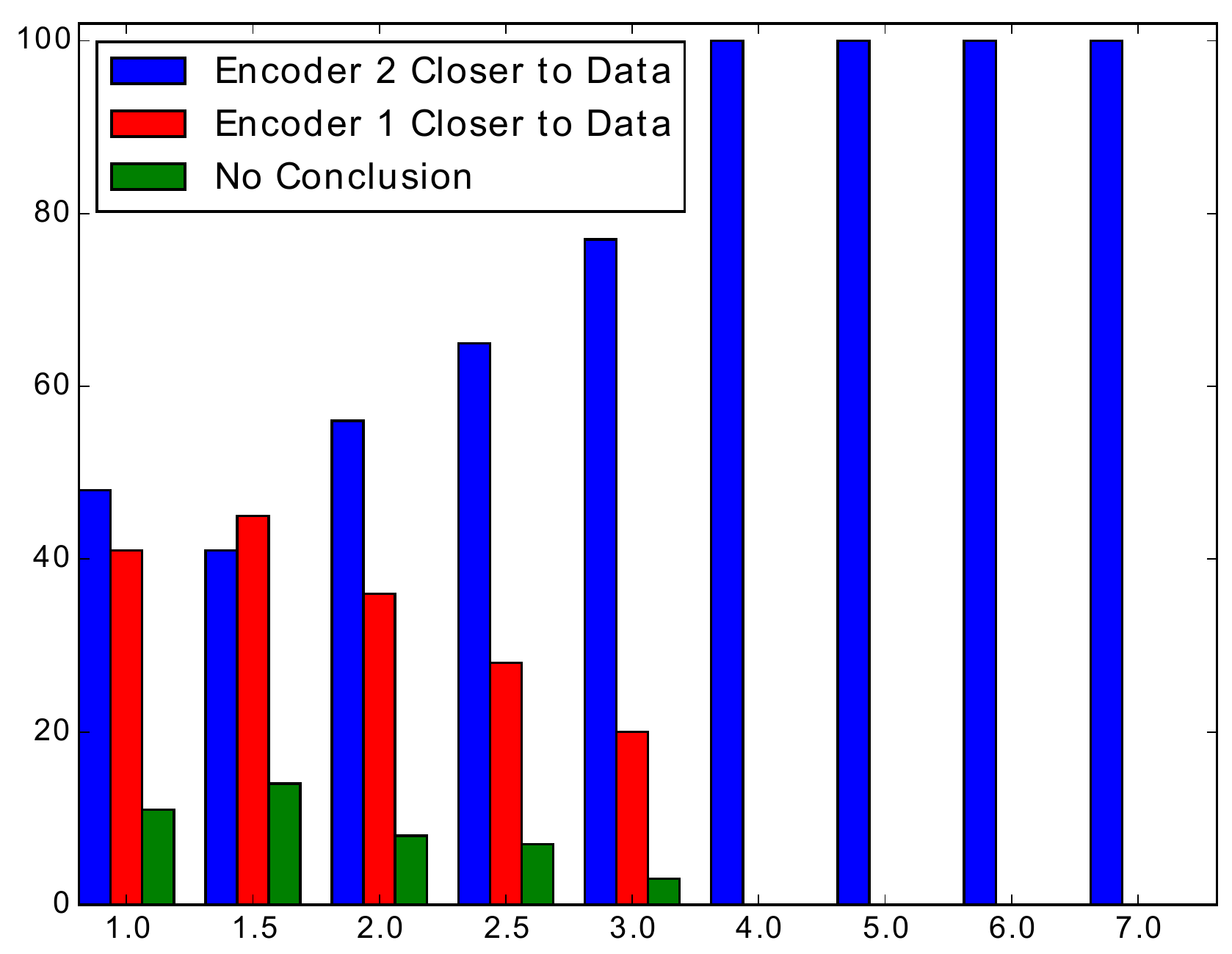} \\
        & \scalebox{0.6}{Ratio Training Samples for Encoder 2 and Training Samples Encoder 1}
        \end{tabular}
        \caption{Selected Hypothesis (100 runs)}
        \label{jmlr2016:fig:similarity_experiments:modelselection:VARresultsFRAYFACE}
    \end{subfigure}\hfill
    \begin{subfigure}{.5\textwidth}
        \centering
            \setlength{\tabcolsep}{-1pt}
\renewcommand{\arraystretch}{0.5}
        \begin{tabular}{cc}
     \begin{sideways} $\qquad \qquad  \quad$  \scalebox{0.6}{Accuracy} \end{sideways} & 
        \includegraphics[width=0.9\textwidth]{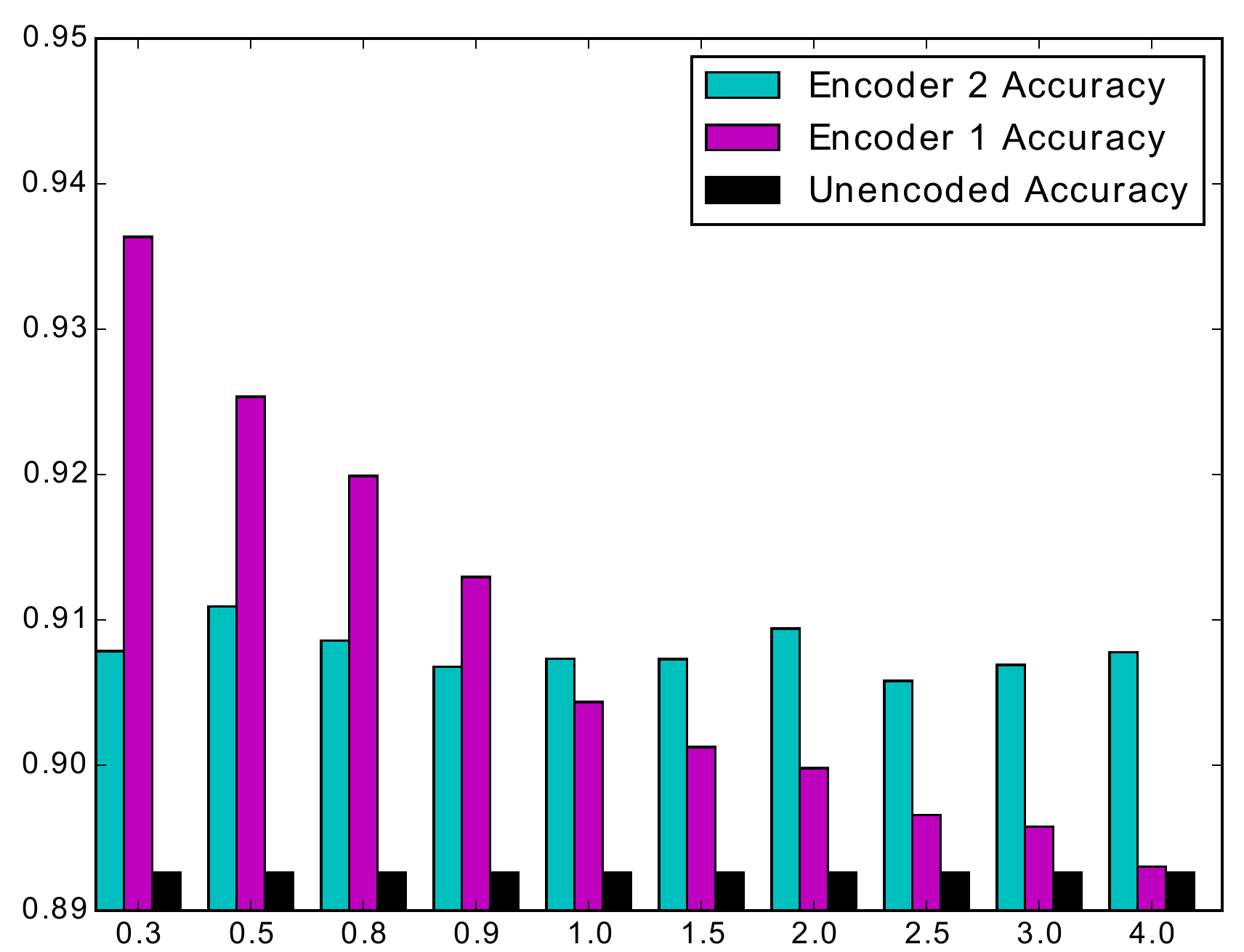} \\
        & \scalebox{0.6}{Ratio Training Samples for Encoder 2 and Training Samples Encoder 1}
        \end{tabular}
        \caption{Average Accuracy Using Autoencoder Features (25 trials)}
        \label{jmlr2016:fig:similarity_experiments:modelselection:VARresultsMNIST_accuracies}
    \end{subfigure}\hfill
    \begin{minipage}{.47\linewidth}
    \caption{In Fig.~\ref{jmlr2016:fig:similarity_experiments:modelselection:VARresultsMNIST}, we show the effect of varying the training set size of one auto-encoder trained on MNIST data. In Fig.~\ref{jmlr2016:fig:similarity_experiments:modelselection:VARresultsMNIST_accuracies} As a secondary validation we compute the classification accuracy of MNIST on a separate train/test set encoded using encoder 1 and encoder 2. In Fig.~\ref{jmlr2016:fig:similarity_experiments:modelselection:VARresultsMNIST_accuracies} We then show the effect of varying the training set size of one auto-encoder using the FreyFace data. We note that due to the size of the FreyFace dataset, we limit the range of ratios used. From this figure we see that the results of the relative similarity test match our expectation: more data produces models which more closely match the true distribution.}
    \label{jmlr2016:fig:similarity_experiments:modelselection:VARresults}
\end{minipage}
\end{figure}
We consider model selection between networks using different architectures.  We train two encoders, one a fixed reference model (400 hidden units and 20 latent variables), and the other varying as specified in Tab.~\ref{jmlr2016:tab:similarity_experiments:modelselection:VARperformance}. 25000 images from the MNIST data set were used for training. We use another 20000 images as the target data in Relative MMD. Finally, we use a set of 10000 training and 10000 test images for a  supervised task experiment. We use the labels in the MNIST data and perform training and classification using an $\ell_2$-regularized logistic regression on the encoded features.
In addition we use the supervised task test data to evaluate the variational lower bound of the data under the two models \citep{kingma2013auto}. We show the result of this experiment in Tab. \ref{jmlr2016:tab:similarity_experiments:modelselection:VARperformance}. For each comparison we take a different subset of training data which helps demonstrate the variation in lower bound and accuracy when re-training the reference architecture. We use a significance value of $5\%$ and indicate when the test favors one auto-encoder over another or fails to reject the null hypothesis. We find that Relative MMD evaluation of the models closely matches performance on the supervised task and the test set variational lower bound.
\begin{table}\centering
\begin{adjustbox}{max width=\textwidth}
\begin{tabular}{l|l|l|l|l|l|l}
Hidden&Latent  &  Result   & Accuracy (\%) & Accuracy (\%) & Lower Bound & Lower Bound  \\
 VAE 1  & VAE 1& RelativeMMD & VAE 1 &VAE 2 & VAE 1 & VAE 2 \\\hline
200 & 5 & Favor VAE 2 & 92.8 $\pm$ 0.3   & \textbf{94.7 $\pm$ 0.2} & -126  & \textbf{-97}\\
200 & 20 & Favor VAE 2  & 92.6$\pm$ 0.3      & \textbf{94.5 $\pm$ 0.2} &-115 &\textbf{-105}  \\
400 & 50 & Favor VAE 1  & \textbf{94.6 $\pm$ 0.2}      & 94.0 $\pm$ 0.2 & \textbf{-99.6} & -123.44\\
800 & 20 & Favor VAE 1  & \textbf{94.8 $\pm$ 0.2}      & 93.9 $\pm$ 0.2 & \textbf{-111} &-115\\
800 & 50 & Favor VAE 1  & 94.2 $\pm$ 0.3      & 94.5 $\pm$ 0.2 &\textbf{-101} &-103
\end{tabular}
\end{adjustbox}
\caption{We compare several variational auto encoder (VAE) architectural choices for the number of hidden units in both decoder and encoder and the number of latent variables for the VAE. The reference encoder, denoted encoder 2, has 400 hidden units and 20 latent variables. We denote the competing architectural models as encoder 1. 
We vary the number of hidden nodes in both the decoder and encoder and the number of latent variables. 
Our test closely follows the performance difference of the auto-encoder on a supervised task (MNIST digit classification) as well as the variational lower bound on a withheld set of data. The data used for evaluating the Accuracy and Lower Bound is separate from that used to train the auto-encoders and for the hypothesis test.} 
\label{jmlr2016:tab:similarity_experiments:modelselection:VARperformance}
\end{table}
%
%
\paragraph{Generative Moment Matching Networks Architecture Experiments}
We demonstrate our hypothesis test on a different class of deep generative models called Generative Moment Matching Networks (GMMN) \citep{li2015generative}. This recently introduced model has shown competitive performance in terms of test set likelihood on the MNIST data. Furthermore the training of this model is based on the MMD criterion. \cite{li2015generative} proposes to use that model along with an auto-encoder, which is the setup we employ in this work. Here a standard auto-encoder model is trained on the data to  obtain a low dimensional representation, then a GMMN network is trained on the latent representations (Fig.~\ref{jmlr2016:fig:similarity_experiments:modelselection:illustration}).  

We use the relative similarity test  to evaluate various architectural choices in this new class of models. We start from the baseline model specified in \cite{li2015generative} and associated software. The details of the reference model are specified in Fig.~\ref{jmlr2016:fig:similarity_experiments:modelselection:illustration}.

We vary the number of auto-encoder hidden layers (1 to 4), generative model layers(1, 4, or 5), the number of network nodes (all or 50\% of the reference model), and use of drop-out on the auto-encoder. We use the same training set of 55000, validation set of 5000 and test set of 10000 as in \citep{li2015generative,goodfellow2014generative}. In total we train 48 models. We use these to compare 4 simplified binary network architecture choices using the Relative MMD: using dropout on the auto-encoder, few (1) or more (4 or 5) GMMN layers, few (1 or 2) or more (3 or 4) auto-encoder layers, and the number of network nodes. We use our test to compare these model settings using the \emph{validation set} as the target in the relative similarity test, and samples from the models as the two sources. To validate our results we compare it to likelihoods computed on the test set.  The results are shown in Tab.~\ref{jmlr2016:tab:similarity_experiments:modelselection:performanceGMMN}. We see that the likelihood results computed on a separate test set follow the conclusions obtained from MMD on the validation set. Particularly, we find that using fewer hidden layers for the GMMN and more hidden nodes generally produces better models.
\begin{table}
\begin{adjustbox}{max width=\textwidth}
\begin{tabular}{l|c|c|c|c|c}
& \multicolumn{3}{c|}{RelativeMMD Preference}   &  \multicolumn{2}{c}{}  \\\hline
Experimental Condition (A/B)& A & Inconclusive 
& B & Avg Likelihood A  & Avg Likelihood B \\ \hline
Dropout/No Dropout & 199 & 17 & 360  & -9.01 $\pm$ 55.43 & 76.76 $\pm$ 42.83 \\ \hline
More/Fewer GMMN Layers& 105 & 14 & 393 & -73.99 $\pm$ 40.96 & \textbf{249.6 $\pm$ 8.07} \\ \hline
More/Fewer Nodes & 450 & 13 & 113 & \textbf{125.2 $\pm$ 43.4} & -57 $\pm$ 49.57    \\ \hline
More/Fewer AE layers& 231 & 21 & 324 & 41.78 $\pm$ 44.07 & 25.96 $\pm$ 55.85    
\end{tabular}
\end{adjustbox}
\caption{For each experimental condition (e.g. dropout or no dropout) we show the number of times the Relative MMD prefers models in group 1 or 2 and number of inconclusive tests. We use the validation set as the target data for Relative MMD. An average likelihood for the MNIST test set for each group is shown with error bars. We can see that the MMD choices are in agreement with likelihood evaluations. Particularly we identify that models with fewer GMMN layers  and models with more nodes have more favourable samples, which is confirmed by the likelihood results.}
\label{jmlr2016:tab:similarity_experiments:modelselection:performanceGMMN}
\end{table}
%
%
%
\subsubsection{Discussion}
In these experiments we have seen that the RelativeMMD test can be used to compare deep generative models obtaining judgments aligned with other metrics. Comparisons to other metrics are important for verifying our test is sensible, but it can occlude the fact that MMD is a valid evaluation technique on its own. When evaluating only sample generating models where likelihood computation is not possible, MMD is an appropriate and tractable metric to consider in addition to Parzen-Window log likelihoods and visual appearance of the samples. In several ways it is potentially more appropriate than Parzen-windows as it allows one to consider directly the discrepancy between the test data samples and the model samples while allowing for significance results. In such a situation, comparing the performance of several models using the MMD against a single set of test samples, the RelativeMMD test can provide an automatic significance value without expensive cross-validation procedures.  

Gaussian kernels are closely related to Parzen-window estimates, thus computing an MMD in this case can be considered related to comparing Parzen window log-likelihoods. The MMD gives several advantages, however. 
First, the asymptotics of MMD are quite different to Parzen-windows, since the Parzen-window bandwidth shrinks as $m$ grows. Asymptotics of relative tests with shrinking bandwidth are unknown: even for two samples this is challenging \citep{KrishnamurthyKP15}. Other two sample tests are not easily extendable to relative tests~\citep{rosenbaum2005exact, friedman1979multivariate,hall2002permutation}. This is because  the tests above rely on graph edge counting or nearest neighbor-type statistics, and null distributions are obtained via combinatorial arguments which are not easily extended from two to three samples. MMD is a $U$-statistic, hence its asymptotic behavior is much more easily generalized to multiple dependent statistics.

There are two primary advantages of the MMD over the variational lower bound, where it is known \citep{kingma2013auto}: first, we have a characterization of the asymptotic behavior, which allows us to determine when the difference in performance is significant; second, comparing two lower bounds produced from two different models is unreliable, as we do not know how conservative either lower bound is.
\section{Conclusion} \label{JMLR2016:sec:conclusion}

We have described two novel non-parametric statistical hypothesis tests using analogous mathematical derivation based on the estimation of two correlated $U$-statistics. The first test of relative dependency determines whether a source random variable is significantly more strongly dependent on one target random variable or another. The test is based on the Hilbert-Schmidt Independent Criterion. And the second test of relative similarity determines whether one model generates samples significantly closer to the reference distribution than the other. The criterion is based on the Maximum Mean Discrepancy. We have shown that both test are consistent, are strictly more powerful than a test with uncorrelated statistics, and the computation requirements of the tests is quadratic in the sample size.
We have applied the test of relative dependency to the problem of identifying relative dependencies between languages using a multilingual corpus, and for discovering the relative relationships between gliomas and genetic information. Additionally, we have shown the application of relative test of similarity to the problem of model selection in deep generative models, and currently an important question in machine learning. Code for our methods is available.


\acks{This work is funded by Internal Funds KU Leuven, ERC Grant 259112, FP7-MC-CIG 334380, the Royal Academy of Engineering through the Newton Alumni Scheme, and DIGITEO 2013-0788D-SOPRANO. WB is supported in part by a CentraleSup\'{e}lec fellowship. }


\appendix
\section{Detailed Derivations of the MMD Variance and Covariance}\label{jmlr2016:appendix}

The variance and the covariance for a $U$-statistic is described in \citet[Eq.\ 5.13]{hoeffding1948class} and \citet[Chap.\ 5]{serfling2009approximation}.

Let $\mathcal{V} := (v_1, ..., v_m)$ be $m$ i.i.d.\ random variables where $v:= (x,y) \sim \mathbbP_x \times \mathbbP_y$. An unbiased estimator of $\squaredMMDpop{x}{y}$ is 
\begin{equation}
\squaredMMDu{X_m}{Y_m} = \frac{1}{m(m-1)} \sum_{i \ne j}^m f(v_i,v_j)
\end{equation}
with $f(v_i,v_j) = k(x_i,x_j) + k(y_i,y_j) -k(x_i,y_j)  -k(x_j,y_i)$.

Similarly, let $\mathcal{W} := (w_1, ..., w_m)$ be $m$ i.i.d.\ random variables where $w:= (x,z) \sim \mathbbP_x \times \mathbbP_z$. An unbiased estimator of $\squaredMMDpop{x}{z}$ is 
\begin{equation}
\squaredMMDu{X_m}{Z_m} = \frac{1}{m(m-1)} \sum_{i \ne j}^m g(w_i,w_j)
\end{equation}
with $g(w_i,w_j) = k(x_i,x_j) + k(z_i,z_j) -k(x_i,z_j)  -k(x_j,z_i)$

Then the variance/covariance for a $U$-statistic with a kernel of order 2 is given by
\begin{align}
Var(\operatorname{MMD}_u^2) &= \frac{4(m-2)}{m(m-1)} \zeta_1 + \frac{2}{m(m-1)} \zeta_2 
\label{eq:allterm_in_variance_of_MMDstatistics}
\end{align}
Eq.~\eqref{eq:allterm_in_variance_of_MMDstatistics}, neglecting higher order terms, can be written as
\begin{align}
Var(\operatorname{MMD}_u^2) & = \frac{4(m-2)}{m(m-1)} \zeta_1 + \mathcal{O}
(m^{-2}) 
\label{eq:derivation_of_the_variance_ofMMD}
\end{align}
where for the variance term, $\zeta_1 = \var \left[ \Eone \left[ f(v_1,V_2) \right] \right] $ and for the covariance term $ \zeta_1 = \var \left[ \operatorname{\E}_{v_1,w_1} \left[ f(v_1,V_2)g(w_1,W_2) \right] \right]$.

\paragraph{Notation} $[\tilde{\matK}_{xx}]_{ij} = [\matK_{xx}]_{ij}$ for all $i \ne j$ and $[\tilde{\matK}_{xx'}]_{ij} =0$ for $j =i$. Same for $\tilde{\matK}_{yy}$ and $\tilde{\matK}_{zz}$. We will also make use of the fact that $k(x_i,x_j) = \langle \phi(x_i), \phi(x_j) \rangle$ for an appropriately chosen inner product, and  function $\phi$.  We then denote $  \mu_x := \int \phi(x) d\mathbbP_x $.

\subsection{Variance of MMD}

We note that many terms in expansion of the squares above cancel out due to independence. For example  $\E_{x_1,y_1}\left[ \langle \phi(y_1), \mu_y \rangle \langle \phi(x_1),\mu_y \rangle \right] - \E_{y_1} \left[ \langle \phi(y_1), \mu_y \rangle \right]  \E_{x_1}\left[ \langle \phi(x_1),\mu_y \rangle \right]=0$. 

We can thus simplify to the following expression for $\zeta_1$

\begin{align}
\zeta_1 &= \EXoneYone \left[ \left( \EXtwoYtwo \left[h(x_1,y_1)\right] \right) ^2 \right] - \left( \squaredMMDpop{X}{Y} \right)^2 \\
&= \EXoneYone \left[ (\langle \phi(x_1) , \mu_x \rangle + \langle \phi(y_1), \mu_y \rangle - \langle \phi(x_1), \mu_y \rangle - \langle \mu_x , \phi(y_1) \rangle)^2 \right] - \left( \squaredMMDpop{X}{Y} \right)^2 \nonumber \\ 
&= \EXoneYone \big[ 
\langle \phi(x_1),\mu_x \rangle^2 
+2 \langle \phi(x_1),\mu_x \rangle \langle \phi(y_1),\mu_y \rangle 
-2 \langle \phi(x_1),\mu_x \rangle \langle \phi(x_1),\mu_y \rangle \nonumber
\\ & \nonumber
\qquad \qquad
-2 \langle \phi(x_1),\mu_x \rangle \langle \phi(y_1),\mu_x \rangle
+ \langle \phi(y_1), \mu_y \rangle^2
\\& \nonumber
\qquad \qquad
-2 \langle \phi(y_1), \mu_y \rangle \langle \phi(x_1),\mu_y \rangle
-2 \langle \phi(y_1), \mu_y \rangle \langle \phi(y_1), \mu_x \rangle
\\ & \nonumber
\qquad \qquad
+ \langle \phi(x_1),\mu_y \rangle^2
+2 \langle \phi(x_1),\mu_y \rangle \langle \phi(y_1),\mu_x \rangle 
\\& \nonumber
\qquad \qquad
+\langle \phi(y_1),\mu_x \rangle^2 
\big] - \left( \squaredMMDpop{X}{Y} \right)^2 \\
&= \label{eq:MMDvarPopulationTerms}
\EXone [\langle \phi(x_1),\mu_x \rangle^2] - \EXone [\langle \phi(x_1),\mu_x \rangle]^2 \nonumber
\\& 
\qquad \qquad
-2 (\EXone [\langle \phi(x_1),\mu_x \rangle \langle \phi(x_1),\mu_y \rangle] - \EXone [\langle \phi(x_1),\mu_x \rangle] \EXone [\langle \phi(x_1),\mu_y \rangle]) \nonumber
\\& \nonumber
\qquad \qquad
+ \EYone[\langle \phi(y_1), \mu_y \rangle^2] - \EYone [\langle \phi(y_1), \mu_y \rangle]^2
\\ & \nonumber
\qquad \qquad
-2 ( \EYone [\langle \phi(y_1), \mu_y \rangle \langle \phi(y_1), \mu_x \rangle ] - \EYone [\langle \phi(y_1), \mu_y \rangle] \mathbb{E}_{y_1} [ \langle \phi(y_1), \mu_x \rangle ])
\\ & \nonumber
\qquad \qquad
+ \EXone [\langle \phi(x_1),\mu_y \rangle^2] - \EXone [\langle \phi(x_1),\mu_y \rangle]^2
\\ & \nonumber
\qquad \qquad
+ \EYone [\langle \phi(y_1),\mu_x \rangle^2 ] - \EYone [\langle \phi(y_1),\mu_x \rangle]^2 
\end{align}

Substituting empirical expectations over the data sample for the population expectations in Eq.~\eqref{eq:MMDvarPopulationTerms} gives
\begin{align}
\zeta_1 & \approx
\frac{1}{m(m-1)^2} \vecOnes^T \tilde{\matK}_{xx}\tilde{\matK}_{xx} \vecOnes - \left(\frac{1}{m(m-1)} \vecOnes^T \tilde{\matK}_{xx} \vecOnes \right)^2
\\ & \nonumber
\qquad \qquad
-2 \left(
\frac{1}{m(m-1) n} \vecOnes^T \tilde{\matK}_{xx} \matK_{xy} \vecOnes
- \frac{1}{m^2(m-1)n} \vecOnes^T \tilde{\matK}_{xx} \vecOnes \vecOnes^T \matK_{xy} \vecOnes
\right)
\\ & \nonumber
\qquad \qquad
+ \frac{1}{n(n-1)^2} \vecOnes^T \tilde{\matK}_{yy} \tilde{\matK}_{yy} \vecOnes
- \left( \frac{1}{n(n-1)} \vecOnes^T \tilde{\matK}_{yy} \vecOnes \right)^2
\\ & \nonumber
\qquad \qquad
-2 \left( \frac{1}{n(n-1)m} \vecOnes^T \tilde{\matK}_{yy} \matK_{yx} \vecOnes
- \frac{1}{n^2(n-1)m} \vecOnes^T \tilde{\matK}_{yy} \vecOnes \vecOnes^T \matK_{xy} \vecOnes
\right)
\\ & \nonumber
\qquad \qquad
+ \frac{1}{n^2 m} \vecOnes^T \matK_{yx} \matK_{xy} \vecOnes
- 2 \left( \frac{1}{nm} \vecOnes^T \matK_{xy} \vecOnes \right)^2
+ \frac{1}{m^2 n} \vecOnes^T \matK_{xy} \matK_{yx} \vecOnes  \nonumber
\end{align}

Derivation of the first term for example
\begin{align}
\EXone [\langle x_1,\mu_x \rangle^2] & \approx \frac{1}{m} \sum_{i=1}^m \langle \phi(x_i), \frac{1}{m-1}\sum_{j=1 \atop j \ne i}^m \phi(x_j) \rangle \langle \phi(x_i),\frac{1}{m-1}\sum_{k=1 \atop k \ne i}^m \phi(x_k) \rangle \\
&= \frac{1}{m(m-1)^2} \sum_{i=1}^m  \sum_{j=1 \atop j \ne i}^m \sum_{k=1 \atop k \ne i}^m k(x_i,x_j)k(x_i,x_k) \nonumber \\
&= \frac{1}{m(m-1)^2} e^T \tilde{\matK}_{xx} \tilde{\matK}_{xx} e \nonumber
\end{align}

\subsection{Covariance of MMD}

We note many terms in expansion of the squares above cancel out due to independence. For example  $\E_{x_1,z_1} \left[\langle \phi(x_1),\mu_x \rangle \langle \phi(z_1),\mu_z \rangle \right] - \E_{x_1} \left[ \langle \phi(x_1),\mu_x \rangle \right]  \E_{z_1}\left[ \langle \phi(z_1),\mu_z \rangle \right]=0$. 

We can thus simplify to the following expression for $\zeta_1$

\begin{align}
\zeta_1 &= \EXoneYoneZone \left[  \EXtwoYtwoZtwo \left[ h(x_1,y_1)  g(x_1,z_1) \right]\right] - \left( \squaredMMDpop{X}{Y} \squaredMMDpop{X}{Z} \right) \\
&= \EXoneYoneZone [ (\langle \phi(x_1),\mu_x \rangle + \langle \phi(y_1),\mu_y \rangle -\langle \phi(x_1),\mu_y \rangle - \langle \phi(x_1),\mu_y) \rangle) \nonumber \\ 
& \qquad \qquad \qquad  (\langle \phi(x_1),\mu_x) \rangle + \langle \phi(z_1),\mu_z \rangle -\langle \phi(x_1),\mu_z \rangle - \langle \phi(x_1),\mu_z \rangle) ] \nonumber \\
& \qquad \qquad  -\squaredMMDpop{X}{Y} \squaredMMDpop{X}{Z} \nonumber \\ 
&= \EXone \left[ \langle \phi(x_1),\mu_x \rangle^2 \right] - \EXone \left[ \langle \phi(x_1),\mu_x \rangle \right]^2 \nonumber \\
& \qquad - \left( \EXone \left[ \langle \phi(x_1),\mu_x \rangle \langle \phi(x_1),\mu_z \rangle \right] - \EXone \left[ \langle \phi(x_1),\mu_x \rangle \right] \EXone \left[ \langle \phi(x_1),\mu_z \rangle \right] \right) \nonumber \\
&  \qquad - \left( \EXone \left[ \langle \phi(x_1),\mu_x \rangle \langle \phi(x_1),\mu_y \rangle \right] - \EXone \left[ \langle \phi(x_1),\mu_x \rangle \right] \EXone \left[ \langle \phi(x_1),\mu_y \rangle \right] \right) \nonumber \\
&  \qquad + \EXone \left[ \langle \phi(x_1),\mu_y \rangle \langle \phi(x_1),\mu_z \rangle \right] - \EXone \left[ \langle \phi(x_1),\mu_y \rangle \right] \EXone \left[  \langle \phi(x_1),\mu_z \rangle \right]  \nonumber \\
& \approx
\frac{1}{m(m-1)^2} \vecOnes^T \tilde{\matK}_{xx}\tilde{\matK}_{xx} \vecOnes - \left(\frac{1}{m(m-1)} \vecOnes^T \tilde{\matK}_{xx} \vecOnes \right)^2 \nonumber \\
& \qquad - \left( \frac{1}{m(m-1)r} \vecOnes^T \tilde{\matK}_{xx} \matK_{xz} \vecOnes - \frac{1}{m^2(m-1)r}  \vecOnes^T \tilde{\matK}_{xx} \vecOnes \vecOnes^T \matK_{xz} \vecOnes \right) \nonumber \\
& \qquad - \left( \frac{1}{m(m-1)n} \vecOnes^T \tilde{\matK}_{xx} \matK_{xy} \vecOnes - \frac{1}{m^2(m-1)n} \vecOnes^T \tilde{\matK}_{xx} \vecOnes \vecOnes^T \matK_{xz} \vecOnes \right) \nonumber \\
& \qquad + \left( \frac{1}{mnr} \vecOnes^T \matK_{yx} \matK_{xz} \vecOnes - \frac{1}{m^2nr} \vecOnes^T \matK_{xy}\vecOnes \vecOnes^T \matK_{xz} \vecOnes\right) \nonumber
\end{align}

\subsection{Derivation of the variance of the difference of two MMD statistics}
In this section we propose an alternate strategy of deriving directly the variance of a u-statistic of the difference of MMDs with a joint variable. This formulation agrees with the derivation of the covariance matrix and subsequent projection, and provides extra insights.

Let $\mathcal{D} := (d_1, ..., d_m)$ be $m$ iid random variables where $d:= (x,y,z) \sim \mathbbP_x \times \mathbbP_y \times \mathbbP_z$.  Then the difference of the unbiased estimators of $\squaredMMDpop{x}{y}$ and $\squaredMMDpop{x}{z}$ is given by
\begin{equation}
\squaredMMDu{x}{y} - \squaredMMDu{x}{z} = \frac{1}{m(m-1)} \sum_{i \ne j}^m f(d_i,d_j)
\label{eq:diff_stat}
\end{equation}
with $f$, the kernel of $\squaredMMDpop{x}{y} - \squaredMMDpop{x}{z}$ of order 2 as follows
\begin{align}
f(d_1,d_2) &= (k(x_1,x_2) + k(y_1,y_2) - k(x_1,y_2) - k(x_2,y_1)) 
\\ &
\qquad - (
k(x_1,x_2) + k(z_1,z_2) - k(x_1,z_2) - k(x_2,z_1)
)
\nonumber \\ &=
(k(y_1,y_2) - k(x_1,y_2) - k(x_2,y_1))
- (
k(z_1,z_2) - k(x_1,z_2) - k(x_2,z_1)
)
\end{align}

Eq.~\eqref{eq:diff_stat} is a $U$-statistic and thus we can apply Eq. \eqref{eq:derivation_of_the_variance_ofMMD} to obtain its variance. We first note
\begin{align}
\E_{d_1}(f(d_1,d_2)) := &
\langle \phi(y_1),\mu_y \rangle - \langle \phi(x_1),\mu_y \rangle - \langle \mu_x,\phi(y_1) \rangle
\\&- (
\langle \phi(z_1),\mu_z \rangle - \langle \phi(x_1),\mu_z \rangle - \langle \mu_x,\phi(z_1) \rangle
) 
\nonumber \\
\E_{d_1,d_2}(f(d_1,d_2)) := &
\squaredMMDpop{x}{y} - \squaredMMDpop{x}{z} \nonumber
\end{align}

We are now ready to derive the dominant leading term,$\zeta_1$, in the variance expression \eqref{eq:derivation_of_the_variance_ofMMD}.

\paragraph{Term $\zeta_1$}

\begin{align}
\zeta_1 :&= \operatorname{Var}(\E_{d_1}(f(d_1,d_2))) \\
&= \EXoneYoneZone [(\langle \phi(y_1),\mu_y \rangle - \langle \phi(x_1),\mu_y \rangle - \langle \mu_x,\phi(y_1) \rangle
- \big(
\langle \phi(z_1),\mu_z \rangle - \langle \phi(x_1),\mu_z \rangle - \langle \mu_x,\phi(z_1) \rangle \big)^2] 
& \qquad \qquad \qquad
- (\squaredMMDpop{x}{y} - \squaredMMDpop{x}{z})^2 \nonumber
\end{align}
We note many terms in expansion of the squares above cancel out due to independence. For example  $\E_{y_1,z_1}[\langle \phi(y_1),\mu_y \rangle \langle \phi(z_1),\mu_z \rangle]-\E_{y_1}[\langle \phi(y_1),\mu_y \rangle]\E_{z_1}[\langle \phi(z_1),\mu_z \rangle]=0$. 

We can thus simplify to the following expression for $\zeta_1$
\begin{align}
\zeta_1=&
\EYone [\langle \phi(y_1),\mu_y \rangle^2] - \EYone [\langle \phi(y_1),\mu_y \rangle]^2
\\& \nonumber
+\EXone  [\langle \phi(x_1),\mu_y \rangle^2] - \EXone  [\langle \phi(x_1),\mu_y \rangle]^2\\ & \nonumber
+\EYone [\langle \mu_x,\phi(y_1) \rangle^2] - \EYone [\langle \mu_x,\phi(y_1) \rangle]^2\\ & \nonumber
+\EZone [\langle \phi(z_1),\mu_z \rangle^2] - \EZone [\langle \phi(z_1),\mu_z  \rangle]^2\\ & \nonumber
+\EXone [\langle \phi(x_1),\mu_z \rangle^2] - \EXone [\langle \phi(x_1),\mu_z \rangle]^2\\ & \nonumber
+\EZone [\langle \mu_x,\phi(z_1) \rangle^2] - \EZone [\langle \mu_x,\phi(z_1) \rangle]^2\\ & \nonumber
-2 (\EYone [\langle \phi(y_1),\mu_y \rangle \langle \mu_x,\phi(y_1) \rangle] - \EYone [\langle \phi(y_1),\mu_y  \rangle] \EYone [\langle \mu_x,\phi(y_1) \rangle])
\\& \nonumber
-2 ( \EXone [\langle \phi(x_1), \mu_y \rangle \langle \phi(x_1), \mu_z \rangle ] - \EXone [\langle \phi(x_1), \mu_y \rangle] \EXone [ \langle \phi(x_1), \mu_z \rangle ])
\\ & \nonumber
-2 ( \EZone [\langle \phi(z_1), \mu_z \rangle \langle \mu_x, \phi(z_1) \rangle ] - \EZone [\langle \phi(z_1), \mu_z \rangle] \EZone [ \langle \mu_x, \phi(z_1) \rangle ])
\\ & \nonumber
\label{eq:variance_difference_of2MMD}
\end{align}

We can empirically approximate these terms as follows:

\begin{align}
\zeta_1 & \approx 
\frac{1}{n(n-1)^2} \vecOnes^T \tilde{\matK}_{yy}\tilde{\matK}_{yy} \vecOnes - \left(\frac{1}{n(n-1)} \vecOnes^T \tilde{\matK}_{yy}\vecOnes \right)^2
\\ & \nonumber
\qquad +  \frac{1}{n^2m} \vecOnes^T \matK_{xy}^T\matK_{xy} \vecOnes - \left( \frac{1}{nm} \vecOnes^T \matK_{xy} \vecOnes \right)^2
\\ & \nonumber
\qquad +  \frac{1}{nm^2} \vecOnes^T \matK_{xy}\matK_{xy}^T \vecOnes - \left( \frac{1}{nm} \vecOnes^T \matK_{xy} \vecOnes \right)^2
\\ & \nonumber
\qquad +  \frac{1}{r(r-1)^2} e^T \tilde{\matK}_{zz}\tilde{\matK}_{zz} \vecOnes - \left( \frac{1}{r(r-1)} \vecOnes^T \tilde{\matK}_{zz}\vecOnes  \right)^2
\\ & \nonumber
\qquad +  \frac{1}{rm^2} \vecOnes^T \matK_{xz}\matK_{xz}^T \vecOnes - \left(\frac{1}{rm} \vecOnes^T \matK_{xz} \vecOnes \right)^2
\\ & \nonumber
\qquad +  \frac{1}{r^2m} \vecOnes^T \matK_{xz}^T\matK_{xz} \vecOnes - \left(\frac{1}{rm} \vecOnes^T \matK_{xz} \vecOnes \right)^2
\\ & \nonumber
\qquad - 2  \left( \frac{1}{n(n-1)m} \vecOnes^T \tilde{\matK}_{yy} \matK_{yx} \vecOnes
- \frac{1}{n(n-1)} \vecOnes^T \tilde{\matK}_{yy}\vecOnes \times \frac{1}{nm} \vecOnes^T \matK_{xy}\vecOnes \right)\\ & \nonumber
\qquad - 2  \left( \frac{1}{nmr} \vecOnes^T \tilde{\matK}_{xy}^T \matK_{xz} \vecOnes
- \frac{1}{nm} \vecOnes^T \matK_{xy}\vecOnes \times \frac{1}{rm} \vecOnes^T \matK_{xz}\vecOnes \right)\\ & \nonumber
\qquad - 2  \left( \frac{1}{r(r-1)m} \vecOnes^T \tilde{\matK}_{zz} \matK_{xz}^T \vecOnes
- \frac{1}{n(n-1)} \vecOnes^T \tilde{\matK}_{yy}\vecOnes  \times \frac{1}{nm} \vecOnes^T \matK_{xy}\vecOnes \right)
\end{align}

\subsection{Equality  }

In this section, we prove that Eq.~\eqref{eq:denominator_of_the_TestStatistics} is equal to the variance of the difference of 2 $\squaredMMDpop{x}{y}$ and $\squaredMMDpop{x}{z}$.

\begin{align}
\sigma_{XY}^2 + \sigma_{XZ} - 2\sigma_{XYXZ} &= \EYone \left[ \langle \phi(y_1),\mu_y \rangle^2 \right] - \EYone \left[ \langle \phi(y_1),\mu_y \rangle \right]^2 \\
& \qquad+ \EZone \left[ \langle \phi(z_1),\mu_y \rangle^2 \right] -  \EZone \left[ \langle \phi(z_1),\mu_y \rangle \right]^2 \nonumber \\
& \qquad - 2 \left( \EYone \left[ \langle \phi(y_1),\mu_y \rangle \langle \phi(y_1),\mu_x \rangle \right] -  \EYone \left[ \langle \phi(y_1),\mu_y \rangle \right] \EYone \left[ \langle \phi(y_1),\mu_x \rangle \right] \right) \nonumber \\
& \qquad -2 \left( \EZone \left[ \langle \phi(z_1),\mu_z \rangle \langle \phi(z_1),\mu_x \rangle \right] - \EZone \left[ \langle \phi(z_1),\mu_z \rangle \right] \EZone \left[ \langle \phi(z_1),\mu_x \rangle \right] \right)  \nonumber \\
&  \qquad + \EXone \left[ \langle \phi(x_1),\mu_y \rangle^2 \right] - \EXone \left[ \langle \phi(x_1),\mu_y \rangle \right]^2 \nonumber\\
& \qquad + \EYone \left[ \langle \phi(y_1),\mu_z \rangle^2 \right] - \EYone \left[ \langle \phi(y_1),\mu_z \rangle\right]^2 \nonumber \\
&  \qquad + \EYone \left[ \langle \phi(y_1),\mu_x \rangle^2 \right] - \EYone \left[ \langle \phi(y_1),\mu_x \rangle \right]^2 \nonumber\\
& \qquad + \EZone \left[ \langle \phi(z_1),\mu_x \rangle^2 \right] - \EZone \left[ \langle \phi(z_1),\mu_x \rangle \right]^2 \nonumber \\
&  \qquad -2 \left( \EXone \left[ \langle \phi(x_1),\mu_y \rangle \right] \EXone \left[ \langle \phi(x_1),\mu_z \rangle \right] \right) \nonumber
\end{align}
We have shown that Eq.~\eqref{eq:denominator_of_the_TestStatistics} is equal to Eq.~\eqref{eq:variance_difference_of2MMD}. We use this equality to compute the $p$-value for the Relative MMD test.


\vskip 0.2in
\bibliographystyle{abbrvnat}
\bibliography{icml2015_bibliography,iclr2016_bibliography}

\end{document}